\documentclass{article}

    \usepackage[preprint]{neurips_2022}

\usepackage{bbding}

\usepackage{amsthm}
\usepackage{amsmath,amsfonts}
\usepackage{mathtools}
\mathtoolsset{showonlyrefs}
\usepackage{enumerate,color,xcolor}
\usepackage{subfigure}
\usepackage{wrapfig}
\usepackage{url}
\usepackage[textsize=footnotesize]{todonotes}

\usepackage{booktabs} 

\usepackage{natbib}

\usepackage{enumitem}

\newtheorem{theorem}{Theorem}
\newtheorem{corollary}[theorem]{Corollary}
\newtheorem{lemma}[theorem]{Lemma}
\newtheorem{proposition}[theorem]{Proposition}

\newcommand{\beq}{\begin{eqnarray}}
\newcommand{\eeq}{\end{eqnarray}}
\newcommand{\beqs}{\begin{eqnarray*}}
\newcommand{\eeqs}{\end{eqnarray*}}

\usepackage{multirow}

\usepackage{natbib}

\usepackage{comment}

\numberwithin{equation}{section}

\begin{document}

\title{Heavy-Tail Phenomenon in Decentralized SGD}

\author{Mert G\"{u}rb\"{u}zbalaban\thanks{The authors are in alphabetical order. \Envelope~ Corresponding author.}\\
       Department of Management Science and Information Systems \\
       Rutgers Business School\\
       Piscataway, NJ 08854, United States of America\\
       \texttt{mg1366@rutgers.edu} \\
        \And
       Yuanhan Hu\footnotemark[1] ~\Envelope~\\ 
       Department of Management Science and Information Systems \\
       Rutgers Business School\\
       Piscataway, NJ 08854, United States of America\\
       \texttt{yuanhan.hu@rutgers.edu}\\
       \And
       Umut \c{S}im\c{s}ekli\footnotemark[1]\\
    DI ENS, \'{E}cole Normale Sup\'{e}rieure, Universit\'{e} PSL, CNRS, INRIA\\ 75005 Paris, France\\
        \texttt{umut.simsekli@inria.fr}\\
        \And
        Kun Yuan\footnotemark[1]\\
        DAMO Academy, Alibaba (US) Group\\
       \texttt{kun.yuan@alibaba-inc.com} \\
       \AND
       Lingjiong Zhu\footnotemark[1]\\
       Department of Mathematics \\
       Florida State University \\
       Tallahassee, FL 32306, United States of America\\
       \texttt{zhu@math.fsu.edu}\\
       }


\maketitle

\begin{abstract}
Recent theoretical studies have shown that heavy-tails can emerge in stochastic optimization due to `multiplicative noise', even under surprisingly simple settings, such as linear regression with Gaussian data. While these studies have uncovered several interesting phenomena, they consider conventional stochastic optimization problems, which exclude \emph{decentralized} settings that naturally arise in modern machine learning applications. In this paper, we study the emergence of heavy-tails in decentralized stochastic gradient descent (DE-SGD), and investigate the effect of decentralization on the tail behavior. We first show that, when the loss function at each computational node is twice continuously differentiable and strongly convex outside a compact region, the law of the DE-SGD iterates converges to a distribution with polynomially decaying (heavy) tails. To have a more explicit control on the tail exponent, we then consider the case where the loss at each node is a quadratic, and show that the tail-index can be estimated as a function of the step-size, batch-size, and the topological properties of the network of the computational nodes. Then, we provide theoretical and empirical results showing that DE-SGD has heavier tails than centralized SGD. We also compare DE-SGD to disconnected SGD where nodes distribute the data but do not communicate. Our theory uncovers an interesting interplay between the tails and the network structure: we identify two regimes of parameters (stepsize and network size), where DE-SGD 
can have lighter or heavier tails than disconnected SGD depending on the regime. Finally, to support our theoretical results, we provide numerical experiments conducted on both synthetic data and neural networks.
\looseness=-1

\end{abstract}


\section{Introduction}
Stochastic gradient descent (SGD) methods are workhorse methods for solving large-scale optimization problems arising in machine learning \citep{bottou2010large,bottou2012stochastic,ruder2016overview}. SGD has good scalability properties to large datasets and to high dimensions and often leads to solutions that can generalize well, i.e. perform well on unseen data. These favorable properties are among the key reasons why SGD is a default method both in academia and in industry for training predictive models in deep learning or more broadly in supervised learning tasks \citep{lecun2015deep,goodfellow2016deep,le2011optimization}.
\looseness=-1

SGD iterates move along the stochastic (noisy) estimates of the gradient of the underlying optimization objective. Recent studies on SGD methods in the last few years unveiled a mysterious ``heavy-tailed behavior" that can be observed in different ways during stochastic optimization. 
More specifically, it has been numerically observed that the gradient noise becomes often heavy-tailed with a non-Gaussian structure over iterations in deep learning practice (even if at the beginning of the SGD iterations the gradient noise may have a Gaussian-like structure) \citep{csimcsekli2019heavy,pmlr-v97-simsekli19a,ht_sgd_quad}. 
\citet{martin2019traditional, martin2020heavy} also showed that the network weights and the eigenvalues of weight matrices show a heavy-tailed behavior for well-trained networks. They proposed to fit a power law distribution to the empirical spectral density of the weight matrices and illustrated that heavier-tailed weight matrices indicate better generalization.
It has been also observed that the level of heaviness of the tails and non-Gaussianity depends on the SGD hyperparameters and the geometry of the landscape of the objective, smaller batch sizes and larger stepsize being associated to heavier tails \citep{panigrahi2019non,ht_sgd_quad}. Formalizing these empirical observations, recent theoretical studies \citep{ht_sgd_quad, hodgkinson2020multiplicative} showed that heavy tails arise due to the ``multiplicative noise" in stochastic optimization, and can arise even under surprisingly simple settings such as linear regression with Gaussian data. In particular, G\"urb\"uzbalaban \emph{et al.} showed that for least square problems subject to Gaussian data, the tails are monotonic with respect to the stepsize and the batch \citep{ht_sgd_quad} and there exists a range of stepsizes for which there exists a stationary distribution with an infinite variance. It has also been found that the heaviness of the tails are positively correlated with the generalization performance on a number of datasets and architectures \citep{simsekli2020hausdorff,barsbey2021heavy}. 

Aforementioned results from the literature about the heavy tails apply to conventional stochastic optimization problems (such as risk minimization), where all the data is assumed to be available at a centralized particular location. On the other hand, modern datasets are often too large to be handled by a single machine (or a processor) and are commonly collected, stored, and processed in a distributed manner over a network of computing (nodes) agents. This fact, together with increasing data privacy requirements, necessitated the development of communication-efficient  
algorithms for solving \emph{decentralized stochastic optimization} problems of the form\looseness=-1
\begin{equation}
\min\nolimits_{x \in \mathbb{R}^d} f(x):= \sum\nolimits_{i=1}^N f_i(x), \quad  f_i(x):=\mathbb{E}_{z_i\sim\mathcal{D}_i}[\ell(x, z_i)],\label{pbm-pop-min}
\end{equation}
where $\ell(x,z_{i})$ 
represents the instantaneous \emph{loss} at node $i$ based on the predictor $x$ and the data point $z_i$, where there are $N$ computation nodes lying on a connected undirected graph $G=(\mathcal{N},\mathcal{E})$. Here $\mathcal{N} = \{1,2,\dots,N\}$ is the set of (vertices) nodes and $\mathcal{E}\subseteq \mathcal{N}\times \mathcal{N}$ is the set of edges that define the connectivity patterns between the nodes. The objective $f_i$ is only available at the node $i$, and the aim is to
train models locally at each agent where only local parameters vectors are shared among the neighbors.
The problem \eqref{pbm-pop-min} arises in a number of key settings including but not limited to decentralized deep learning \citep{dean2012large,ying2021exponential} and federated learning \citep{li2020review}.

Decentralized versions of SGD that are based on a weighted averaging of local parameter vectors among neighbors are popular for solving \eqref{pbm-pop-min} where the weights are specified through an $N\times N$ stochastic matrix called the \emph{communication matrix } or the \emph{mixing matrix} \citep{robust-network-agd,gurbuzbalaban2021decentralized}. There are two popular decentralized versions of SGD, first version averages both local parameter vectors and gradients \citep{pu2020asymptotic,xin2020improved} whereas the second version averages only local parameter vectors \citep{robust-network-agd,gurbuzbalaban2021decentralized,ram2009asynchronous, chen2012diffusion,lian2017can}. Both versions admit similar updates, they attain similar convergence rates at least for strongly convex stochastic optimization problems \citep{robust-network-agd,pu2020asymptotic}. 
In this paper, we will focus on the latter version which has a long history at least going back to decentralized gradient methods studied in \citet{nedic2009distributed}. \looseness=-1

A natural question in this context is whether the ``heavy-tail phenomenon" arises in the decentralized SGD setting as well, and what would be the effect of decentralization on the tails; i.e. how would the heaviness of tails differ from the default centralized SGD (C-SGD) setting. 
To our knowledge, none of the existing theoretical works about heavy tails apply to the decentralized stochastic optimization, which will be the focus of this paper. Our contributions are summarized as follows: 
\begin{itemize}[noitemsep,topsep=0pt,leftmargin=*,align=left]
    \item We show in Proposition~\ref{prop:ht_noncvx} that if the local objectives $f_i$ are twice continuously differentiable and strongly convex outside a compact region, then DE-SGD iterates are heavy-tailed at stationarity, provided that the stepsize is small enough. 
    This result applies to a general class of objectives that can be non-convex on a compact set. Its proof is based on (i) relating the DE-SGD iterations in dimension $d$ to C-SGD iterations in higher dimensions (in dimension $Nd$) on a ``modified" objective where $N$ is the number of agents, (ii) showing by a careful analysis that the ``modified" objective inherits strong convexity properties outside a compact set in higher dimensions as well provided that the stepsize is small enough, (iii) building on existing results for C-SGD from \citet{hodgkinson2020multiplicative}. \looseness=-1
    \item To get more precise characterizations of the tail behavior, similar to \citet{ht_sgd_quad}, we consider 
    least square problems (where the losses $f_i$ are
quadratics) when the underlying data has a continuous distribution with finite moments. Our result (Theorem~\ref{thm:main:dsg}) shows that the iterates converge to a heavy-tailed distribution whose tail probabilities decay polynomially at a specific rate $\alpha$. 
The decay rate $\alpha$ 
is hard to compute exactly in general; however, we discuss how it can be estimated as a unique positive solution $\hat{\alpha}$ to a non-linear equation $\hat h(s) = 1$ where $\hat h$ is a function of the stepsize, batch-size, the curvature of the objective and the mixing matrix. Furthermore, building on the connections between C-SGD and DE-SGD summarized in the previous bullet point, we show in Theorem~\ref{thm:mono} that the tail-index $\hat{\alpha}$ is strictly increasing in batch sizes $b_i$ at every node $i$ and strictly decreasing in stepsize $\eta$ provided that $\hat\alpha\geq 1$. These results show that the effect of hyperparameters in the decentralized setting resembles to the centralized setting previously studied in \citet{ht_sgd_quad}. \looseness=-1
\item
To complement Theorem~\ref{thm:main:dsg}, that is of asymptotic nature, 
we provide non-asymptotic moment
bounds for $k$-th iterate and also characterize the non-asymptotic convergence speed
from the $k$-th iterate to the stationary limit
in Wasserstein metric 
in the Appendix.
\item To understand the effect of decentralization and network averaging, we compare the tail-index of DE-SGD to that of C-SGD and disconnected SGD (Dis-SGD); the latter algorithm runs independent copies of SGD at each node without connecting them through a network while distributing the data points evenly. Here, Dis-SGD behaves like an intermediate point between C-SGD and DE-SGD and serves as a benchmark. We first infer from existing results about C-SGD that Dis-SGD has heavier tails compared to C-SGD (Proposition~\ref{prop:disconnected:centralized}). Then, under Gaussian input data assumption, we show in Theorem \ref{thm:first:order:alpha} that if the network size is large enough or if the stepsize is large enough, DE-SGD has heavier tails compared to Dis-SGD under a mild technical assumption which ensures that the stationary distribution exists. On the other hand, when the stepsize is small enough or if the network is small enough, we show that DE-SGD has lighter tails compared to Dis-SGD. These results allow us to compare DE-SGD with C-SGD showing that DE-SGD has heavier tails than C-SGD for some range of parameters (Corollary~\ref{cor:DSGD:CSGD}). In fact, our theory uncovers two regimes of parameters (stepsize and network size), where addition of the network links can make the tails lighter or heavier (compared to the Dis-SGD setting). This phenomenon is demonstrative of rich interplays between the heavy tails and the network structure, and we believe our work is the first step towards understanding this interaction. In the main text, due to space limitations we consider the simpler (and more intuitive) $d=1$ case, results about general dimension $d$ can be found in the Appendix.
\item Adapting the results of \citet{gao2015stable} to our setting, we provide a generalized central limit theorem for the averaged iterates in the Appendix which show that they follow an $\alpha$-stable distribution in the limit. This result allows us to estimate the tails more accurately in our experiments based on recent advanced estimators available for $\alpha$-stable distributions. Finally, we provide numerical experiments conducted on both synthetic data and neural
networks to empirically support and illustrate our theoretical results. For both fully-connected networks and modern networks with convolutional layers such as ResNet-20, we provide experiments for DE-SGD that illustrate the heavy-tailed behavior in practice.
\end{itemize}

\section{Decentralized SGD and Preliminaries}
\label{sec:bg}
\textbf{Decentralized SGD (DE-SGD).}  
A standard method for decentralized learning is the decentralized stochastic gradient (DE-SGD) method 
\citep{robust-network-agd, Yuan16, ram2009asynchronous, lian2017can}. At iteration $k$, DE-SGD updates the local variable $x_{i}^{(k)}$ at node $i$ with the following recursion
\begin{equation}\label{eqn:dsg_update}
x_{i}^{(k+1)}=\sum\nolimits_{\ell\in\Omega_{i}}W_{i\ell}x_{\ell}^{(k)}-\eta\tilde{\nabla} f_{i}(x_{i}^{(k)}),
\end{equation}
where $\eta>0$ is the stepsize, $W \in \mathbb{R}^{N\times N}$ is a symmetric double stochastic matrix also known as the communication (mixing) matrix, with $W_{ij}=W_{ji}>0$ if $j\in\Omega_{i}$,
and $W_{ij}=W_{ji}=0$ if $j\not\in \Omega_{i}$ and $i\neq j$,
and finally $W_{ii}=1-\sum_{j\neq i}W_{ij}>0$
for every $1\leq i\leq N$.\footnote{We refer the reader to the Appendix for more about the properties of the matrix $W$.} $\Omega_i$ is the set of neighbors of node $i$ on the network, i.e. $\Omega_i = \{ j : (i,j) \in \mathcal{E}\}$. Moreover, in \eqref{eqn:dsg_update}, $\tilde{\nabla} f_{i}(x)$ is an estimate of the gradient of the loss $f_i(x)$ at node $i$ based on a batch size $b_i$, i.e. it is based on $b_i$ random draws from data satisfying, 
\begin{equation} \tilde{\nabla} f_{i}( x_i^{(k)})
 := (1/b_i) \sum\nolimits_{j=1}^{b_i} \nabla\ell (x_i^{(k)}, z_{i,j}^{(k)}),
 \label{eq-stoc-grad}
\end{equation}
where $z_{i,j}^{(k)}$ are (fresh) i.i.d. draws from data at step $k$. \looseness=-1

\textbf{Heavy-tailed distributions with a power-law decay.}
A real-valued random variable $X$ is said to be \emph{heavy-tailed}
if the right tail or the left tail of the distribution  decays slower than any exponential distribution. 
We say $X$ has heavy (right) tail
if $\lim_{x\rightarrow\infty}\mathbb{P}(X\geq x)e^{cx}=\infty$
for any $c>0$, 
and a real-valued random variable $X$ has heavy (left) tail 
if $\lim_{x\rightarrow\infty}\mathbb{P}(X\leq-x)e^{c|x|}=\infty$
for any $c>0$; see e.g. \citep{Schmidli2009}. Similarly, an $\mathbb{R}^{d}$-valued random vector $X$
has heavy tail if $u^{T}X$ has heavy right tail 
for some vector $u\in\mathbb{S}^{d-1}$, 
where $\mathbb{S}^{d-1}:=\{u\in\mathbb{R}^{d}:\Vert u\Vert=1\}$ is the unit sphere. 

Heavy tail distributions include $\alpha$-stable distributions,
Pareto distribution, log-normal distribution and the Weilbull distribution. One important class of the heavy-tailed distributions
is the distributions with \emph{power-law} decay, 
which is the focus of our paper. That is,
$\mathbb{P}(X\geq x)\sim c_{0}x^{-\alpha}$
as $x\rightarrow\infty$ for some $c_{0}>0$ and $\alpha>0$,
where $\alpha>0$ is known as the \emph{tail-index},
which determines the tail thickness of the distribution.
Similarly, we say that the random vector $X$ has power-law decay
with tail-index $\alpha$ if for some $u\in\mathbb{S}^{d-1}$,
we have $\mathbb{P}(u^{T}X\geq x)\sim c_{0}x^{-\alpha}$,
for some $c_{0},\alpha>0$.
For a reference on heavy-tailed and power-law distributions, we refer to \citet{Foss2013}.


\textbf{Viewing DE-SGD as centralized SGD in higher dimensions.}
As in \citet{robust-network-agd,Yuan16}, we can express the DE-SGD iterations as
$x^{(k+1)}=\mathcal{W}x^{(k)}-\eta \tilde{\nabla} F\left(x^{(k)}\right)$,
where $\mathcal{W} := W \otimes I_d$, $F:\mathbb{R}^{Nd}\rightarrow\mathbb{R}$ defined as
$F(x):=F(x_{1},\ldots,x_{N})
=\sum\nolimits_{i=1}^{N}f_{i}(x_{i})$, 
with $x^{(k)}:=[(x_{1}^{(k)})^{T},(x_{2}^{(k)})^{T},\ldots,(x_{N}^{(k)})^{T}]^{T}\in\mathbb{R}^{Nd}$
and 
\begin{equation}
\tilde{\nabla}F(x^{(k)}):=[(\tilde{\nabla} f_1 (x_1^{(k)}))^{T},\ldots,(\tilde{\nabla} f_N (x_N^{(k)}))^{T}]^{T}.
\label{def-stoc-grad}
\end{equation} 
We can alternatively view DE-SGD as C-SGD iterations
\begin{equation} 
x^{(k+1)} = x^{(k)} - \eta \tilde \nabla F_{\mathcal{W}}(x^{(k)})
\label{eq-dsgd-vs-sgd}
\end{equation}
on a modified objective $F_{\mathcal{W}}$ defined as 
\begin{equation} 
F_{\mathcal{W}}(x) := F(x) +  (1/2)\eta^{-1}x^T (I_{Nd}-\mathcal{W}) x, \label{def-FW}
\end{equation}
with the convention that 
$\tilde \nabla F_{\mathcal{W}}(x) = \tilde \nabla F(x) + \frac{1}{\eta} (I_{Nd}-\mathcal{W}) x$ 
(see e.g. \citep{robust-network-agd}). Similar to \eqref{def-stoc-grad}, we can define the stochastic Hessian as $\tilde{\nabla}^2 f_{i}( x_i^{(k)})
 := \frac{1}{b_i} \sum_{j=1}^{b_i} \nabla^2 \ell(x_i^{(k)}, z_{i,j}^{(k)})$ with 
\begin{equation}
\tilde{\nabla}^2 F(x^{(k)}):=[(\tilde{\nabla}^2 f_1 (x_1^{(k)}))^{T},\ldots,(\tilde{\nabla}^2 f_N (x_N^{(k)}))^{T}]^{T},
\label{def-stoc-hessian} 
\end{equation} 
when $f_i$'s are twice differentiable for every $i=1,2,\dots,N$.

\section{Decentralized SGD and Heavy Tails}
\label{sec-least-squares}


\textbf{3.1 General smooth loss. }
We begin by the following result, which is an extension of \cite[Theorem~1]{hodgkinson2020multiplicative}, from centralized to the decentralized setting. It shows that if the objectives $f_i$ are strongly convex outside a compact set for every $i$ and if the stepsize is small enough, then the iterates will admit a heavy-tailed distribution. The proof is based on showing that under these assumptions, the function $F_{\mathcal{W}}(x)$ defined in \eqref{def-FW} is strongly convex outside a compact set for $\eta$ small and then using the identity \eqref{eq-dsgd-vs-sgd} that relates DE-SGD to C-SGD. Under our assumptions, the function $F(x)$ is not necessarily strongly convex outside a compact set (because when $x$ is large some coordinates $x_i$ can potentially be arbitrarily small); and the quadratic term in \eqref{def-FW} can vanish on a subspace. However, by a careful analysis, we show that outside a large enough compact set, when $F(x)$ does not have strong convexity then the quadratic term does have it so that the sum $F_{\mathcal{W}}(x)$ defined in \eqref{def-FW} is strongly convex as desired.

%

\begin{proposition}
\label{prop:ht_noncvx}
Assume that $f_i(x)$ is twice continuously differentiable on $\mathbb{R}^d$ and each $f_i$ is strongly convex outside a compact region in $\mathbb{R}^d$ for every $i=1,2,\dots, N$. Define 
$\mathfrak{L}_{\text{up}}:= \sup_{x\in\mathbb{R}^{Nd}} \|\mathcal{W} - \eta \tilde{\nabla}^2 F(x) \|$ and
$\mathfrak{L}_{\text{low}}:=  \liminf_{  \|x\| \to \infty } \sigma_{\mathrm{min}} ( \mathcal{W} - \eta \tilde \nabla^2 F(x) )$,
where $\tilde \nabla^2 F(x)$ is as in \eqref{def-stoc-hessian}, $\sigma_{\mathrm{min}} (A)$ denotes the smallest the singular value of a matrix $A$.
Further assume that the following holds:
$\mathbb{E} \left[\log \mathfrak{L}_{\text{up}}\right] < 0, \quad \mathbb{E} \left[\mathfrak{L}_{\text{up}}\right] < \infty, \quad \mathbb{P}(\mathfrak{L}_{\text{low}} > 1) > 0$.
%
%
Then, for $\eta$ small enough, there exists $\alpha,\beta >0$ such that $\mathbb{E} \left[\mathfrak{L}_{low}^\alpha\right] = \mathbb{E}\left[\mathfrak{L}_{up}^\beta\right] = 1$ and for every $\varepsilon >0$,
$\limsup_{t \to \infty} t^{\alpha+\varepsilon} \mathbb{P}\left(\|x^{(\infty)}\|>t\right) >0$
and
$\limsup_{t \to \infty} t^{\beta-\varepsilon} \mathbb{P}\left(\|x^{(\infty)}\|>t\right) < \infty$.
\end{proposition}
Proposition~\ref{prop:ht_noncvx} shows that the norm of the iterates is heavy-tailed with a tail-index lying between $\beta$ and $\alpha$. In the following section, we show that for least square problems (where the losses $f_i$ are quadratics), we can get stronger results 
which characterize the tail-index 
in a more precise manner.
\textbf{3.2 Quadratic loss. }  
In the specific case of decentralized least squares, the loss function in \eqref{pbm-pop-min} is a quadratic of the form $ \ell(x,z_i) =\frac{1}{2}(a_i^T x - y_i)^2 \quad \mbox{for every node $i$},$ 
where $z_i = (a_i,y_i)$ is the local data at agent $i$ with $a_i$ representing the input feature vector and $y_i$ being the output. Recalling that each node $i$ has access to $b_i$ samples from data $\{z_{i,j}^{(k)} =(a_{i,j}^{(k)},y_{i,j}^{(k)})\}_{j=1}^{n_i}$ at every iteration $k$ to form a stochastic gradient estimate, \eqref{eq-stoc-grad} becomes
$\tilde{\nabla} f_{i}(x_i^{(k)})
 := \frac{1}{b_i} \sum_{j=1}^{b_i}[ a_{i,j}^{(k)}(a_{i,j}^{(k)})^T x_i^{(k)} - y_{i,j}^{(k)} a_{i,j}^{(k)}]$. 
DE-SGD iterations then become
\begin{equation}\label{eq-iter-dgd}
x^{(k+1)}= M^{(k+1)} x^{(k)} + q^{(k+1)},
\quad
\text{where}
\quad
M^{(k+1)} := \mathcal{W} - \eta H^{(k+1)}, 
\end{equation}
with 
$H^{(k+1)} := \mbox{blkdiag} ( \{ H_{i}^{(k+1)} \}_{i=1}^N )$, 
$q^{(k+1)} := [(q_{1}^{(k+1)})^{T},(q_{2}^{(k+1)})^{T},\ldots,(q_{N}^{(k+1)})^{T}]^{T}$ where for $i=1,2,\dots,N\,$
\begin{align}
H_{i}^{(k+1)} := (1/b_i)\sum\nolimits_{j = 1}^{b_i} a_{i,j}^{(k)} (a_{i,j}^{(k)})^T,
\quad
q_i^{(k+1)} := (\eta/b_i) \sum\nolimits_{j=1}^{b_i} a_{i,j}^{(k)} y_{i,j}^{(k)}\,,\label{def-Hk}
\end{align}
where $a_{i,j}^{(k)}$ and $y_{i,j}^{(k)}$ are i.i.d. over $k$ random draws from the data
with the same distribution as $a_{i,j}, y_{i,j}$ that satisfy the following assumptions:

\begin{itemize}
    \item [\textbf{(A1)}] For every $i$, $a_{i,j}$ are i.i.d. over $j$ following a continuous distribution supported
on $\mathbb{R}^d$ with all the moments finite. 
    \item [\textbf{(A2)}] For every $i=1,2,\dots,N$, $y_{i,j}$ are i.i.d. over $j$ with a continuous density whose support is $\mathbb{R}$ with all the moments finite. 
\end{itemize} 

We assume $\textbf{(A1)}$ and $\textbf{(A2)}$ throughout the paper, 
and they are satisfied in a large variety of cases, for instance when $a_{i,j}$ and $y_{i,j}$ are Gaussian distributed. We recall the concatenated iterates 
$x^{(k+1)}=M^{(k+1)}x^{(k)}+q^{(k+1)}$,
where $M^{(k+1)},q^{(k+1)}$ are defined by \eqref{eq-iter-dgd}, \eqref{def-Hk}.
Let us introduce
\begin{equation}
h(s) := \lim\nolimits_{k\to\infty}(\mathbb{E}\| M^{(k)} M^{(k-1)}\dots M^{(1)}\|^s)^{1/k}\,,
\label{def-hs}
\end{equation}
which arises in stochastic matrix recursions (see e.g. \citep{buraczewski2014multidimensional}) where $\|\cdot\|$ denotes the matrix 2-norm (i.e. largest singular value of a matrix).  
Since $\mathbb{E}\left\|M^{(k)}\right\|^s < \infty$ for all $k$ and $s>0$, we have $h(s) < \infty$.  
Let us also define $\Pi^{(k)} := M^{(k)} M^{(k-1)}\dots M^{(1)}$ and
\begin{equation}
\rho := \lim\nolimits_{k\to\infty} (2k)^{-1} \log(\mbox{largest eigenvalue of } (\Pi^{(k)})^T \Pi^{(k)})\,.
\label{def-rho}
\end{equation}
The latter quantity is 
called the top Lyapunov exponent of the stochastic recursion (\ref{eq-iter-dgd}). 

In the following, by following similar arguments as in \citet{ht_sgd_quad}, 
we can show that the limit density
has a polynomial tail
with a tail-index given precisely by $\alpha$, the unique critical value
such that $h(\alpha)=1$. 
The result builds on adapting the techniques developed in stochastic matrix recursions \citep{alsmeyer2012tail,buraczewski2016stochastic} to our setting. Our result shows that even in the simplest setting when the input data is i.i.d. without any heavy tail, DE-SGD iterates can lead to a heavy-tailed stationary distribution with an infinite variance.

\begin{theorem}\label{thm:main:dsg}
Suppose Assumptions \textbf{(A1)}-\textbf{(A2)} hold.
Consider the DE-SGD iterations \eqref{eq-iter-dgd}.
If $\rho<0$ and there exists a unique positive $\alpha$ such that $h(\alpha)=1$, then 
\eqref{eq-iter-dgd} admits a unique stationary solution $x^{(\infty)}$
and the DE-SGD iterations converge to $x^{(\infty)}$ in distribution,
where the distribution of $x^{(\infty)}$ satisfies
$\lim\nolimits_{t\to\infty} t^\alpha \mathbb{P}\left(u^T x^{(\infty)} > t \right)= g_\alpha(u)$, for any $u\in\mathbb{S}^{Nd-1}$,
for some positive and continuous function $g_\alpha$ on 
$\mathbb{S}^{Nd-1}$.
\end{theorem}
\textbf{Generalized CLT (GCLT).} Following Theorem~\ref{thm:main:dsg}, one can derive a generalized central limit theorem (GCLT) result
that shows that when properly scaled, the sum of the iterates $S_{K}:=\sum_{k=1}^{K} x^{(k)}$ converges in law to a stable distribution. 
We refer to 
the Appendix
for the details.

Theorem \ref{thm:main:dsg} provides a formula for the tail-index $\alpha$. However, since $\rho$ and $h(s)$ do not have simple closed-form formulas, it is hard to write an explicit formula for $\alpha$ in Theorem~\ref{thm:main:dsg}. In fact, computing $\rho$ is hard in general \citep{tsitsiklis1997lyapunov}. That being said, 
using the sub-multiplicativity of the norm of matrix products appearing in \eqref{def-hs} and \eqref{def-rho}, they can be upper bounded as follows:\looseness=-1
\begin{align}
\rho\leq\hat{\rho}:=\mathbb{E}\log\left\Vert\mathcal{W}-\eta H\right\Vert,
\qquad
h(s) \leq \hat{h}(s):=\mathbb{E}\left[\left\Vert\mathcal{W}-\eta H\right\Vert^{s}\right], 
    \label{ineq-h-s}
\end{align}
where $H$ is a matrix that has the same distribution as $H^{(k+1)}$ (which does not depend on $k$). 


\textbf{Lower bounds on the tail-index $\alpha$.} If $\hat{\alpha}$ is such that $\hat{h}(\hat\alpha)=1$, then by (\ref{ineq-h-s}), $\hat{\alpha}$ is a lower bound on the tail-index $\alpha$ that satisfies $h(\alpha)=1$ where $h$ is defined as in (\ref{def-hs}). In other words, we have $\hat{\alpha}\leq \alpha$ and therefore $\hat{\alpha}$ serves as a lower bound on the tail-index.

\textbf{Checking the conditions in Theorem~\ref{thm:main:dsg}.} We remark that $\hat{\rho}$ and $\hat{h}(s)$ can help us check the conditions in Theorem~\ref{thm:main:dsg}.
Since $\rho\leq\hat{\rho}$, we have $\rho<0$ when
$\hat{\rho}<0$. Moreover, $h(0)=\hat{h}(0)=1$, and one can check that $h(s)$ is convex in $s$. When $\hat{\rho}<0$, $\hat{h}'(0)=\hat{\rho}<0$, and $h(s)\leq\hat{h}(s)<1$ for any sufficiently small $s>0$. Under some mild assumption {on the data distribution}, one can check that $\liminf_{s\rightarrow\infty}h(s)>1$ and thus there exists a unique positive $\alpha$ such that $h(\alpha)=1$.

In the next section, we study the properties of the tail-index of DE-SGD iterations further.
\section{Theoretical Analysis for the Tail-Index}\label{sec:theoretical}



\textbf{4.1 Monotonicity of the tail-index.}
Assume that $\hat{\rho}<0$ so that there exists
a unique positive $\hat{\alpha}$ such that $\hat{h}(\hat{\alpha})=1$.
We can obtain the following monotonicity result for $\hat{\alpha}$ with respect to the hyperparameters of the model 
by extending the proof of Theorem~4 in \citet{ht_sgd_quad} from centralized to the decentralized setting. 

\begin{theorem}\label{thm:mono}
Suppose Assumptions \textbf{(A1)}-\textbf{(A2)} hold. 
The tail-index $\hat{\alpha}$ is strictly increasing in batch-sizes $b_{i}$
and strictly decreasing in stepsize $\eta$ provided that $\hat{\alpha}\geq 1$.
Moreover, the tail-index $\hat{\alpha}$ is strictly decreasing in dimension $d$.
\end{theorem}
In the next section, we compare the DE-SGD to Dis-SGD (where nodes do not communicate at all with each other) in terms of the tail-index. We will also compare with the C-SGD. 

\textbf{4.2 Tail-index comparison between Disconnected SGD and Centralized SGD. } We start with defining properly what exactly we mean by Disconnected SGD and Centralized SGD iterations. 

\textbf{{Disconnected SGD}.} Disconnected SGD (Dis-SGD) corresponds to the case $W=I$ (where nodes do not share information with other nodes), and for every $i=1,2,\ldots,N$, 
the iterates follow the recursion:
$x_{i}^{(k+1)}=x_{i}^{(k)}-\eta\tilde{\nabla}f_{i}(x_{i}^{(k)})$, 
where each gradient $\tilde{\nabla}f_{i}(x_{i}^{(k)})$ is based on $b_{i}$ samples from node $i$'s dataset. The total number of samples (cumulatively over the nodes) equals $\sum_{i=1}^N b_i$.

\textbf{{Centralized SGD}.} Centralized SGD (C-SGD) consists of the iterations 
$x_{k+1}=x_{k}-\eta\tilde{\nabla}f(x_{k})$,
where we take 
(number of data points per iteration) batch-size
to be $\sum_{i=1}^{N}b_{i}$ for centralized SGD. 

We will first be comparing C-SGD to Dis-SGD and then we will be comparing it to the DE-SGD. 

To make the comparison of Dis-SGD and C-SGD easier, we assume
that $b_{i}\equiv b$ and $a_{i,j}$ are i.i.d. over $i$ and $j$.
Under Assumptions \textbf{(A1)}-\textbf{(A2)}, 
the iterates $x_{i}^{(k)}$ are independent and
as it is shown in \citet{ht_sgd_quad} that
$\hat{\alpha}=\hat{\alpha}(b)$ is a lower bound
of the tail-index of $x_{i}^{(\infty)}$
which is the unique positive value satisfying $\hat{h}(\hat{\alpha}(b))=1$, 
where $\hat{h}(s)=\mathbb{E}[\Vert I-\eta H\Vert^{s}]$,
where $\hat{\alpha}(b)$ emphasize the dependence on the batch-size $b$
such that for each node $i$, $b$ data points are chosen. 
We use $\hat{\alpha}(b)$ as a proxy of the tail-index.
In the centralized setting, 
where $bN$ data points are chosen
at each iteration, and hence the batch-size 
equals $bN$ for centralized SGD. 
Thus, as it is shown in \citet{ht_sgd_quad} that
the corresponding tail-index (proxy) is $\hat{\alpha}(bN)$. 
We have the following observation by adapting the
the monotonicity properties of tail-index shown in \cite[Theorem~4]{ht_sgd_quad}.

\begin{proposition}
\label{prop:disconnected:centralized}
The tail-index for disconnected SGD is smaller than that
of the centralized SGD. Indeed, their difference gets larger
as the network size increases.
\end{proposition}

\textbf{4.3 Tail-index comparison between Decentralized SGD and Disconnected SGD. }
We choose the mixing matrix as $W= I_N - \delta L$ with $\delta>0$ small enough so that the spectral radius of $W$ is not larger than 1, where we recall that $I_N$ denotes the $N\times N$ identity
matrix and $L$ is the $N\times N$ graph Laplacian. This choice of the mixing matrix has been common in the literature \citep{pu2020asymptotic, hendrikx2019accelerated}. We consider the function
$g(\delta) := \| \mathcal{W} - \eta H\| = 
\| I_{Nd} - \eta H - \delta (L\otimes I_d) \|$, where we recall that $H^{(k+1)} := \mbox{blkdiag} ( \{ H_{i}^{(k+1)} \}_{i=1}^N )$ and $\mathcal{W} = W \otimes I_d$. Recall from \eqref{ineq-h-s} that $\hat{h}(s):=\mathbb{E}\left[\left(g(\delta)\right)^{s}\right]$ and the tail-index estimate $\hat\alpha = \hat\alpha (\delta)$ depends on $\delta$ and is the unique positive value that satisfies $\hat h\left(\hat\alpha(\delta)\right) = 1$ provided that $\hat{\rho}=\hat{\rho}(\delta)=\mathbb{E}[\log g(\delta)]<0$. The case $\delta = 0$ corresponds to Dis-SGD, whereas $\delta>0$ corresponds to the DE-SGD. 
When $\delta=0$, we use the notation $\hat{h}_{dis}(s)=\mathbb{E}\left[\left(g(0)\right)^{s}\right]$, $\hat{\alpha}_{dis}=\hat{\alpha}(0)$ and $\hat{\rho}_{dis}=\hat{\rho}(0)$ for the disconnected case.
A natural question would be how the tail-index changes when the network effect is introduced; i.e. whether $\alpha'(0):=\frac{d\alpha(\sigma)}{d\sigma}|_{\sigma=0}$ is positive or not. 
To answer this question, we use the Taylor 
series expansion of $g(\delta)$ around $\delta=0$ by the standard perturbation theory for singular values
(see e.g. \cite[Lemma 2.3]{guglielmi2011fast})
to obtain a first-order expansion for $\hat{h}(s)$
in terms of $\hat{h}_{dis}(s)$, a correction term
linear in $\delta$, and a higher-order error term
(see Lemma~\ref{lem:first:order} in the Appendix).
Based on this, we can obtain a first-order approximation of $\hat{\alpha}$
in terms of $\hat{\alpha}_{dis}$, a correction term linear in $\delta$
and a higher-order error term, where the tail-indexes $\hat{\alpha}$ and $\hat{\alpha}_{dis}$
are the unique positive values such that $\hat{h}(\hat{\alpha})=1$
and $\hat{h}_{dis}(\hat{\alpha}_{dis})=1$ provided that $\hat{\rho}_{dis}<0$
and $\delta$ is sufficiently small.

\begin{theorem}\label{thm:first:order:alpha}
Assume $d=1$ and $b=1$.
Also assume that $\hat{\rho}_{dis}=\mathbb{E}\left[\log(\max_{1\leq i\leq N}|1-\eta a_{i}^{2}|)\right]<0$. 
As $\delta\rightarrow 0$, we have 
\begin{align}\label{alpha:expansion}
\hat{\alpha}&=\hat{\alpha}_{dis}
-\frac{s\delta N^{-1}\sum_{i=1}^{N}L_{ii}[1-2\mathbb{P}(\min_{1\leq i\leq N}a_{i}^{2}+\max_{1\leq i\leq N}a_{i}^{2}<2/\eta)]}{\mathbb{E}[\log(\max_{1\leq i\leq N}|1-\eta a_{i}^{2}|)
\max_{1\leq i\leq N}|1-\eta a_{i}^{2}|^{\hat{\alpha}_{dis}}]}+o(\delta),
\end{align}
where we provide an explicit formula for the probability and the expectation on the right hand side of \eqref{alpha:expansion} in Theorem~\ref{thm:first:order:alpha-extended} in the Appendix.
\end{theorem}

\begin{wrapfigure}{r}{0.5\textwidth}
\vspace{-0.35in} 
  \begin{center}
    \includegraphics[width=0.4\columnwidth]{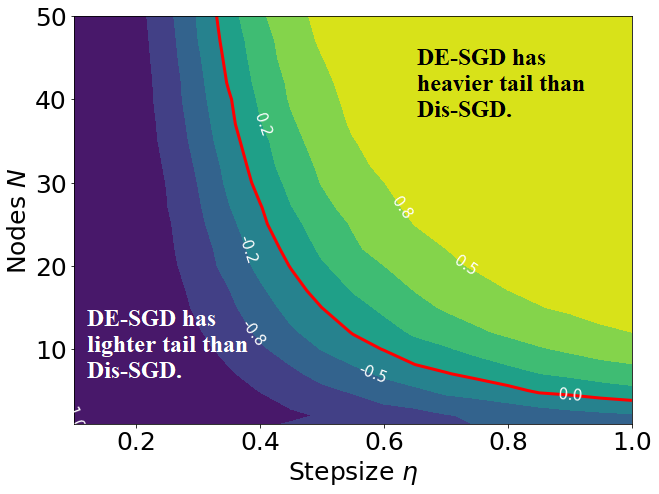}
    \label{fig:probability-contour}
  \end{center}
  \vspace{-0.2in}
  \caption{The contour plot of $e({\eta,N})=1-2\mathbb{P}\left(\min_{1\leq i\leq N}a_{i}^{2}+\max_{1\leq i\leq N}a_{i}^{2}<2/\eta\right)$.}
\end{wrapfigure}

We can see from Theorem~\ref{thm:first:order:alpha}
that the second term on the right hand side of \eqref{alpha:expansion}
is positive if and only if $e({\eta,N}):= 1 - 2\mathbb{P}\left(\min_{1\leq i\leq N}a_{i}^{2}+\max_{1\leq i\leq N}a_{i}^{2}<\frac{2}{\eta}\right)>0$, 
and this term is monotonically increasing in $N$ and $\eta$. 
Therefore, we conclude from Theorem~\ref{thm:first:order:alpha}
that when $\delta$ is small, 
$\hat{\alpha}<\hat{\alpha}_{dis}$ given the stepsize $\eta$ or network size $N$
is large.
On the other hand, when $\delta$ is small, 
$\hat{\alpha}>\hat{\alpha}_{dis}$ given the stepsize $\eta$ or network size $N$
is small. This is illustrated in Fig. \ref{fig:probability-contour} for $d=b=1$ and $a_i \sim \mathcal{N}(0,1)$ which shows the contour plot of $e(\eta,N)$, where the curve $e(\eta,N)=0$ is drawn in red color. In the plot, every$(\eta,N)$ pair satisfies $\hat\rho_{dis}<0$ and large stepsize or large enough network to the right of red curve results in heavier tail ($\hat\alpha<\hat\alpha_{dis}$) in DE-SGD. 
We recall that $\hat{\alpha}$ and $\hat{\alpha}_{dis}$ are the unique positive values
such that $\hat{h}(\hat{\alpha})=1$ and $\hat{h}_{dis}(\hat{\alpha}_{dis})=1$. 
They correspond to the tail-indexes in the decentralized case with network effect
and without network effect. We also have the following result that compares $\hat{\alpha}$ to  $\hat{\alpha}_{dis}$.

\begin{corollary}\label{cor:first:order:comparison}
(i) Under the assumptions in Theorem~\ref{thm:first:order:alpha}, 
further assume that the stepsize $\eta>\frac{2}{F_{a}^{-1}(2^{-\frac{1}{N}})}$.
For sufficiently small $\delta$, $\hat{\alpha}<\hat{\alpha}_{dis}$, i.e.
the tail gets heavier with the presence of network effect.
(ii) Further assume that $a_{i}$ are i.i.d. $\mathcal{N}(0,\sigma^{2})$ distributed.
Then, the first-order expansion \eqref{alpha:expansion} can be further simplified (see Corollary~\ref{cor:first:order:comparison:extended} in the Appendix).
Moreover, when the variance $\sigma^{2}>\frac{1}{\eta}(\text{erf}^{-1}(2^{-\frac{1}{N}}))^{-2}$,
for sufficiently small $\delta$, $\hat{\alpha}<\hat{\alpha}_{dis}$, i.e.
the tail gets heavier with the presence of network effect.
\end{corollary}
\begin{wrapfigure}{r}{0.5\textwidth}
  \begin{center}
    \includegraphics[width=0.4\columnwidth]{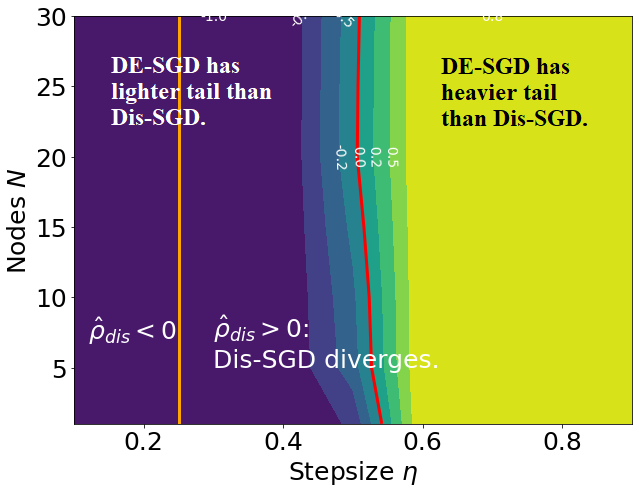}
  \end{center}
  \caption{\label{fig:probability-contour-100}The red curve satisfies $\hat\alpha=\hat\alpha_{dis}$ (up to a $o(\delta)$ term), $d$=$b$=100, $a_i \sim \mathcal{N}(0, I_d)$ with $\hat\rho_{dis}=0$ line in orange.}
\end{wrapfigure}
We can see from Cor.~\ref{cor:first:order:comparison} that when
the network size $N$ or the stepsize $\eta$ is large, then, 
for sufficiently small $\delta$, 
the tail gets heavier with the presence of network effect provided that $\hat{\rho}_{dis}<0$ (this condition ensures existence of stationary distribution and is tight in the sense that $\hat{\rho}_{dis}>0$ implies divergence \citep{buraczewski2016stochastic}.
Moreover, when $a_{i}$ are i.i.d. $\mathcal{N}(0,\sigma^{2})$ distributed
with $\sigma^{2}$ being large, then for sufficiently small $\delta$,
the tail gets heavier with the presence of network effect.
The general $d$ case is analyzed in the Appendix, and is illustrated in Fig. \ref{fig:probability-contour-100} for $d=100$. Similar to Fig. \ref{fig:probability-contour}, we can identify a red curve, which divides the $(\eta, N)$ space into two regions. In Region I (left-hand side of the red curve), DE-SGD has lighter tails compared to Dis-SGD, and in Region II (right-hand side of the red curve), DE-SGD has heavier tails compared to Dis-SGD. One difference with respect to the $d=1$ case, is that here the stepsize and $N$ cannot be too large otherwise the iterates will diverge. This is illustrated by the orange line which shows the points with $\hat\rho_{dis}=0$. In Region II, we have unstability $\hat\rho_{dis}>0$ (so that iterates diverge). Correspondingly, only part of the Region I is stable (with convergent iterates) and in this case DE-SGD has lighter tails.\looseness=-1

\textbf{4.4 Tail-index comparison between Decentralized SGD and Centralized SGD. }

In this section, we are interested in comparing the tail-index
for the DE-SGD with that of the C-SGD.
To make the comparison simpler, we assume $d=1$, $b_{i}\equiv 1$ and $\sigma_{i}\equiv\sigma$.
The general case will be provided in the Appendix. 
In Proposition~\ref{prop:disconnected:centralized}, we showed that
the tail-index for Dis-SGD is smaller than that
of the C-SGD. 
In Theorem~\ref{thm:first:order:alpha}
we showed that when $\delta$ is small, 
the tail-index for the DE-SGD is smaller 
than that of the Dis-SGD given the stepsize $\eta$ or network size $N$
is large.
Therefore, we have the following corollary 
that compares DE-SGD with C-SGD. 

\begin{corollary}\label{cor:DSGD:CSGD}
In the setting of Theorem~\ref{thm:first:order:alpha},
the tail-index $\hat\alpha$ of the decentralized SGD is smaller
than that of the centralized SGD provided that the stepsize $\eta$ or network size $N$
is large and $\delta$ is small.
\end{corollary}

On the other hand, we showed in Theorem~\ref{thm:first:order:alpha} when $\delta$ is small, 
the tail-index for the DE-SGD is larger
than that of the Dis-SGD given the stepsize $\eta$ or network size $N$ is small.
Therefore, our theory (Proposition~\ref{prop:disconnected:centralized} and Theorem~\ref{thm:first:order:alpha})
does not provide a guidance to the comparison of the tail-indexes 
between DE-SGD and C-SGD in this regime. However, in the numerical experiments section, we have always observed that DE-SGD iterates had heavier tails than C-SGD.

\section{Numerical Experiments}
We present our numerical results on both synthetic and real data in this section. Our main goal is to illustrate that DE-SGD iterates admit heavy-tailed distributions in the limit with proper choice of stepsizes $\eta$ and batch-sizes $b$. We will also validate that the behavior of the tail-index in DE-SGD is heavier than that of C-SGD under the assumptions we studied in Section~\ref{sec:theoretical} illustrating our theory. During the experiments, we found that the tail-index from different nodes are very close to each other. For better illustration, we will not report the tail-index of each local node; instead, we regard the median of $N$ tail indices from each node as the tail-index of the DE-SGD algorithm. 


\begin{figure}[t]
\centering
    \hfill
    \subfigure[Case I: $d$=$b$=1, $\sigma$=1, $\sigma_y$=0.2, $N$=30 over complete network.]{
    \includegraphics[width=0.315\columnwidth]{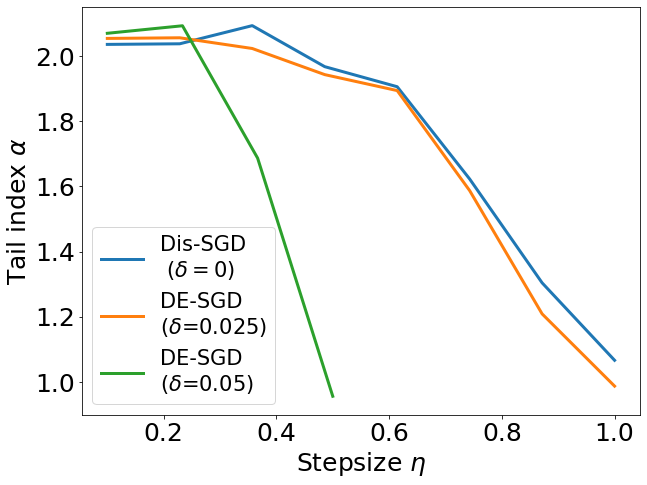}
    \label{fig:tau_eta_eta}
    }
    \hfill
    \subfigure[Case II: $d$=$b$=100, $\sigma$=1, $\sigma_y$=0.2, $N$=10 over complete network.]{
    \includegraphics[width=0.315\columnwidth]{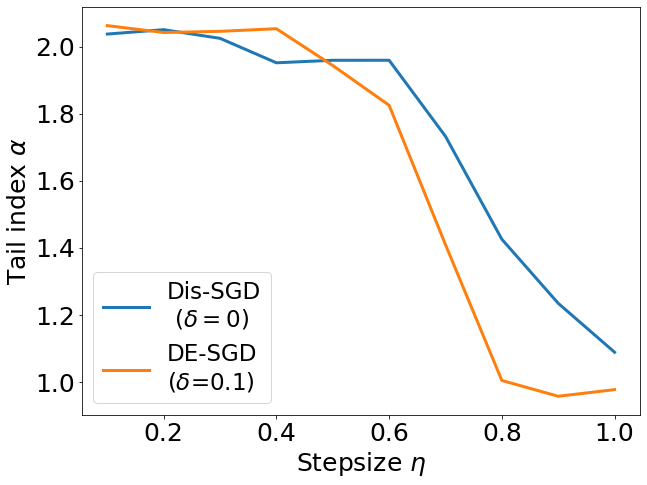}
    \label{fig:eta_tau_eta_d=1}
    }
    \hfill
    \subfigure[Case III: $d$=100, $b$=5, $\sigma$=1, $\sigma_y$=3, $N$=8 over star network.]{
    \includegraphics[width=0.315\columnwidth]{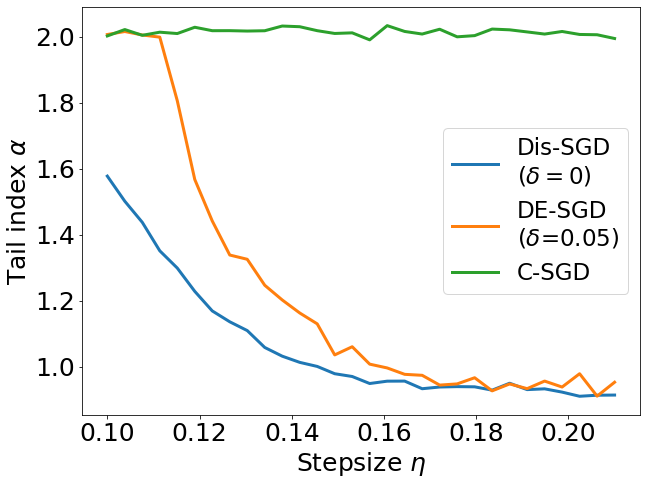}
    \label{fig:eta_eta_tau}
    }
    \caption{Illustration of three main cases on synthetic data}
\end{figure}
\paragraph{5.1 Synthetic experiments.}
In this set of experiments, we consider a simple synthetic setup, where the studied problem is a simple linear regression and each data point is assumed to follow a Gaussian distribution. More specifically, we consider the following model:
\begin{equation}
     a_i \sim \mathcal{N}\left(0,\sigma^2I\right), \quad\epsilon \sim \mathcal{N} \left(0,\sigma_y^2\right), \quad y_i = x^T a_i + \epsilon,
\label{eqn:data_generation}
\end{equation}
where subscript $i$ indicates the node index,  $x, a_i \in \mathbb{R}^d$, $y_i \in \mathbb{R}$ for $i=1,2,\dots,n$ and $\sigma, \sigma_y > 0$. 


\textbf{Tail index estimation.} By the generalized CLT result we obtain in the Appendix, the ergodic average of the iterates converges to an $\alpha$-stable distribution: $\alpha \in (0,2)$ in the case of a stationary distribution with an infinite variance and $\alpha=2$ otherwise. In our experiments, we take the average of the DE-SGD iterates and this enables us to advanced estimators \citep{mohammadi2015estimating} specific to $\alpha$-stable distributions, details are in the Appendix.

    
\paragraph{{Illustrating our Thm. \ref{thm:first:order:alpha} and Cor. \ref{cor:first:order:comparison}}.} Dis-SGD can be analyzed with the existing theory developed for C-SGD. In particular, by the results of \citet{ht_sgd_quad}, there exists an explicit constant $\eta_{\max}$ such that $\hat\rho_{dis}<0$ is equivalent to $\eta \in (0, \eta_{max})$ and for $\eta > \eta_{max}$ the distribution of the iterates diverge. Furthermore, Dis-SGD has a stationary distribution with heavy tails with infinite variance (with $\alpha<2$) only if $\eta \in (\eta_{crit}, \eta_{max})$; otherwise for $\eta \in (0, \eta_{crit}]$ the tail-index is greater than or equal to 2.
Theorem~\ref{thm:first:order:alpha} and Corollary~\ref{cor:first:order:comparison} show that provided that $\hat\rho_{dis}<0$ (i.e. provided that $0 < \eta < \eta_{max}$ ) when stepsize $\eta > \tau := \frac{2}{F_{a}^{-1}(2^{-\frac{1}{N}})}$, the tail gets heavier with the presence of network effect $\delta$, and when $\eta < \tau$, the tail gets lighter with the presence of network effect. If $\eta>\eta_{max}$, then we have divergence. Therefore, there are three main cases depending on the values of $\tau, {\eta}_{crit}$ and ${\eta}_{max}$:

\textbf{\underline{\textit{Case I}}: ($0<\tau<{\eta}_{crit}<{\eta}_{max}$)} In this case, for $\eta \in (0,\tau)$, stationary distribution of Dis-SGD and DE-SGD has a finite variance, so the tail index of the averaged iterates is $\alpha=2$. As stepsize increases, the stationary distribution of both Dis-SGD and DE-SGD will have an infinite variance (with $\alpha<2$), eventually DE-SGD has heavier tails compared to Dis-SGD, due to Thm. \ref{thm:first:order:alpha} and Cor. \ref{cor:first:order:comparison}. 
This is illustrated in Figure~\ref{fig:tau_eta_eta} on a complete network with $N=30$, $d=b=\sigma=1$ and $\sigma_y=0.2$.\looseness=-1

\textbf{\underline{\textit{Case II}}: ($0<{\eta}_{crit}<\tau<{\eta}_{max}$)}
For $\eta \in (0,\eta_{max})$, both Dis-SGD and DE-SGD are expected to converge to a heavy-tailed distribution. While in the small stepsize regime, Dis-SGD will have heavier tail, in the big stepsize regime, DE-SGD will have heavier tail. This is consistent with Thm. \ref{thm:first:order:alpha} and Cor. \ref{cor:first:order:comparison}. Figure~\ref{fig:eta_tau_eta_d=1} illustrates this on a complete network $N=10$, $d=b=100$, $\sigma=1$, $\sigma_y=0.2$.\looseness=-1

\textbf{\underline{\textit{Case III}}: ($0<{\eta}_{crit}<{\eta}_{max}<\tau$)}
In this case, both Dis-SGD and DE-SGD are expected to converge to a heavy-tailed distribution, while Dis-SGD will have heavier tail than DE-SGD. We can observe this phenomenon in Figure~\ref{fig:eta_eta_tau}. Note that although the limit distribution of DE-SGD will have lighter tails compared to Dis-SGD, it still lead to heavier tail compared to C-SGD.


\textbf{{Effect of parameters}.} In the Appendix, we investigate the tail-index $\alpha$ of the stationary distribution of DE-SGD
over different network topologies with varied stepsize $\eta$ and varied batch size $b_i = b$. We found that the tails were monotonic with respect to $\eta$ and $b$ as predicted by Theorem~\ref{thm:mono}. We also found that DE-SGD had heavier tails 
compared to C-SGD in every case. This is also consistent with our Corollary~\ref{cor:DSGD:CSGD}.

\paragraph{5.2 Deep learning experiments.}
This subsection examines the heavy-tail index of the iterates that C-SGD and DE-SGD converge to for neural networks. Tail index is estimated in a similar way to the synthetic section (please see the Appendix for the details).
\paragraph{Fully-connected network on MNIST.}
As a first dataset/architecture, we validate our theory on a three-layer  fully-connected network with MNIST dataset \citep{deng2012mnist}. We calculate the tail-index of C-SGD as well as DE-SGD with different topologies over varied learning rates. In this experiment, we assume there are $N=8$ nodes in the network and batch size is set to $b=5$. The model is trained for 10K iterations and step size $\eta$ ranges from $10^{-4}$ to $7.5\times 10^{-2}$. 
 Figure~\ref{fig:fcn-mnist} shows the tail-index of the parameters in the first layer (the results of other layers are similar). 
From this figure, it is observed that DE-SGD will lead to a heavier tail than both Dis-SGD and C-SGD, the behavior is  consistent with Cor. \ref{cor:DSGD:CSGD}. 
\begin{wrapfigure}{r}{0.7\textwidth}
\subfigure[FCN on MNIST.]{
    \includegraphics[width=0.315\columnwidth]{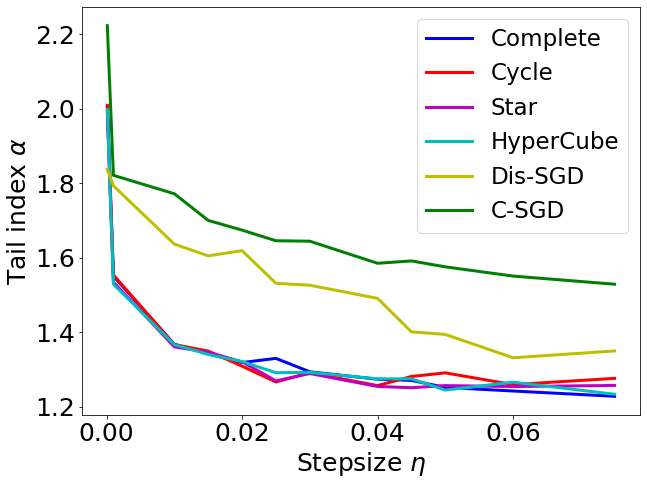}
    \label{fig:fcn-mnist}
    }
    \subfigure[ResNet-20 on CIFAR10.]{
    \includegraphics[width=0.315\columnwidth]{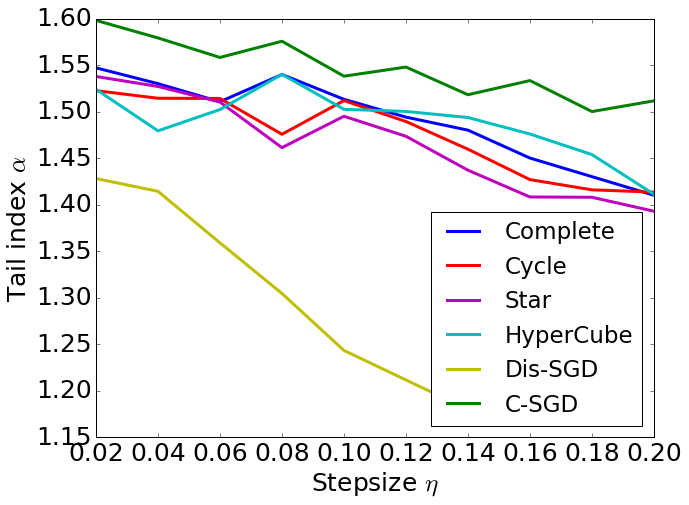}
    \label{fig:resnet_24n16b}
    }
    \caption{Tail-index $\alpha$ for different setting on MNIST and CIFAR10.}
    
\end{wrapfigure}
\paragraph{ResNet-20 on CIFAR10.}
As a second (larger-scale) experiment, we evaluate our theory on ResNet-20 model \citep{he2016deep} with CIFAR10 dataset \citep{krizhevsky2009learning}. 
We compare the heavy-tail index between C-SGD and DE-SGD over different network topologies  
with $N=24$. Further details of the experimental setup can be found in the Appendix.  
Figure~\ref{fig:resnet_24n16b} illustrates the tail-index comparison between C-SGD and DE-SGD. It is observed that DE-SGD for all networks is heavier tailed than C-SGD but lighter tailed than Dis-SGD.  In the Appendix, we also obtained similar results for $N=8$. These results are consistent with the behavior of Region I reported in Fig. \ref{fig:probability-contour-100}. We cannot tune stepsize $\eta$ larger than $0.2$ since it will significantly degrade the test accuracy of ResNet-20 and result in instability issues for the iterates. For this reason, we do not observe the behavior of Region II which requires large stepsizes.

\looseness=-1 
\section{Conclusion}
Existing works about the heavy-tails for centralized SGD (C-SGD) do not apply to the decentralized setting. We provided theoretical and numerical results showing that heavier tails arise in DE-SGD setting compared to C-SGD. We also compare DE-SGD to disconnected SGD where nodes distribute the data but do not communicate. Our theory and experiments uncover an interesting interplay between the tails and the network structure: we identify two regimes of parameters (stepsize and network size), where DE-SGD 
can have lighter or heavier tails than disconnected SGD depending on the regime. Finally, we provided numerical experiments on synthetic data and neural networks which illustrate the results. \looseness=-1

\section*{Acknowledgements}
We thank Dr. Rong Jin for the stimulating discussions and for all the helpful feedback about our paper. Mert G\"{u}rb\"{u}zbalaban and Yuanhan Hu's research are supported in part by the grants Office of Naval Research Award Number
N00014-21-1-2244, National Science Foundation (NSF)
CCF-1814888, NSF DMS-2053485.  Umut \c{S}im\c{s}ekli's research is supported by the French government under management of Agence Nationale de la Recherche as part of the ``Investissements d’avenir'' program, reference ANR-19-P3IA-0001 (PRAIRIE 3IA Institute). 
Lingjiong Zhu is grateful to the partial support from a Simons Foundation Collaboration Grant
and the grant NSF DMS-2053454 from the National Science Foundation.

\bibliographystyle{plainnat}
\bibliography{decentralized,heavy}

\newpage
\appendix


\begin{center}

\Large \bf Heavier-Tail Phenomenon in Decentralized SGD \vspace{3pt}\\ {\normalsize APPENDIX}

\end{center}

\section{Illustration of the network architectures}\label{sec-network-drawings}
Figure \ref{fig:network} illustrates the five network structures we have chosen for our experiments. 
\begin{figure}[htb]
\centering
    \subfigure[Complete]{
    \includegraphics[width=0.25\columnwidth]{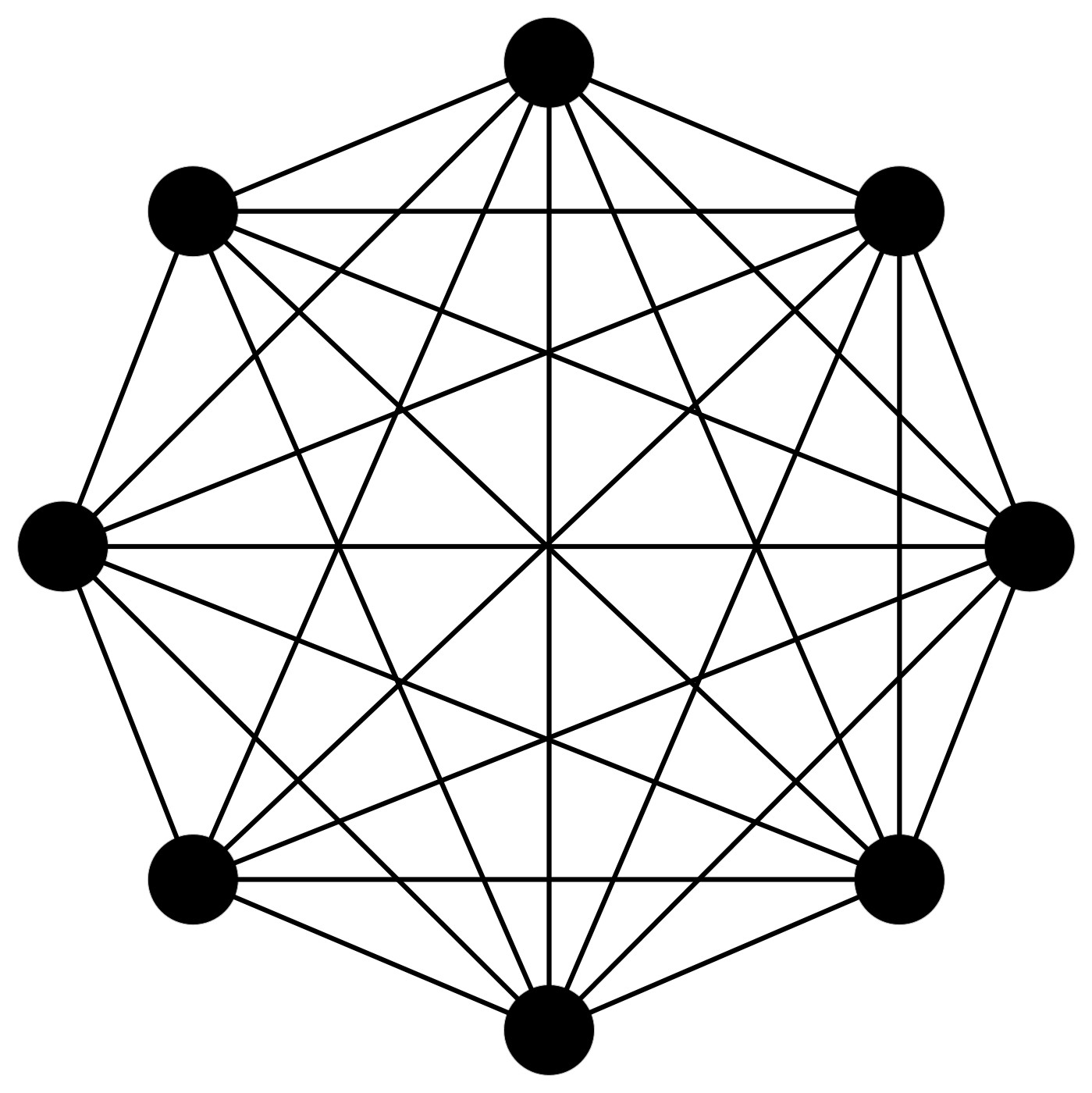}
    }
    \hfill
    \subfigure[Star]{
    \includegraphics[width=0.25\columnwidth]{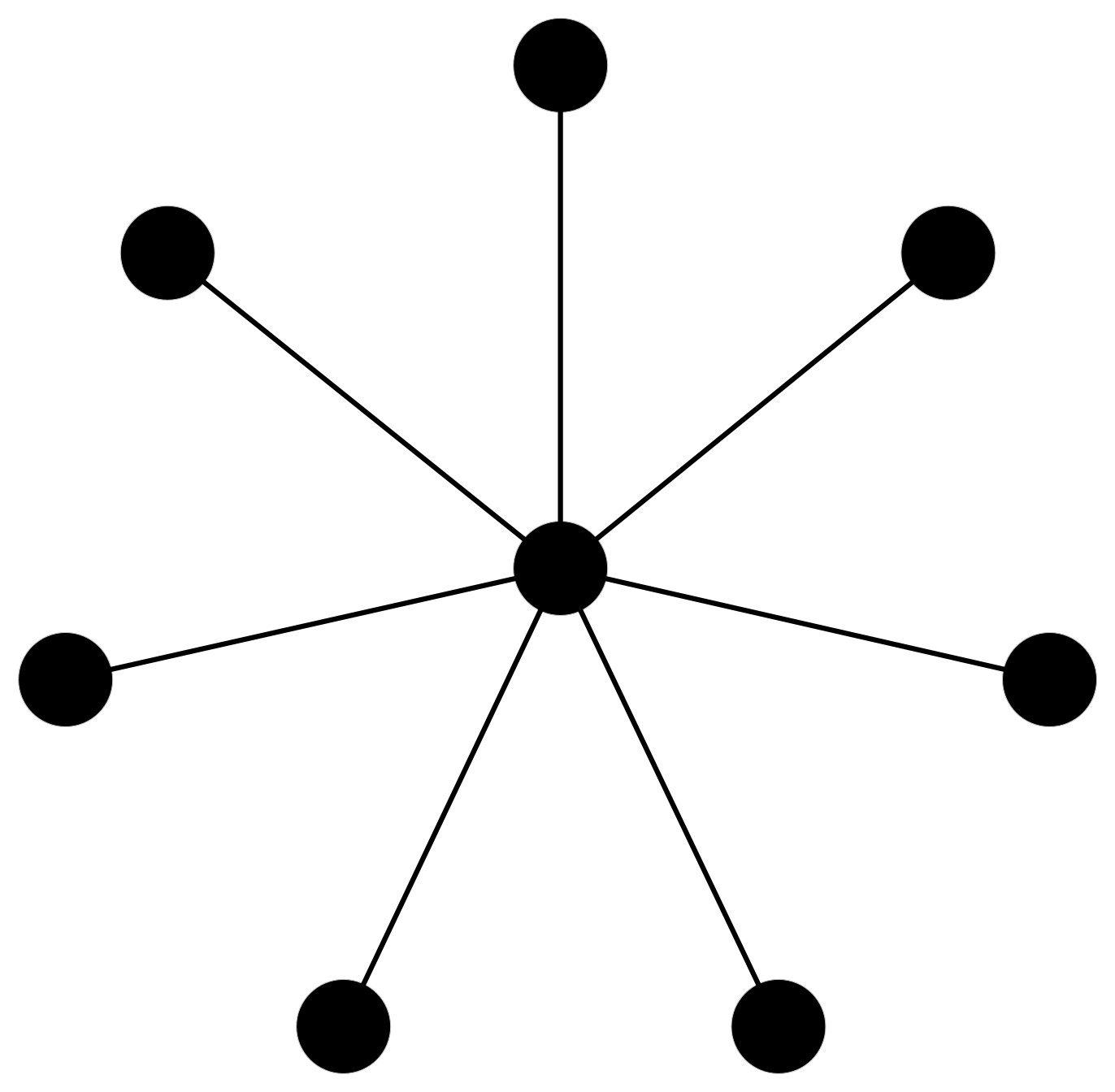}
    }
    \hfill
    \subfigure[Cycle]{
    \includegraphics[width=0.25\columnwidth]{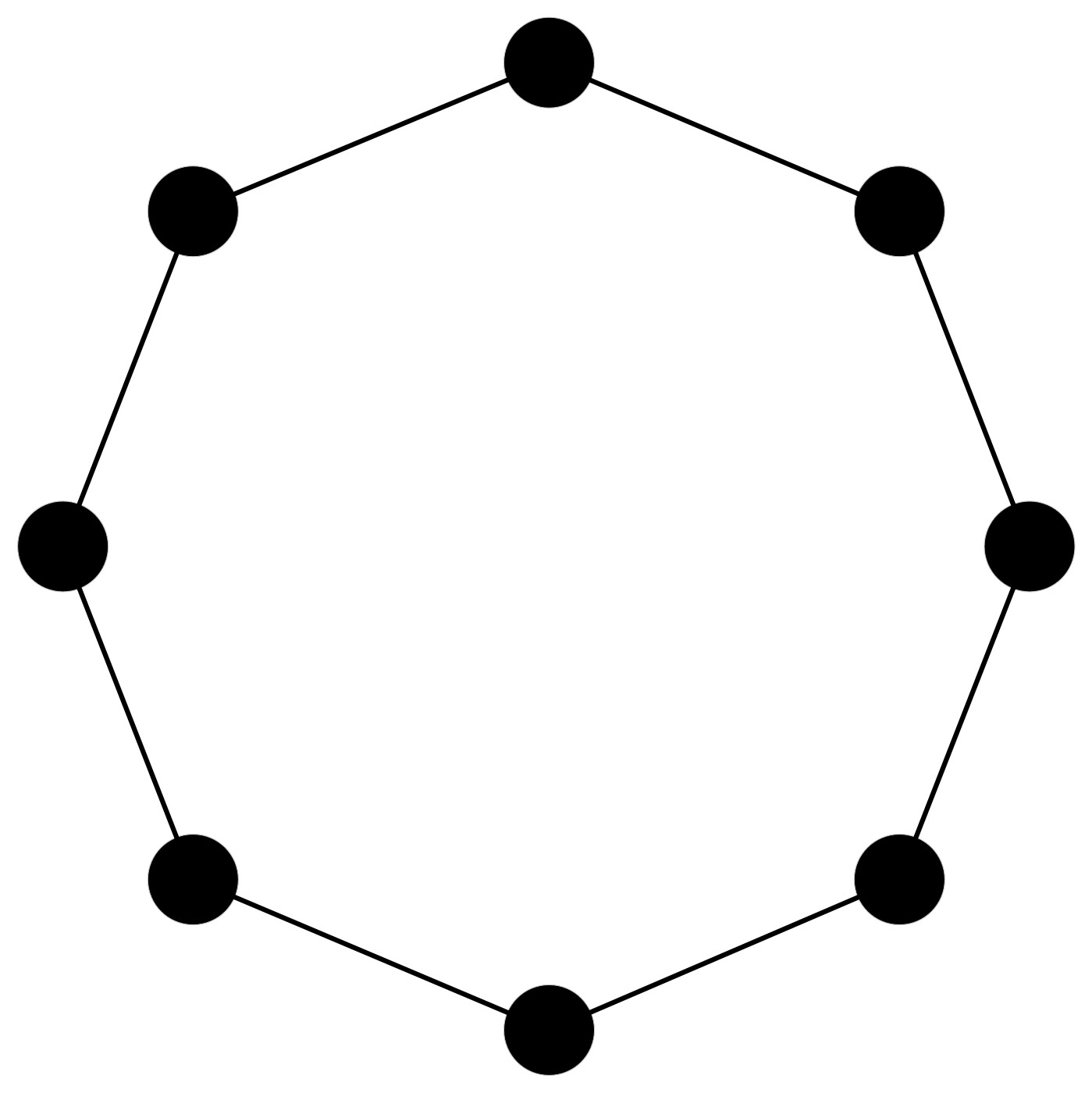}
    }
    
    \subfigure[Hypercube]{
    \includegraphics[width=0.25\columnwidth]{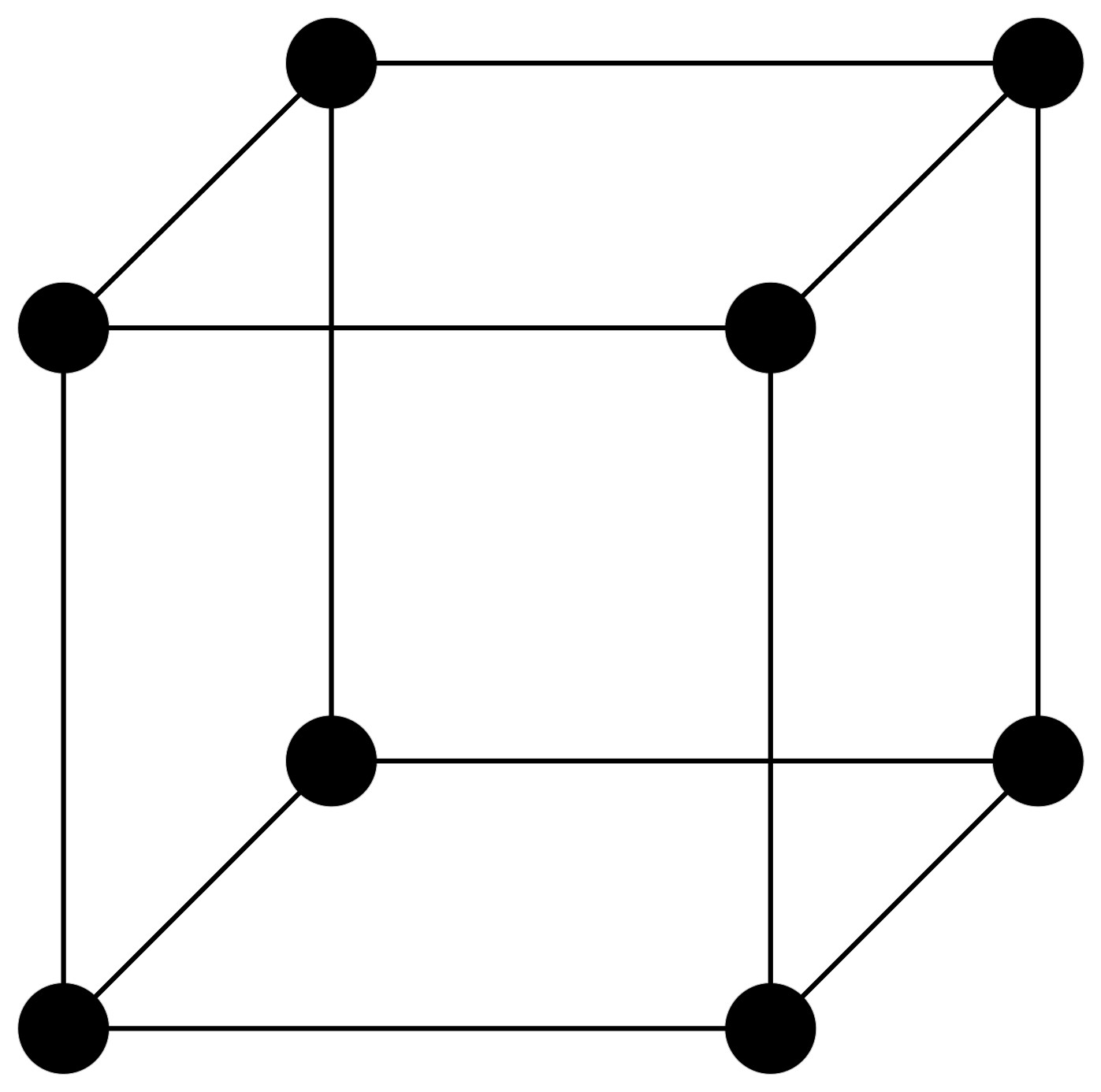}
    }
    \quad \quad \quad
    \subfigure[Bipartite]{
    \includegraphics[width=0.175\columnwidth]{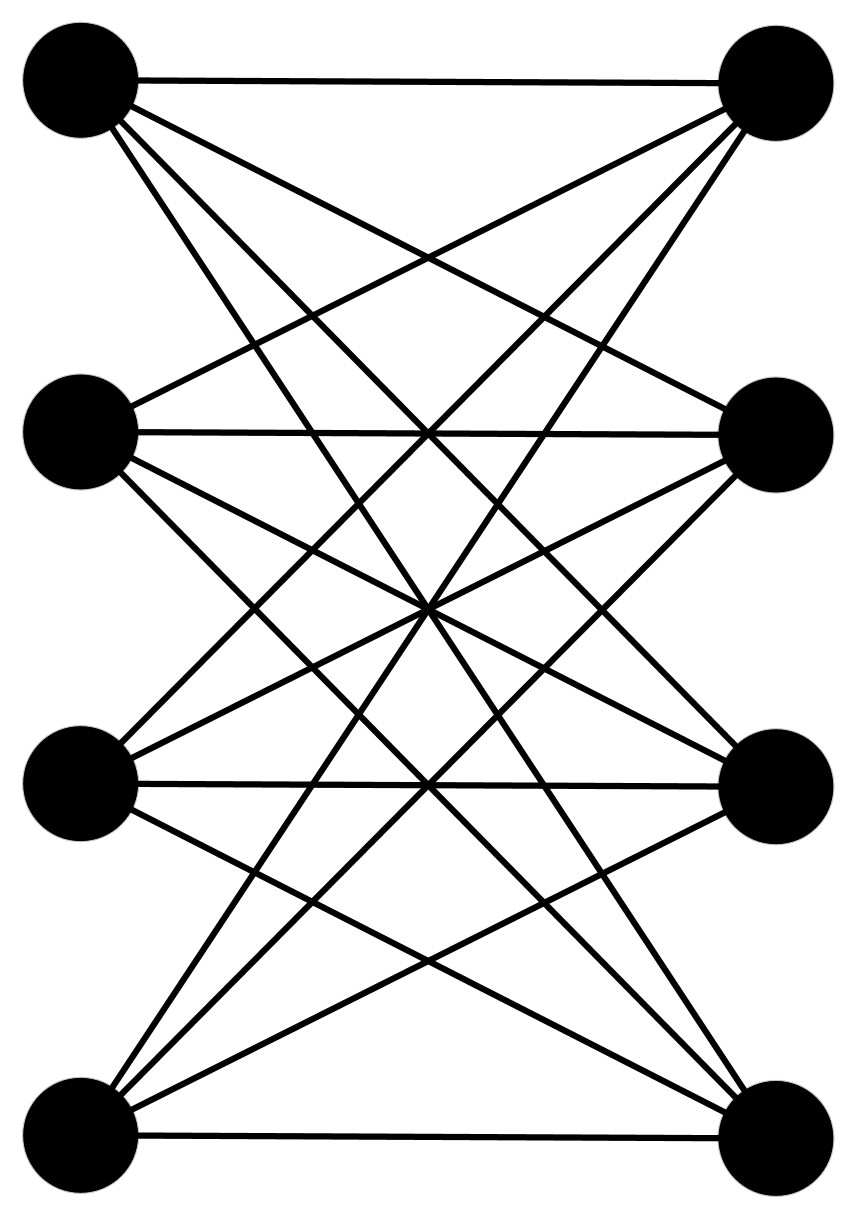}
    }

\caption{Illustration of the network architectures.}
\label{fig:network}
\end{figure}

\section{Effect of Network Structure}\label{sec:examples}
\textbf{Communication matrix.}
We use the convention that a node is a neighbor of itself, i.e. $i \in \Omega_i$. We assume that $W$ is doubly stochastic, i.e. $W_{ij}=W_{ji}>0$ if $j\in\Omega_{i}$,
and $W_{ij}=W_{ji}=0$ if $j\not\in \Omega_{i}$ and $i\neq j$,
and finally $W_{ii}=1-\sum_{j\neq i}W_{ij}>0$
for every $1\leq i\leq N$.
It is known that the eigenvalues of a doubly stochastic matrix $W$ can be ordered
in a descending manner satisfying
$1=\lambda_{1}>\lambda_{2}\geq\cdots\geq\lambda_{N}>-1$,
where the largest eigenvalue is $1$ with an all-one eigenvector.
The eigenvalues $(\lambda_{j})_{j=1}^{N}$ of $W$ can be used to reveal the properties of the network associated with the weight matrix $W$ (see e.g. \citep{chung1997spectral}). 
A common approach is to take $W = I - \delta L$ where $I$ is the identity, $L$ is the graph Laplacian and $\delta>0$ small enough \citep{olfati2007consensus}.
\subsection{Examples}

\textbf{Adjacency matrix and degree matrix:}
For an undirected graph with $N$ nodes $i=1,2,\ldots,N$,
an adjacency matrix $A=(A_{ij})_{1\leq i,j\leq N}$ is
a symmetric matrix with $A_{ij}=1$
if there is an edge between node $i$ and node $j$
and $A_{ij}=0$ otherwise.
The degree matrix $D_{\text{deg}}$ of an undirected graph is a diagonal matrix
with the diagonals counting the degrees of the nodes, 
more precisely, with the $i$-th diagonal being $\sum_{j=1}^{N}A_{ij}$.
See e.g. \citep{chung1997spectral}.

\textbf{Graph Laplacian:}
The graph Laplacian for an undirected graph with $N$ nodes has the form 
$$ L = D_{deg} - A, $$
where $D_{deg}$ is the degree matrix and $A$ is the adjacency matrix
which is symmetric; see e.g. \citep{cvetkovic}.
The eigenvalues of $L$ are non-negative:
$0=\lambda_{1}^{L}\leq\lambda_{2}^{L}\leq\cdots\leq\lambda_{N}^{L}$
(see e.g. \citep{chung1997spectral}).
For any $0<\delta<2/\lambda_{N}^{L}$, 
we take the communication matrix as:
\begin{equation}
W=I_{N}-\delta L.
\end{equation}
It follows that $W$ is symmetric doubly stochastic
with eigenvalues $1=\lambda_{1}>\lambda_{2}\geq\cdots\geq\lambda_{N}>-1$,
where $\lambda_{i}=1-\delta\lambda_{i}^{L}$
for any $i=1,2,\ldots,N$.

\textbf{Complete graph with $N$ nodes:}
For the complete graph $K_{N}$ on $N$ vertices,
the adjacency matrix $A=(A_{ij})_{1\leq i,j\leq N}$,
where $A_{ii}=0$ and $A_{ij}=A_{ji}=1$ for any $i\neq j$.
One can compute that the Laplacian is $L=(L_{ij})_{1\leq i,j\leq N}$
where $L_{ii}=N-1$ and $L_{ij}=-1$ for any $i\neq j$.
In this case, the communication matrix is given by
\begin{equation}
W=I_{N}-\delta L=
\left[
\begin{array}{ccccc}
1-\delta(N-1) & \delta & \cdots & \delta & \delta
\\
\delta & 1-\delta(N-1) & \delta & \cdots & \delta
\\
\vdots & \vdots & \vdots & \vdots & \vdots
\\
\delta & \cdots & \delta & 1-\delta(N-1) & \delta
\\
\delta & \delta & \cdots & \delta & 1-\delta(N-1)
\end{array}
\right].
\end{equation}

\textbf{Star graph with $N$ nodes:}
For the star graph $S_{N}$ on $N$ vertices,
the adjacency matrix $A=(A_{ij})_{1\leq i,j\leq N}$,
where $A_{11}=0$ and $A_{1j}=A_{j1}=1$ for any $j=2,3,\ldots,N$ and $A_{ij}=0$ for any $i,j\geq 2$.
One can compute that the Laplacian is $L=(L_{ij})_{1\leq i,j\leq N}$
where $L_{11}=N-1$, $L_{1j}=L_{j1}=-1$ for any $2\leq i\neq j\leq N$, $L_{22}=\cdots=L_{NN}=1$ and $L_{ij}=0$
otherwise.
In this case, the communication matrix is given by
\begin{equation}
W=I_{N}-\delta L=
\left[
\begin{array}{cccccc}
1-\delta(N-1) & \delta & \delta & \cdots & \delta & \delta
\\
\delta & 1-\delta & 0 & 0 & \cdots & 0
\\
\delta & 0 & 1-\delta & 0 & \cdots & 0
\\
\vdots & \vdots & \vdots & \vdots  & \vdots & \vdots
\\
\delta & 0 & \cdots & 0 & 1-\delta & 0
\\
\delta & 0 & \cdots & \cdots & 0 & 1-\delta
\end{array}
\right].
\end{equation}

\textbf{Cycle graph with $N$ nodes:}
For the cycle graph $C_{N}$ on $N$ vertices,
the adjacency matrix $A=(A_{ij})_{1\leq i,j\leq N}$,
where $A_{ij}=A_{ji}=1$ for any $|i-j|=1$,
$A_{1N}=A_{N1}=1$ and $A_{ij}=0$ otherwise.
One can compute that the Laplacian is $L=(L_{ij})_{1\leq i,j\leq N}$
where $L_{ii}=2$ for any $i=1,\ldots,N$, 
and $L_{ij}=-1$ for any $|i-j|=1$, $L_{1N}=L_{N1}=-1$,
and $L_{ij}=0$ otherwise.
In this case, the communication matrix is given by
\begin{equation}
W=I_{N}-\delta L=
\left[
\begin{array}{cccccc}
1-2\delta & \delta & 0 & \cdots & 0 & \delta
\\
\delta & 1-2\delta & \delta & 0 & \cdots & 0
\\
0 & \delta & 1-2\delta & \delta & \cdots & 0
\\
\vdots & \vdots & \vdots & \vdots  & \vdots & \vdots
\\
0 & 0 & \cdots & \delta & 1-2\delta & \delta
\\
\delta & 0 & \cdots & \cdots & \delta & 1-2\delta
\end{array}
\right].
\end{equation}

\textbf{Hypercube graph with $N=2^{n}$ nodes:}
Consider the hypercube graph $Q_{n}$ with $N=2^{n}$ nodes. 
The hypercube has adjacency matrix:
\begin{equation}
A_{Q_{1}}=\left[\begin{array}{cc}
0 & 1
\\
1 & 0
\end{array}
\right],
\qquad
A_{Q_{n}}=\left[\begin{array}{cc}
A_{Q_{n-1}} & I_{2^{n-1}}
\\
I_{2^{n-1}} & A_{Q_{n-1}}
\end{array}
\right],
\end{equation}
and the Laplacian matrix:
\begin{equation}
L_{Q_{1}}=\left[\begin{array}{cc}
1 & -1
\\
-1 & 1
\end{array}
\right],
\qquad
L_{Q_{n}}=\left[\begin{array}{cc}
L_{Q_{n-1}}+I_{2^{n-1}} & -I_{2^{n-1}}
\\
-I_{2^{n-1}} & L_{Q_{n-1}}+I_{2^{n-1}}
\end{array}
\right].
\end{equation}

\textbf{Bipartite graph with $N$ nodes:}
We consider the complete bipartite graph $K_{\frac{N}{2},\frac{N}{2}}$ on $N$ vertices,
where $N$ is an even integer whose vertices are divided into two equal size sets with $N/2$ vertices,
and the adjacency matrix $A=(A_{ij})_{1\leq i,j\leq N}$ is given by
\begin{equation}
A=\left[
\begin{array}{cc}
0_{\frac{N}{2}\times\frac{N}{2}} & 1_{\frac{N}{2}\times\frac{N}{2}}
\\
1_{\frac{N}{2}\times\frac{N}{2}} & 0_{\frac{N}{2}\times\frac{N}{2}}
\end{array}
\right].
\end{equation}
One can compute that the Laplacian is given by $L=\frac{N}{2}I_{N}-A$.
In this case, the communication matrix is given by
\begin{equation}
W=I_{N}-\delta L=
\left[
\begin{array}{cc}
\left(1-\frac{\delta N}{2}\right)I_{\frac{N}{2}} & \delta 1_{\frac{N}{2}\times\frac{N}{2}}
\\
\delta 1_{\frac{N}{2}\times\frac{N}{2}} & \left(1-\frac{\delta N}{2}\right)I_{\frac{N}{2}}
\end{array}
\right].
\end{equation}

\textbf{Barbell graph with $N=2n$ nodes:}
We consider the barbell graph on $N=2n$ vertices,
where $N$ is an even integer whose vertices are divided into two equal size sets with $N/2$ vertices,
and the adjacency matrix $A=(A_{ij})_{1\leq i,j\leq N}$ is given by
$A_{ij}=1$ for any $1\leq i\neq j\leq n$,
$A_{ij}=1$ for any $n+1\leq i\neq j\leq 2n$, 
$A_{1,n+1}=A_{n+1,1}=1$
and $A_{ij}=0$ otherwise.
We can compute that the Laplacian is given by
$L=(L_{ij})_{1\leq i,j\leq N}$, 
where $L_{ij}=-1$ for any $1\leq i\neq j\leq n$,
$L_{ij}=-1$ for any $n+1\leq i\neq j\leq 2n$, 
$L_{1,n+1}=L_{n+1,1}=-1$, 
$L_{11}=L_{n+1,n+1}=n$,
$L_{ii}=n-1$ for any $i\neq 1,n+1$
and $L_{ij}=0$ otherwise.

\textbf{Path graph with $N$ nodes:}
For the path graph (line graph) $P_{N}$ on $N$ vertices, the
adjacency matrix $A=(A_{ij})_{1\leq i,j\leq N}$, 
where $A_{i,i+1}=A_{i,i-1}=1$ for any $i=2,\ldots,N-1$
$A_{1,2}=1$ and $A_{N,N-1}=1$ and $A_{ij}=0$ otherwise.
We can compute that the Laplacian is given by $L=(L_{ij})_{1\leq i,j\leq N}$,
where $L_{11}=L_{NN}=1$, $L_{ii}=2$ for any $i=2,\ldots,N-1$,
$L_{i,i+1}=L_{i,i-1}=-1$ for any $i=2,\ldots,N-1$
$L_{1,2}=-1$ and $L_{N,N-1}=-1$ and $L_{ij}=0$ otherwise.


\section{Generalized Central Limit Theorem}\label{sec:CLT}

In this section, 
by following Theorem~\ref{thm:main:dsg}, we derive a generalized central limit theorem (GCLT) as a corollary to \cite[Theorem 1.1]{gao2015stable}. 

Before we state the next result, let us first recall 
\emph{$\alpha$-stable distributions} from probability theory.
The $\alpha$-stable distributions are heavy-tailed
distributions
that appear as the limiting distribution of the generalized CLT for
a sum of i.i.d. random variables
with infinite variance \citep{paul1937theorie}.
Here, $\alpha\in(0,2]$ is known as the tail-index, which determines
the tail thickness of the distribution.
An $\alpha$-stable distribution becomes
heavier-tailed as $\alpha$ gets smaller.
When $\alpha=1$ and $\alpha=2$, an $\alpha$-stable distribution reduces to the Cauchy and the Gaussian distributions, respectively.

In particular, the next result shows that, when properly scaled, the sum of the iterates $S_{K}:=\sum_{k=1}^{K} x^{(k)}$ converges in law to an $\alpha$-stable distribution. 

\begin{corollary}
Assume the conditions of Theorem \ref{thm:main:dsg} are satisfied, i.e. assume 
$\rho<0$ {and there exists a unique positive $\alpha$ such that $h(\alpha)=1$}. Then, there exists a function $C_\alpha : \mathbb{S}^{Nd-1} \mapsto \mathbb{C}$ such that for $\alpha \in (0, \infty) \setminus \{1\} $, the random variables 
$\mathbf{a}_K \left(S_{K}-\mathbf{d}_{K}\right)$ converge in law to the random variable with
characteristic function
$\Upsilon_{\alpha}(t v)=\exp \left(t^{\min(\alpha, 2)} C_{\alpha}(v)\right)$, for $t>0$ and $v \in \mathbb{S}^{d-1}$ as $K\rightarrow\infty$.

Here, if $\alpha \in (0,1)$, $\mathbf{a}_K = K^{-1/\alpha}$, $\mathbf{d}_K = 0$; if $\alpha \in (1,2)$, $\mathbf{a}_K = K^{-1/\alpha}$, $\mathbf{d}_K = K^{1-1/\alpha}\bar{x}$; if $\alpha = 2$, $\mathbf{a}_K = (K \log K)^{-1/2}$, $\mathbf{d}_K = K \bar{x}$; finally, if $\alpha > 2$, $\mathbf{a}_K = K^{-1/2}$, $\mathbf{d}_K = K \bar{x}$, where $\bar{x}= \mathbb{E} [x^{(\infty)}]$.
\end{corollary}
The proof requires verifying that the conditions of \cite[Theorem 1.1]{gao2015stable} hold under our assumptions. Similar to Theorem~\ref{thm:main:dsg}, the proof is a straightforward extension of the proof of \cite[Theorem 2]{ht_sgd_quad}, hence omitted. For $\alpha=1$, a similar GCLT holds with more complicated terms, which we omitted for brevity. 

This result shows that, when $\alpha \geq 2$, the scaled and centered sum $\mathbf{a}_K \left(S_{K}-\mathbf{d}_{K}\right)$ converges to a Gaussian distribution (that is a $2$-stable distribution); whereas when $\alpha <2$, the sequence converges to an $\alpha$-stable distribution, where $\alpha$ coincides with the tail-index of the stable distribution as well.


\section{Moment Bounds and Convergence Speed}\label{sec:conv}

\textbf{Wasserstein metric.}
For any $p\geq 1$, define $\mathcal{P}_{p}(\mathbb{R}^{d})$
as the space consisting of all the Borel probability measures $\nu$
on $\mathbb{R}^{d}$ with the finite $p$-th moment
(based on the Euclidean norm).
For any two Borel probability measures $\nu_{1},\nu_{2}\in\mathcal{P}_{p}(\mathbb{R}^{d})$, 
we define the standard $p$-Wasserstein
metric \citep{villani2008optimal}:
$$\mathcal{W}_{p}(\nu_{1},\nu_{2}):=\left(\inf\mathbb{E}\left[\Vert Z_{1}-Z_{2}\Vert^{p}\right]\right)^{1/p},$$
where the infimum is taken over all joint distributions of the random variables $Z_{1},Z_{2}$ with marginal distributions
$\nu_{1},\nu_{2}$.

\subsection{Moment Bounds}

Theorem~\ref{thm:main:dsg} is of asymptotic nature which characterizes the stationary distribution $x^{(\infty)}$ of
SGD iterations.
Next, we provide non-asymptotic moment
bounds for $x^{(k)}$ at each $k$-th iterate,
and also for the limit $x^{(\infty)}$.
In the rest of Appendix~\ref{sec:conv}, we 
assume that 
\begin{equation}\label{assump:hat:rho}
\hat{\rho}:=\mathbb{E}\log\left\Vert\mathcal{W}-\eta H\right\Vert<0,
\end{equation}
so that there exists a unique positive value $\hat{\alpha}$ such that $\hat{h}(\hat{\alpha})=1$. 
We have shown previously in Section~\ref{sec:theoretical} that a sufficient condition for
\eqref{assump:hat:rho} to hold is
$b\geq d$, $\hat{\rho}_{dis}:=\mathbb{E}\log\Vert I_{Nd}-\eta H\Vert<0$ and $\delta$ being sufficiently small.

\begin{theorem}\label{thm:moments}
Suppose Assumptions \textbf{(A1)}-\textbf{(A2)} hold.

(i) If the tail-index $\hat{\alpha}\leq 1$, 
then for any $p\in(0,\hat{\alpha})$, we have $\hat{h}(p)<1$ and
\begin{equation}
\mathbb{E}\left\Vert x^{(k)}\right\Vert^{p}
\leq
\left(\hat{h}(p)\right)^{k}\mathbb{E}\left\Vert x^{(0)}\right\Vert^{p}
+\frac{1-(\hat{h}(p))^{k}}{1-\hat{h}(p)}\mathbb{E}\left\Vert q^{(1)}\right\Vert^{p}.
\end{equation}

(ii) If the tail-index $\hat{\alpha}>1$, 
then for any $p\in(1,\hat{\alpha})$, we have $\hat{h}(p)<1$
and for any $0<\epsilon<\frac{1}{\hat{h}(p)}-1$, we have
\begin{equation}
\mathbb{E}\left\Vert x^{(k)}\right\Vert^{p}
\leq
((1+\epsilon)\hat{h}(p))^{k}\mathbb{E}\left\Vert x^{(0)}\right\Vert^{p}
+\frac{1-((1+\epsilon)\hat{h}(p))^{k}}{1-(1+\epsilon)\hat{h}(p)}
\frac{(1+\epsilon)^{\frac{p}{p-1}}-(1+\epsilon)}{((1+\epsilon)^{\frac{1}{p-1}}-1)^{p}}
\mathbb{E}\left\Vert q^{(1)}\right\Vert^{p}.
\end{equation}
\end{theorem}

By letting $k\rightarrow\infty$, we obtain the following corollary.

\begin{corollary}\label{cor:moments}
Suppose Assumptions \textbf{(A1)}-\textbf{(A3)} hold.

Suppose Assumptions \textbf{(A1)}-\textbf{(A2)} hold.

(i) If the tail-index $\hat{\alpha}\leq 1$, 
then for any $p\in(0,\hat{\alpha})$, we have $\hat{h}(p)<1$ and
\begin{equation}
\mathbb{E}\left\Vert x^{(\infty)}\right\Vert^{p}
\leq
\frac{1}{1-\hat{h}(p)}\mathbb{E}\left\Vert q^{(1)}\right\Vert^{p}.
\end{equation}

(ii) If the tail-index $\hat{\alpha}>1$, 
then for any $p\in(1,\hat{\alpha})$, we have $\hat{h}(p)<1$
and for any $0<\epsilon<\frac{1}{\hat{h}(p)}-1$, we have
\begin{equation}
\mathbb{E}\left\Vert x^{(\infty)}\right\Vert^{p}
\leq
\frac{1}{1-(1+\epsilon)\hat{h}(p)}
\frac{(1+\epsilon)^{\frac{p}{p-1}}-(1+\epsilon)}{((1+\epsilon)^{\frac{1}{p-1}}-1)^{p}}
\mathbb{E}\left\Vert q^{(1)}\right\Vert^{p}.
\end{equation}
\end{corollary}

\subsection{Convergence Speed}

Next, we will study the speed of convergence of the $k$-th iterate $x^{(k)}$ to its stationary distribution $x^{(\infty)}$ in the Wasserstein metric $\mathcal{W}_{p}$ for any $1\leq p<\alpha$. 

\begin{theorem}\label{thm:conv}
Suppose Assumptions \textbf{(A1)}-\textbf{(A2)} hold.
Assume $\hat{\alpha}>1$.
Let $\nu^{(k)}$, $\nu^{(\infty)}$ denote the probability laws
of $x^{(k)}$ and $x^{(\infty)}$ respectively. 
Then for any $1\leq p<\hat{\alpha}$,
\begin{align}
\mathcal{W}_{p}\left(\nu^{(k)},\nu^{(\infty)}\right)\nonumber
\leq
\left(\hat{h}(p)\right)^{k/p}\mathcal{W}_{p}\left(\tilde{\nu}^{(0)},\tilde{\nu}^{(\infty)}\right),
\end{align}
where $(\hat{h}(p))^{1/p}\in(0,1)$.
\end{theorem}

We can see from Theorem~\ref{thm:conv} that $\nu^{(k)}$
converges exponentially fast in $k$ to $\nu^{(\infty)}$
with the convergence rate $(\hat{h}(p))^{1/p}$.
Theorem~\ref{thm:conv} works for any $p<\hat{\alpha}$.
At the critical $p=\hat{\alpha}$,
by adapting the proof of Proposition~10 in \citet{ht_sgd_quad}, one can show
the following result.

\begin{proposition}\label{prop:k:bound}
Suppose Assumptions \textbf{(A1)}-\textbf{(A2)} hold.
Then we have $\mathbb{E}\left[\left\Vert x^{(k)}\right\Vert^{\hat{\alpha}}\right]=O(k)$
if $\hat{\alpha}\leq 1$, and
$\mathbb{E}\left[\left\Vert x^{(k)}\right\Vert^{\hat{\alpha}}\right]=O\left(k^{\hat{\alpha}}\right)$
if $\hat{\alpha}>1$.
\end{proposition}


\section{Tail-Index Comparison}\label{sec:general:case} 

\subsection{Tail-Index Comparison between Decentralized SGD and Disconnected SGD: The General Case}

In Section~\ref{sec:theoretical}, 
we compared the tail-index between decentralized SGD
and disconnected SGD for the special case when
the dimension $d=1$, the batch-sizes $b_{i}\equiv 1$. 
In this section, we will discuss the general case. 

\subsubsection{General dimension $d$, batch-sizes $b_{i}\equiv b$ and variances $\sigma_{i}^{2}\equiv\sigma^{2}$}

More generally, we consider the general dimension $d$, batch-sizes $b_{i}\equiv b$
and $\sigma_{i}\equiv\sigma$.
Let $u_i$ and $v_i$
be the right and left singular vectors of the matrix $I_{d} - \eta H_i$, i.e. they satisfy
\begin{equation} 
(I_{d} - \eta H_i)v_i = \sigma_{\max,i} u_i, \quad u_i^T (I_{d} - \eta H_i) = \sigma_{\max,i}v_i^T,
\end{equation}
where 
$$
\sigma_{\max,i} =  \| I_{d} - \eta H_i \|.
$$
Then, we can choose the left and right eigenvectors of the matrix
$I_{Nd}-\eta H = \mbox{blkdiag}_i(I_d - \eta H_i)$
corresponding to the largest singular value
\begin{equation} 
\sigma_{\max} = \max_{1\leq i \leq N} \sigma_{\max,i}
\label{eq-sigma-max}
\end{equation}
as
$$
v = v_{i_*} \otimes e_{i_*}, \quad u =  \mbox{sign} \left(1-\eta\lambda_{j_*(i_*)}(H_{i_*})\right) 
v_{i_*} \otimes e_{i_*},  
$$
where $i_*$ is the index that maximizes $\sigma_{\max,i}$ in \eqref{eq-sigma-max} and
$$j_*(i) := \arg \max_{1\leq j\leq d} \left| 1- \eta \lambda_{j}(H_{i})\right|$$
is the index of the eigenvalues of $H_{i_*}$ so that we have 
\beq 
\|I_d - \eta H_{i_*}\| = \left| 1- \eta \lambda_{j_*(i_{\ast})}(H_{i_*}) \right|.
\label{def-j-star}
\eeq
Basically, $i_* \in \{1,2,\dots,N\}$ is the index of the node where $\sigma_{\max,i}$ is maximized whereas $j_*(i_{\ast}) \in \{1,2,\dots,d\}$ is the index of the eigenvalue of $H_{i_*}$ for which the 2-norm $\|I_d - \eta H_{i_*}\|$ is attained (i.e. \eqref{def-j-star} is satisfied).
For example, if $N=3$ and $i_* = 2$; we have
$$ v = \begin{bmatrix} 0_d^T & v_2^T & 0_d^T \end{bmatrix}^T\,,$$
where $v_2$ is the left eigenvector of $I_d - \eta H_2$ and $0_d$ denotes the zero vector of length $d$. As before,
\begin{align}
\left\| I_{Nd} - \eta H - \delta (L\otimes I_d) \right\|^s 
&= \left\| I_{Nd} - \eta H \right\|^s - s \delta u^T ( L \otimes I_d) v  +  o(\delta)\nonumber
\\
&= \left\| I_{Nd} - \eta H \right\|^s - s\delta \cdot\mbox{sign} \left(1-\eta\lambda_{j_*(i_{\ast})}(H_{i_*})\right) v_{i_*}^T L_{i_* i_*} v_{i_*} +  o(\delta)\nonumber
\\
&=\left\| I_{Nd} - \eta H \right\|^s - s\delta \cdot\mbox{sign} \left(1-\eta\lambda_{j_*(i_{\ast})}(H_{i_*})\right)  L_{i_* i_*} +  o(\delta)\,,
\label{without:expectation-2} 
\end{align}
where we used $\|v_{i_*}\|=1$. In the special case when $d=1$ and $b=1$, we have $$\mbox{sign} \left(1-\eta\lambda_{j_*}(H_{i_*})\right)  = \mbox{sign}\left(1-\eta a_{i_*}^2\right),$$ so that we recover \eqref{without:expectation}.

\begin{lemma}\label{lem:general:d:b}
As $\delta\rightarrow 0$, we have the first-order expansion:
\begin{equation}
\hat{h}(s)
=\hat{h}_{dis}(s)
-\frac{s\delta\sum_{i=1}^{N}L_{ii}}{N}\left(2\mathbb{P}\left(\lambda_{j_*(i_{\ast})}(H_{i_*})<1/\eta\right)-1\right)
+o(\delta),
\end{equation}
where
\begin{align}
\hat{h}(s)=\mathbb{E}\left\| I_{Nd} - \eta H - \delta (L\otimes I_d) \right\|^s,
\qquad
\hat{h}_{dis}(s)=\mathbb{E}\left\| I_{Nd} - \eta H \right\|^s.
\end{align}
\end{lemma}

In Lemma~\ref{lem:general:d:b}, we obtained the first-order approximation
of $\hat{h}(s)$ in terms of $\hat{h}_{dis}(s)$, a correction term
linear in $\delta$, and a higher-order error term.
Note that
\begin{equation*}
\Vert I_{Nd}-\eta H\Vert
=\max_{1\leq i\leq N}\Vert I_{d}-\eta H_{i}\Vert
=\max_{1\leq i\leq N}\max_{1\leq j\leq d}\left|1-\eta\lambda_{j}(H_{i})\right|,
\end{equation*}
and since $a_{i,j}$ are i.i.d. continuously distributed, 
such that when $b<d$, with probability $1$, the smallest eigenvalue
of $H_{i}$ is $0$ such that $\Vert I_{Nd}-\eta H\Vert\geq 1$
and when $b\geq d$, with probability $1$, the smallest eigenvalue
of $H_{i}$ is greater than $1$, and hence we focus
on the case $b\geq d$ 
and $\hat{\rho}_{dis}:=\mathbb{E}\log\Vert I_{Nd}-\eta H\Vert<0$
so that there exists
a unique positive value $\hat{\alpha}_{dis}$ such that $\hat{h}_{dis}(\hat{\alpha}_{dis})=1$
and for $\delta$ sufficiently small, there exists
a unique positive value $\hat{\alpha}$ such that $\hat{h}(\hat{\alpha})=1$.
Next, we are going to utilize Lemma~\ref{lem:general:d:b} 
to obtain a first-order approximation of the tail-index $\hat{\alpha}$
in terms of the tail-index $\hat{\alpha}_{dis}$, a correction term linear in $\delta$
and a higher-order error term.

\begin{theorem}\label{thm:general:d:b}
Assume $b\geq d$ and $\hat{\rho}_{dis}=\mathbb{E}\log\Vert I_{Nd}-\eta H\Vert<0$. 
As $\delta\rightarrow 0$, we have the first-order expansion:
\begin{equation}\label{expansion:general:d:b}
\hat{\alpha}
=\hat{\alpha}_{dis}
-\frac{\frac{s\delta}{N}\sum_{i=1}^{N}L_{ii}\left(1-2\mathbb{P}\left(\lambda_{j_*(i_{\ast})}(H_{i_*})<1/\eta\right)\right)}{\mathbb{E}\left[\log\left(\left\Vert I_{Nd}-\eta H\right\Vert\right)\left\Vert I_{Nd}-\eta H\right\Vert^{\hat{\alpha}_{dis}}\right]}
+o(\delta).
\end{equation}
\end{theorem}

We can see from Theorem~\ref{thm:general:d:b}
that the second term on the right hand side of \eqref{expansion:general:d:b}
is positive if and only if $\mathbb{P}\left(\lambda_{j_*(i_{\ast})}(H_{i_*})<1/\eta\right)<\frac{1}{2}$, and this probability is monotonically decreasing
in $N$ and $\eta$. Therefore, we conclude from Theorem~\ref{thm:general:d:b}
that when $\delta$ is small, 
$\hat{\alpha}<\hat{\alpha}_{dis}$ given the stepsize $\eta$ or network size $N$
is large.
On the other hand, when $\delta$ is small, 
$\hat{\alpha}>\hat{\alpha}_{dis}$ given the stepsize $\eta$ or network size $N$
is small.

Next, let us consider the special case $b\in\mathbb{N}$ and $d=1$, 
under which we are able to obtain a more explicit characterization of
the tail-index.

\begin{theorem}\label{thm:first:order:alpha:general:b}
Assume $d=1$.
Also assume that $\hat{\rho}_{dis}=\mathbb{E}\left[\log\left(\max_{1\leq i\leq N}\left|1-\frac{\eta}{b}\sum_{j=1}^{b}a_{i,j}^{2}\right|\right)\right]<0$.
As $\delta\rightarrow 0$, we have 
\begin{align}\label{alpha:expansion:general:b}
\hat{\alpha}&=\hat{\alpha}_{dis}
-\frac{\frac{s\delta\sum_{i=1}^{N}L_{ii}}{N}\left[1-2\mathbb{P}\left(\min_{1\leq i\leq N}\sum_{j=1}^{b}a_{i,j}^{2}+\max_{1\leq i\leq N}\sum_{j=1}^{b}a_{i,j}^{2}<\frac{2b}{\eta}\right)\right]}{\mathbb{E}\left[\log\left(\max_{1\leq i\leq N}\left|1-\frac{\eta}{b}\sum_{j=1}^{b}a_{i,j}^{2}\right|\right)
\max_{1\leq i\leq N}\left|1-\frac{\eta}{b}\sum_{j=1}^{b}a_{i,j}^{2}\right|^{\hat{\alpha}_{dis}}\right]}+o(\delta),
\end{align}
where
\begin{align}
&\mathbb{P}\left(\min_{1\leq i\leq N}\sum_{j=1}^{b}a_{i,j}^{2}+\max_{1\leq i\leq N}\sum_{j=1}^{b}a_{i,j}^{2}<\frac{2b}{\eta}\right)
\nonumber
\\
&=\left(F_{b}\left(\frac{2b}{\eta}\right)\right)^{N}
-\int_{\frac{b}{\eta}}^{\frac{2b}{\eta}}N\left(F_{b}(y)-F_{b}\left(\frac{b}{\eta}-y\right)\right)^{N-1}f_{b}(y)dy,
\label{numerator:general:b}
\end{align}
and
\begin{align}
&\mathbb{E}\left[\log\left(\max_{1\leq i\leq N}\left|1-\frac{\eta}{b}\sum_{j=1}^{b}a_{i,j}^{2}\right|\right)
\max_{1\leq i\leq N}\left|1-\frac{\eta}{b}\sum_{j=1}^{b}a_{i,j}^{2}\right|^{\hat{\alpha}_{dis}}\right]
\nonumber
\\
&=\int_{0}^{1}\log(x)x^{\hat{\alpha}_{dis}}\frac{bN}{\eta}\left(F_{b}\left(\frac{b(1+x)}{\eta}\right)-F_{b}\left(\frac{b(1-x)}{\eta}\right)\right)^{N-1}
\nonumber
\\
&\qquad\qquad\qquad\qquad\qquad\qquad\cdot\left(f_{b}\left(\frac{b(1+x)}{\eta}\right)+f_{b}\left(\frac{b(1-x)}{\eta}\right)\right)dx
\nonumber
\\
&\qquad\qquad\qquad
+\int_{1}^{\infty}\log(x)x^{\hat{\alpha}_{dis}}\frac{bN}{\eta}\left(F_{b}\left(\frac{b(1+x)}{\eta}\right)\right)^{N-1}
f_{b}\left(\frac{b(1+x)}{\eta}\right)dx,
\label{denominator:general:b}
\end{align}
where $f_{b}$ and $F_{b}$ are the probability density and cumulative distribution functions of $\sum_{j=1}^{b}a_{i,j}^{2}$.
\end{theorem}

In particular, when $a_{i,j}$ are i.i.d. $\mathcal{N}(0,\sigma^{2})$ distributed,
we have the following corollary.

\begin{corollary}\label{cor:first:order:alpha:general:b}
Under the assumptions in Theorem~\ref{thm:first:order:alpha:general:b}
and further assume that $a_{i,j}$ are i.i.d. $\mathcal{N}(0,\sigma^{2})$ distributed.
As $\delta\rightarrow 0$, we have the first-order expansion \eqref{alpha:expansion:general:b} with
\begin{equation}
f_{b}(x):=\frac{1}{2^{b/2}\Gamma(b/2)}\frac{x^{\frac{b}{2}-1}}{\sigma^{b}}e^{-\frac{x}{2\sigma^{2}}},
\qquad
F_{b}(x):=\frac{1}{\Gamma(b/2)}\gamma\left(\frac{b}{2},\frac{x}{2\sigma^{2}}\right),
\end{equation}
in the expressions~\eqref{numerator:general:b} and \eqref{denominator:general:b}, where $\Gamma(\cdot)$ is the gamma function
and $\gamma(\cdot,\cdot)$ is the lower incomplete gamma function.
\end{corollary}

\subsubsection{General dimension $d$, batch-sizes $b_{i}$ and variances $\sigma_{i}^{2}$}

More generally, we consider the general dimension $d$, batch-sizes $b_{i}$
and variances $\sigma_{i}^{2}$.
It follows from \eqref{without:expectation-2} that
\begin{align}
\left\| I_{Nd} - \eta H - \delta (L\otimes I_d) \right\|^s 
=\left\| I_{Nd} - \eta H \right\|^s - s\delta\cdot \mbox{sign} \left(1-\eta\lambda_{j_*(i_{\ast})}(H_{i_*})\right)  L_{i_* i_*} +  o(\delta)\,,
\end{align}
where $i_*$ is the index that maximizes $\sigma_{\max,i}$ in \eqref{eq-sigma-max} and
$$j_*(i) := \arg \max_{1\leq j\leq d} \left| 1- \eta \lambda_{j}(H_{i})\right|$$
is the index of the eigenvalues of $H_{i_*}$ so that we have 
\beq 
\|I_d - \eta H_{i_*}\| = \left| 1- \eta \lambda_{j_*(i_{\ast})}(H_{i_*}) \right|.
\eeq

\begin{lemma}\label{lem:general:d:b:i}
As $\delta\rightarrow 0$, we have the first-order expansion:
\begin{equation}
\hat{h}(s)
=\hat{h}_{dis}(s)
-s\delta\sum_{i=1}^{N}\mathbb{P}(i_{\ast}=i)L_{ii}\left(2\mathbb{P}\left(\lambda_{j_*(i_{\ast})}(H_{i_*})<1/\eta|i_{\ast}=i\right)-1\right)
+o(\delta),
\end{equation}
where
\begin{align}
\hat{h}(s)=\mathbb{E}\left\| I_{Nd} - \eta H - \delta (L\otimes I_d) \right\|^s,
\qquad
\hat{h}_{dis}(s)=\mathbb{E}\left\| I_{Nd} - \eta H \right\|^s.
\end{align}
\end{lemma}

In Lemma~\ref{lem:general:d:b:i}, we obtained the first-order approximation
of $\hat{h}(s)$ in terms of $\hat{h}_{dis}(s)$, a correction term
linear in $\delta$, and a higher-order error term.
Note that
\begin{equation*}
\Vert I_{Nd}-\eta H\Vert
=\max_{1\leq i\leq N}\Vert I_{d}-\eta H_{i}\Vert
=\max_{1\leq i\leq N}\max_{1\leq j\leq d}\left|1-\eta\lambda_{j}(H_{i})\right|,
\end{equation*}
and since $a_{i,j}$ are continuously distributed, 
such that when $b_{i}<d$, with probability $1$, the smallest eigenvalue
of $H_{i}$ is $0$ such that $\Vert I_{Nd}-\eta H\Vert\geq 1$
and when $b_{i}\geq d$, with probability $1$, the smallest eigenvalue
of $H_{i}$ is greater than $1$, and hence we focus
on the case $b_{i}\geq d$ for every $i$ 
and $\hat{\rho}_{dis}:=\mathbb{E}\log\Vert I_{Nd}-\eta H\Vert<0$
so that there exists
a unique positive value $\hat{\alpha}_{dis}$ such that $\hat{h}_{dis}(\hat{\alpha}_{dis})=1$
and for $\delta$ sufficiently small, there exists
a unique positive value $\hat{\alpha}$ such that $\hat{h}(\hat{\alpha})=1$.
Next, we are going to utilize Lemma~\ref{lem:general:d:b:i} 
to obtain a first-order approximation of the tail-index $\hat{\alpha}$
in terms of the tail-index $\hat{\alpha}_{dis}$, a correction term linear in $\delta$
and a higher-order error term.

\begin{theorem}\label{thm:general:d:b:i}
Assume $b_{i}\geq d$ for every $i=1,2,\ldots,N$ and $\hat{\rho}_{dis}=\mathbb{E}\log\Vert I_{Nd}-\eta H\Vert<0$. 
As $\delta\rightarrow 0$, we have the first-order expansion:
\begin{equation}\label{expansion:general:d:b:i}
\hat{\alpha}
=\hat{\alpha}_{dis}
-\frac{s\delta\sum_{i=1}^{N}\mathbb{P}(i_{\ast}=i)L_{ii}\left(2\mathbb{P}\left(\lambda_{j_*(i_{\ast})}(H_{i_*})<1/\eta|i_{\ast}=i\right)-1\right)}{\mathbb{E}\left[\log\left(\left\Vert I_{Nd}-\eta H\right\Vert\right)\left\Vert I_{Nd}-\eta H\right\Vert^{\hat{\alpha}_{dis}}\right]}
+o(\delta).
\end{equation}
\end{theorem}

Next, let us consider the special case $d=1$, under which we are able to obtain
a more explicit characterization of the tail-index

\begin{theorem}\label{thm:first:order:alpha:general:b:i}
Assume $d=1$.
Also assume that 
$\hat{\rho}_{dis}=\mathbb{E}\log(\max_{1\leq i\leq N}|1-\eta X_{i}|)<0$, 
where $X_{i}:=\frac{1}{b_{i}}\sum_{j=1}^{b_{i}}a_{i,j}^{2}$.
As $\delta\rightarrow 0$, we have 
\begin{align}\label{alpha:expansion:general:b:i}
\hat{\alpha}&=\hat{\alpha}_{dis}
-\frac{s\delta\sum_{i=1}^{N}L_{ii}\mathbb{P}\left(X_{i}<\min_{k\neq i}X_{k},X_{i}+\max_{k\neq i}X_{k}<\frac{2}{\eta}\right)}{\mathbb{E}\left[\log\left(\max_{1\leq i\leq N}\left|1-\eta X_{i}\right|\right)
\max_{1\leq i\leq N}\left|1-\eta X_{i}\right|^{\hat{\alpha}_{dis}}\right]}
\nonumber
\\
&\qquad\qquad
+\frac{s\delta\sum_{i=1}^{N}L_{ii}\mathbb{P}\left(X_{i}>\max_{k\neq i}X_{k},X_{i}+\min_{k\neq i}X_{k}>\frac{2}{\eta}\right)}{\mathbb{E}\left[\log\left(\max_{1\leq i\leq N}\left|1-\eta X_{i}\right|\right)
\max_{1\leq i\leq N}\left|1-\eta X_{i}\right|^{\hat{\alpha}_{dis}}\right]}+o(\delta),
\end{align}
where
\begin{align*}
&\mathbb{P}\left(X_{i}<\min_{k\neq i}X_{k},X_{i}+\max_{k\neq i}X_{k}<\frac{2}{\eta}\right)
\\
&\qquad=\int_{0}^{\frac{2}{\eta}}\int_{x}^{\frac{2}{\eta}}F_{i}\left(\min\left(x,\frac{2}{\eta}-y\right)\right)\sum_{k\neq i}\sum_{j\neq k,i}f_{k}(y)f_{j}(x)\prod_{\ell\neq j,k,i}(F_{\ell}(y)-F_{\ell}(x))dydx,
\\
&\mathbb{P}\left(X_{i}>\max_{k\neq i}X_{k},X_{i}+\min_{k\neq i}X_{k}>\frac{2}{\eta}\right)
\\
&\qquad
=\sum_{i=1}^{N}L_{ii}\int_{0}^{\infty}\int_{x}^{\infty}\left(1-F_{i}\left(\max\left(y,\frac{2}{\eta}-x\right)\right)\right)
\\
&\qquad\qquad\qquad\qquad\qquad\cdot\sum_{k\neq i}\sum_{j\neq k,i}f_{k}(y)f_{j}(x)\prod_{\ell\neq j,k,i}(F_{\ell}(y)-F_{\ell}(x))dydx,
\end{align*}
and
\begin{align}
&\mathbb{E}\left[\log\left(\max_{1\leq i\leq N}\left|1-\eta X_{i}\right|\right)
\max_{1\leq i\leq N}\left|1-\eta X_{i}\right|^{\hat{\alpha}_{dis}}\right]
\nonumber
\\
&=\int_{0}^{1}\log(x)x^{\hat{\alpha}_{dis}}\sum_{i=1}^{N}
\left(f_{i}\left(\frac{1+x}{\eta}\right)+f_{i}\left(\frac{1-x}{\eta}\right)\right)
\prod_{k\neq i}\left(F_{k}\left(\frac{1+x}{\eta}\right)-F_{k}\left(\frac{1-x}{\eta}\right)\right)dx
\nonumber
\\
&\qquad
+\int_{1}^{\infty}\log(x)x^{\hat{\alpha}_{dis}}\sum_{i=1}^{N}f_{i}\left(\frac{1+x}{\eta}\right)\prod_{k\neq i}^{N}F_{k}\left(\frac{1+x}{\eta}\right)dx,
\end{align}
where $f_{i},F_{i}$ are the probability density and probability distribution functions of $X_{i}$.
\end{theorem}

In particular, when for every $i$, $a_{i,j}$ are i.i.d. $\mathcal{N}(0,\sigma_{i}^{2})$ distributed,
we have the following corollary.

\begin{corollary}\label{cor:first:order:alpha:general:b:i}
Under the assumptions in Theorem~\ref{thm:first:order:alpha:general:b:i}
and further assume that 
for every $i$, $a_{i,j}$ are i.i.d. $\mathcal{N}(0,\sigma_{i}^{2})$ distributed.
As $\delta\rightarrow 0$, we have the first-order expansion \eqref{alpha:expansion:general:b:i} with
\begin{equation}
f_{i}(x):=\frac{1}{2^{b_{i}/2}\Gamma(b_{i}/2)}\frac{x^{\frac{b_{i}}{2}-1}}{(\sigma_{i}^{2}/b_{i})^{b_{i}/2}}e^{-\frac{x}{2(\sigma_{i}^{2}/b_{i})}},
\qquad
F_{i}(x):=\frac{1}{\Gamma(b_{i}/2)}\gamma\left(\frac{b_{i}}{2},\frac{b_{i}x}{2\sigma_{i}^{2}}\right),
\end{equation}
where $\Gamma(\cdot)$ is the gamma function
and $\gamma(\cdot,\cdot)$ is the lower incomplete gamma function.
\end{corollary}

\subsection{Tail-Index Comparison between Decentralized SGD and Centralized SGD: The General Case}

In this section, we are interested in comparing the tail-index
for the DE-SGD with that of the C-SGD for the general dimension case
where we assume $b_{i}\equiv b$ and $\sigma_{i}\equiv\sigma$. 

In Proposition~\ref{prop:disconnected:centralized}, we showed that
the tail-index for Dis-SGD is smaller than that
of the C-SGD. 
In Theorem~\ref{thm:general:d:b}
we showed that when $\delta$ is small, 
the tail-index for the DE-SGD is small
than that of the Dis-SGD given the stepsize $\eta$ or network size $N$
is large.
Therefore, we have the following corollary 
that compares the tail-index
for the DE-SGD with that of the C-SGD.

\begin{corollary}\label{cor:DSGD:CSGD:general}
Under the assumptions in Theorem~\ref{thm:general:d:b},
the tail-index of the decentralized SGD is smaller
than that of the centralized SGD provided that the stepsize $\eta$ or network size $N$
is large and $\delta$ is small.
\end{corollary}

On the other hand, we showed in Theorem~\ref{thm:general:d:b} when $\delta$ is small, 
the tail-index for the DE-SGD is larger
than that of the Dis-SGD given the stepsize $\eta$ or network size $N$ is small.
Therefore, our theory (Proposition~\ref{prop:disconnected:centralized} and Theorem~\ref{thm:general:d:b})
does not provide a guidance to the comparison of the tail-indexes 
between DE-SGD and C-SGD in this regime.
However, in the numerical experiments section, we have observed that DE-SGD lead to heavier tails than C-SGD.

\section{Technical Proofs}

\subsection{Proof of Results in Section~\ref{sec-least-squares}}

\subsubsection{Proof of Proposition~\ref{prop:ht_noncvx}} 
We will first show that under the assumptions of Proposition~\ref{prop:ht_noncvx}, the function $F_{\mathcal{W}}$ is strongly convex outside a compact region, and the result will follow from \citet{hodgkinson2020multiplicative}.

By assumption, the function $x_i \mapsto f_i(x_i)$ is $\mu_i$-strongly convex outside a compact region for some $\mu_i>0$, therefore there exists $R_i>0$ such that the Hessian matrix $\nabla^2_{x_i x_i} f_i(x_i) \succeq \mu_i I$ whenever $\|x_i\|\geq R_i$. For any given $c \in (0, \frac{1}{N})$, consider the set
\begin{align*}
\mathcal{S}_c &= \Bigg\{ 
x:=\left[\left(x_{1}\right)^{T},\left(x_{2}\right)^{T},\ldots,\left(x_{N}\right)^{T}\right]^{T}\in\mathbb{R}^{Nd},
\\
&\qquad\qquad\qquad\qquad
\min_{1\leq i\leq N} \|x_i\|^2  > c  \left(\|x_1\|^2 + \|x_2\|^2  + \dots + \|x_N\|^2\right)
\Bigg\}\,.
\end{align*}

Notice that 
\begin{align*}
\nabla^2_{xx} F_\mathcal{W}(x) 
&= \nabla^2_{xx} F(x) + \frac{1}{\eta} (I_{Nd} - \mathcal{W})  
\\
&=\mbox{blkdiag}\big(\nabla^2_{x_1 x_1} f_1(x_1),\nabla^2_{x_2 x_2} f_2(x_2), \dots, \nabla^2_{x_N x_N} f_N(x_N)\big) + \frac{1}{\eta} (I_{Nd} - \mathcal{W}). 
\end{align*}

For $x\in \mathcal{S}_c$ and $\|x\|\geq R/\sqrt{c}$ with $R:=\max_{1\leq i\leq N} R_i$, we have $\|x_i\|> R$ for every $i=1,2,\dots,N$ by the definition of $\mathcal{S}_c$. Therefore, 
$$x^T \nabla_{xx}^2 F_{\mathcal{W}}(x) x \geq \sum_{i=1}^N x_i^2\left( \nabla^2_{x_i x_i} F(x_i)\right) x_i^2
\geq \mu \sum_{i=1}^N \|x_i\|^2 = \mu \|x\|^2\,, $$
with $\mu=\min_i \mu_i$ where we used the facts that $f_i$ is $\mu_i$- strongly convex outside the compact ball of radius $R$ and we have $x^T (I_{Nd}-\mathcal{W})x \geq 0$ for any $x$. 

It remains to show that for $x\not\in\mathcal{S}_c$ and $\|x\|\geq R/\sqrt{c}$, strong convexity also holds; i.e. 
\begin{eqnarray} \nabla^2_{xx} F_\mathcal{W}(x)  \geq c_0 \|x\|^2\,,
\label{ineq-to-prove-2}
\end{eqnarray}
for some $c_0 > 0$. Towards this direction, we start with a technical lemma.

\begin{lemma} \label{eq-set-Sc-bound}
For $x\not\in \mathcal{S}_c$, we have 
\begin{equation} x^T( I_{Nd}-\mathcal{W}) x \geq c_2 (1 - \lambda_2(W))\|x\|^2\,,
\label{ineq-to-prove}
\end{equation}
for some $c_2 > 0$. 
\end{lemma}

\begin{proof}
Note that the eigenvalues of $I_{Nd}-\mathcal{W}$ are $1 - \lambda_j(W)$ each with multiplicity $d$. 
Let 
\begin{equation*}
\beta_1\leq \beta_2 \leq \cdots \leq\beta_{Nd}
\end{equation*}
be the eigenvalues of $I_{Nd}-\mathcal{W}$ in increasing order and let $y_i$ be an eigenvector of unit norm corresponding to the eigenvalue $\beta_i$. Note that we have 
\begin{equation*}
\beta_1 = \beta_2 = \cdots = \beta_d = 1- \lambda_1(W) = 0,
\end{equation*}
and 
\begin{equation*}
\beta_{d+1} =\beta_{d+2} = \cdots = \beta_{2d} = 1-\lambda_2(W). 
\end{equation*}
Note that we can take $y_i = \textbf{1} \otimes e_i / \sqrt{N}$ for $i=1,2,\dots,d$ where $\textbf{1}$ is the vector of ones of length $N$ and $e_i$ is the $i$-th basis vector in $\mathbb{R}^d$. 

To show \eqref{ineq-to-prove}, without loss of generality, assume $\|x\|=1$.\footnote{If this is not the case, we could apply the same proof technique to the normalized vector $x/\|x\|$ which will have unit norm).} We introduce the $d$-dimensional subspace:
 \begin{equation}
 \mathcal{H} = \left\{ 
 x:=\left[\left(x_{1}\right)^{T},\left(x_{2}\right)^{T},\ldots,\left(x_{N}\right)^{T}\right]^{T}\in\mathbb{R}^{Nd},~ x_1 = x_2 = \dots = x_N
 \right\}\,.
 \end{equation}
We observe that $\mathcal{H} = \mbox{span}(y_1, y_2,\dots, y_d)$ and notice that $\mathcal{H}$ is the null space of $\mathcal{W}$. We can also write $x$ in the basis of eigenvectors; i.e. we can write $x = \sum_{i=1}^{Nd} a_i y_i$ where $a_i = \langle x, y_i \rangle $. Since the eigenvectors $y_i$ and $y_j$ are orthogonal for $i\neq j$, we have  $\|x\|^2 = \sum_{i=1}^{Nd} a_i^2 = 1$.
We  can also write $$x = z_1 + z_2\,,$$ 
where $z_1 := \sum_{i=1}^d a_i y_i \in \mathcal{H}$ and $z_2 = \sum_{i=d+1}^{Nd} a_i y_i$ lies in the subspace orthogonal to $\mathcal{H}$. Since $( I_{Nd}-\mathcal{W})y_i = \beta_i y_i$, we have also
\begin{align} 
 x^T( I_{Nd}-\mathcal{W}) x = \sum_{i=1}^{Nd} \beta_i a_i^2 \|y_i\|^2
   &= \sum_{i=d+1}^{Nd} \beta_i a_i^2 \nonumber \\
 &\geq  (1-\lambda_2(W)) \sum_{i=d+1}^{Nd}  a_i^2 \nonumber\\
 &=  (1-\lambda_2(W)) \|z_2\|^2\,, 
 \label{ineq-intermediate}
\end{align}
where we used $\|y_i\|=1$ and $\beta_i \geq  (1-\lambda_2(W))$ for $i\geq d+1$.  Recalling that $\|x\|=1$, to show \eqref{ineq-to-prove}, it suffices to show that $\|z_2\|^2 \geq c_2$ for some $c_2>0$.

Since $x\not\in \mathcal{S}_c$ by the assumption, we have $x_j$ such that 
\begin{equation} 
\|x_j\|^2 \leq c\|x\|^2  < \frac{1}{N}.
\label{ineq-xj}
\end{equation} 
If we consider the optimization problem
\begin{align}
    c_3(x) &:= \max_{u\in\mathbb{R}^d, \|u\|=\frac{1}{\sqrt{N}}} \left\langle 
    \left[u^{T},u^{T},\cdots,u^{T}\right]^{T},
    \left[(x_{1})^{T},(x_{2})^{T},\cdots,(x_{N})^{T}\right]^{T}
\right\rangle^2
\nonumber
\\
&=  \max_{u\in\mathbb{R}^d, \|u\|=\frac{1}{\sqrt{N}}} \left|u^T (x_1 + x_2 + \dots + x_N)\right|^2.
\label{eq-c3}
\end{align} 
It is easy to see from the right-hand side of \eqref{eq-c3} that the maximizer is 
\begin{equation*}
u= \frac{1}{\sqrt{N}} \frac{x_1 + x_2 + \dots + x_N}{\|x_1 + x_2 + \dots + x_N\|}\,,
\end{equation*}
which yields the optimum value 
\begin{equation*}
c_3(x) = \frac{\|x_1 + x_2 + \dots + x_N \|^2}{N}. 
\end{equation*}
Note that we have $c_3(x) < 1$ because by Cauchy-Schwarz inequality
\begin{equation} 
c_3(x) = \frac{\|x_1 + x_2 + \dots + x_N \|^2}{N} \leq \sum_{i=1}^{N} \|x_i\|^2 = 1\,, 
\end{equation}
and the equality is attained only when $x_1 = x_2 = \cdots =x_N$ (and this cannot be the case due to \eqref{ineq-xj}). In particular, we have
\begin{equation} c_4 : = \sup_{x\not\in \mathcal{S}_c, \|x\|=1} c_3(x) < 1,
\label{def-max-coef}
\end{equation}
because the supremum is taken over a compact set and attained at a point $x_*$ with $c_3(x_*)<1$. 
Note that $z_1 = \textbf{1} \otimes \sum_{i=1}^d a_i e_i/\sqrt{N}$ and is of the form $z_1 = \begin{bmatrix} u_1^T & u_1^T & \cdots & u_1^T
\end{bmatrix}^T$ with $u_1 = \sum_{i=1}^d a_i e_i/\sqrt{N}$ where $\|u_1\| = \sqrt{\sum_{i=1}^d a_i^2}/{\sqrt{N}}$.
Therefore, we obtain from \eqref{eq-c3} and \eqref{def-max-coef} that
$$ \langle z_1/\|z_1\|, x \rangle \leq c_3(x) \leq c_4 < 1.$$
Writing $x = z_1 + z_2$ into this equation and using $\langle z_1, z_2 \rangle = 0$ leads to 
\begin{equation*} 
\langle z_1/\|z_1\|_2, z_1 \rangle = \|z_1\|  \leq c_4 < 1.
\end{equation*}

Since $\|x\|^2 = 1 = \|z_1\|^2 + \|z_2\|^2$, we conclude that 
 $$\|z_2\|^2 \geq 1 - c_4 > 0.$$
Together with \eqref{ineq-intermediate}, this implies that for $\|x\|=1$, \eqref{ineq-to-prove} holds with $c_2 = 1-c_4>0$. This completes the proof.
\end{proof} 
Equipped with this lemma, the stage is set for completing the proof of Proposition~\ref{prop:ht_noncvx}. Notice that by the continuity of the second derivatives, there exists a positive constant $m>0$ such that $\nabla^2_{x_i x_i} f_i(x_i) \succeq -m I_d$ for every $i=1,2,\dots,N$ whenever $\|x_i\|\leq R$. 
For $x\not\in \mathcal{S}_c$, we have
\begin{align*}
x^T \nabla^2_{xx} F_\mathcal{W}(x) x 
&=\left(\sum_{i=1}^{N} x_i^T \nabla^2_{x_i x_i} f_i(x_i) x_i \right) + \frac{1}{\eta} x^T(I_{Nd}-\mathcal{W})x 
\\
&\geq \left(-m + c_2\frac{1-\lambda_2(W)}{\eta}\right) \|x\|^2 > 0\,,
\end{align*}
for sufficiently small $\eta>0$ where we used Lemma \ref{eq-set-Sc-bound}. This proves \eqref{ineq-to-prove-2}. We conclude that for $\|x\|\geq R/\sqrt{c}$, we have $\nabla^2_{x x} F_\mathcal{W}(x) \succ 0$ for sufficiently small $\eta$. 

Based on this property, the results follow from \cite[Theorem 1 and Example 4]{hodgkinson2020multiplicative}. This completes the proof.
\hfill $\Box$


\subsection{Proof of Results in Section~\ref{sec:theoretical}}


\subsubsection{Proof of Theorem~\ref{thm:mono}}
The follows by adapting the proof of Theorem~4 in \citet{ht_sgd_quad}
by working with the $x\mapsto\Vert x\Vert$ norm instead of the $x\mapsto\Vert x\cdot e_{1}\Vert$ norm,
where $e_1$ is the first basis vector.
\hfill $\Box$


\subsubsection{Proof of Proposition~\ref{prop:disconnected:centralized}}
By adapting the proof of Theorem~4 in \citet{ht_sgd_quad} to work 
with the $x\mapsto\Vert x\Vert$ norm instead of the $x\mapsto\Vert x\cdot e_{1}\Vert$ norm,
where $e_1$ is the first basis vector, we can show
that $\hat{\alpha}(b)$ as a function of the batch-size $b$ is increasing
in $b$. Thus $\hat{\alpha}(bN)>\hat{\alpha}(b)$. 
Moreover, $\hat{\alpha}(bN)$ is increasing in $N$. 
Hence, the conclusion follows.
\hfill $\Box$

\subsubsection{Proof of Theorem~\ref{thm:first:order:alpha}}

Before we proceed to the proof of Theorem~\ref{thm:first:order:alpha},
let us first provide a complete statement with all
the technical details.

\begin{theorem}[Complete Statement of Theorem~\ref{thm:first:order:alpha}]\label{thm:first:order:alpha-extended}
Assume $d=1$ and $b=1$.
Also assume that $\hat{\rho}_{dis}=\mathbb{E}\left[\log\left(\max_{1\leq i\leq N}|1-\eta a_{i}^{2}|\right)\right]<0$. 
As $\delta\rightarrow 0$, we have 
\begin{align}\label{alpha:expansion:extended}
\hat{\alpha}&=\hat{\alpha}_{dis}
-\frac{s\delta N^{-1}\sum_{i=1}^{N}L_{ii}\left[1-2\mathbb{P}\left(\min_{1\leq i\leq N}a_{i}^{2}+\max_{1\leq i\leq N}a_{i}^{2}<2/\eta\right)\right]}{\mathbb{E}\left[\log\left(\max_{1\leq i\leq N}|1-\eta a_{i}^{2}|\right)
\max_{1\leq i\leq N}|1-\eta a_{i}^{2}|^{\hat{\alpha}_{dis}}\right]}+o(\delta),
\end{align}
where we provide an explicit formula for the probability and the expectation on the right hand side of \eqref{alpha:expansion:extended} as follows:
\begin{align}
&\mathbb{P}\left(\min\nolimits_{1\leq i\leq N}a_{i}^{2}+\max\nolimits_{1\leq i\leq N}a_{i}^{2}<2/\eta\right)
\nonumber
\\
&=\left(F_{a}\left(2\eta^{-1}\right)\right)^{N}
-\int_{\eta^{-1}}^{2\eta^{-1}}N\left(F_{a}(y)-F_{a}\left(2\eta^{-1}-y\right)\right)^{N-1}f_{a}(y)dy,\label{given:1}
\\
&\mathbb{E}\left[\log\left(\max\nolimits_{1\leq i\leq N}\left|1-\eta a_{i}^{2}\right|\right)
\max\nolimits_{1\leq i\leq N}\left|1-\eta a_{i}^{2}\right|^{\hat{\alpha}_{dis}}\right]
\nonumber
\\
&=\int_{0}^{1}\frac{\log(x)x^{\hat{\alpha}_{dis}}N\eta^{-1}}{\left(F_{a}\left((1+x)\eta^{-1}\right)-F_{a}\left((1-x)\eta^{-1}\right)\right)^{1-N}}
\left(f_{a}\left((1+x)\eta^{-1}\right)+f_{a}\left((1-x)\eta^{-1}\right)\right)dx
\nonumber
\\
&\qquad
+\int_{1}^{\infty}\log(x)x^{\hat{\alpha}_{dis}}N\eta^{-1}\left(F_{a}\left((1+x)\eta^{-1}\right)\right)^{N-1}
f_{a}\left((1+x)\eta^{-1}\right)dx,\label{given:2}
\end{align}
where $f_{a}$ and $F_{a}$ are the probability density and cumulative distribution functions of $a_{i}^{2}$.
\end{theorem}

Before we proceed to the proof of Theorem~\ref{thm:first:order:alpha},
let us first prove the following lemma, which provides a first-order
expansion of $\hat{h}(s)$. The proof of Theorem~\ref{thm:first:order:alpha}
will be then based on utilizing this lemma.

\begin{lemma}\label{lem:first:order}
Assume $d=1$ and $b=1$.
Also assume that $a_{i}$ are i.i.d. with a continuous distribution.
As $\delta\rightarrow 0$, we have 
\begin{align}\label{hat:h:expansion}
\hat{h}(s)=\hat{h}_{dis}(s)
-\frac{s\delta\sum_{i=1}^{N}L_{ii}}{N}\left[2\mathbb{P}\left(\min_{1\leq i\leq N}a_{i}^{2}+\max_{1\leq i\leq N}a_{i}^{2}<\frac{2}{\eta}\right)-1\right]+o(\delta),
\end{align}
where
\begin{align}
\hat{h}_{dis}(s)
&=\int_{0}^{1}\left(1-\left(F_{a}\left(\frac{1+t^{\frac{1}{s}}}{\eta}\right)-F_{a}\left(\frac{1-t^{\frac{1}{s}}}{\eta}\right)\right)^{N}\right)dt
\nonumber
\\
&\qquad\qquad\qquad\qquad\qquad\qquad
+\int_{1}^{\infty}\left(1-\left(F_{a}\left(\frac{1+t^{\frac{1}{s}}}{\eta}\right)\right)^{N}\right)dt,
\end{align}
and
\begin{align}
&\mathbb{P}\left(\min_{1\leq i\leq N}a_{i}^{2}+\max_{1\leq i\leq N}a_{i}^{2}<\frac{2}{\eta}\right)
\nonumber
\\
&=\left(F_{a}\left(\frac{2}{\eta}\right)\right)^{N}
-\int_{\frac{1}{\eta}}^{\frac{2}{\eta}}N\left(F_{a}(y)-F_{a}\left(\frac{2}{\eta}-y\right)\right)^{N-1}f_{a}(y)dy,
\end{align}
where $f_{a}$ and $F_{a}$ are the probability density and cumulative distribution functions of $a_{i}^{2}$.
\end{lemma}

Before we proceed to the proof of Lemma~\ref{lem:first:order},
let us first recall that
\begin{equation}
\hat{h}(s)=\mathbb{E}\left[\left(g(\delta)\right)^{s}\right],
\end{equation}
where:
\begin{equation}
g(\delta)=\left\Vert\mathcal{W}-\eta H\right\Vert=\left\Vert I_{Nd}-\eta H-\delta(L\otimes I_{d})\right\Vert.
\end{equation}
In the following, we start with considering the Taylor series expansion of $g(\delta)$ around $\delta=0$: By the standard perturbation theory for singular values, we have
 \begin{equation} g(\delta) = g(0) - \delta  u^T (L\otimes I_d) v + o(\delta), 
 \label{eq-sing-pert}
 \end{equation}
(see e.g. \cite[Lemma 2.3]{guglielmi2011fast}) where $u$ and $v$ are the left and right eigenvectors of $I -\eta H$ satisfying
\begin{equation} (I_{Nd} - \eta H)v = \sigma_{\max} u, \quad u^T (I_{Nd} - \eta H)= \sigma_{\max}v^T,
\label{eq-sing-vecs}
\end{equation}
with $\|u\|=\|v\|$ where $\sigma_{\max} =\| I_{Nd} - \eta H \|$ is the largest singular value of  $I_{Nd} - \eta H$. Since $H$ 
is symmetric we can choose $u=v$. Because of the block diagonal structure of $H$, we have also
$$
\sigma_{\max} = \max_{1\leq i \leq N} \| I_{d} - \eta H_i \|.
$$ 
We recall that in Theorem~\ref{thm:first:order:alpha} and Lemma~\ref{lem:first:order}, 
we take $d=1$ and $b_{i}\equiv 1$, 
so that we have $H_i = a_i^2$ and $I_{d}-\eta H_i =  1- \eta a_i^2$ is a scalar. Thus, 
$$
\sigma_{\max} = \max_{1\leq i \leq N} \left|1 - \eta a_i^2 \right|.
$$ 
Let $i_*$ be the index for which the maximum on the right-hand side is attained. Then, we can choose the left and right eigenvectors $u$ and $v$ as 
$$
v = e_{i_*}, \quad u =  e_{i_*} \cdot\mbox{sign}(1 - \eta a_{i_*}^2),  
$$
where $\mbox{sign}(\cdot)$ denotes the sign function
and $e_r$ is the $r$-th basis vector. Indeed, it is easy to check that this choice of $u$ and $v$ satisfies \eqref{eq-sing-vecs} for $d=1$. Then, by plugging this choice of $u$ and $v$ into \eqref{eq-sing-pert}, we obtain
$$ \left\| I_{N} - \eta H - \delta L \right\| 
= \left\| I_{N}  - \eta H \right\| - \delta L_{i_* i_*} \cdot\mbox{sign}\left(1-\eta a_{i_*}^2\right) + o(\delta). $$
We note $L_{i_*i_*}>0$. This implies 
\begin{equation}\label{without:expectation} 
\left\| I_{N} - \eta H - \delta L \right\|^s 
= \left\| I_{N} - \eta H \right\|^s - s\delta L_{i_* i_*} \cdot\mbox{sign}\left(1-\eta a_{i_*}^2\right) + o(\delta). 
\end{equation}
This shows that for sufficiently small $\eta$, we have $\mbox{sign}(1-\eta a_{i_*}^2)=+1$ with high probability in which case we obtain
$$ 
\left\| I_{N} - \eta H - \delta L \right\|^s < \left\| I_{N} - \eta H \right\|^s. 
$$
On the other hand, if $\eta$ is sufficiently large (or 
$\sigma$ is sufficiently large) then $\mbox{sign}(1-\eta a_{i_*}^2)= -1$ with high probability in this case we obtain
$$ 
\left\| I_{N} - \eta H - \delta L \right\|^s > \left\| I_{N} - \eta H \right\|^s.
$$
We conclude that if the stepsize is very small, adding the network will increase the tail-index. 
whereas if the stepsize is large enough, adding the network will decrease the tail-index.
We will next provide a rigorous proof of Lemma~\ref{lem:first:order}
to make this statement precise.

\textbf{Proof of Lemma~\ref{lem:first:order}.}

By taking the expectations in \eqref{without:expectation}, we obtain
\begin{equation}
\mathbb{E}\| I_{Nd} - \eta H - \delta (L\otimes I_d) \|^s 
= \mathbb{E}\| I_{Nd} - \eta H \|^s - s\delta \mathbb{E}\left[L_{i_* i_*} \cdot\mbox{sign}(1-\eta a_{i_*}^2)\right] + o(\delta). 
\end{equation}
Notice that $a_{i}$ are i.i.d. distributed with a continuous distribution. 
Therefore, $i_{\ast}$ is uniformly distributed on $\{1,2,\ldots,N\}$ and we can compute that
\begin{align*}
\mathbb{E}\left[L_{i_* i_*} \cdot\mbox{sign}(1-\eta a_{i_*}^2)\right]
&=\sum_{i=1}^{N}\mathbb{P}(i_{\ast}=i)\mathbb{E}\left[L_{i_* i_*} \cdot\mbox{sign}(1-\eta a_{i_*}^2)|i_{\ast}=i\right]
\\
&=\sum_{i=1}^{N}\mathbb{P}(i_{\ast}=i)L_{ii}\mathbb{E}\left[ \mbox{sign}(1-\eta a_{i_*}^2)|i_{\ast}=i\right]
\\
&=\left(\frac{1}{N}\sum_{i=1}^{N}L_{ii}\right)\mathbb{E}\left[ \mbox{sign}(1-\eta a_{i_*}^2)|i_{\ast}=1\right].
\end{align*}
Moreover, we can compute that
\begin{align*}
\mathbb{E}\left[ \mbox{sign}(1-\eta a_{i_*}^2)\right]
=\frac{1}{N}\sum_{i=1}^{N}\mathbb{E}\left[ \mbox{sign}(1-\eta a_{i_*}^2)|i_{\ast}=i\right]
=\mathbb{E}\left[ \mbox{sign}(1-\eta a_{i_*}^2)|i_{\ast}=1\right],
\end{align*}
and we can further compute that
\begin{align*}
\mathbb{E}\left[ \mbox{sign}(1-\eta a_{i_*}^2)|i_{\ast}=1\right]
&=\mathbb{E}\left[ \mbox{sign}(1-\eta a_{i_*}^2)\right]
\\
&=\mathbb{P}\left(a_{i\ast}^{2}<1/\eta\right)-\mathbb{P}\left(a_{i\ast}^{2}>1/\eta\right)
\\
&=2\mathbb{P}\left(a_{i\ast}^{2}<1/\eta\right)-1,
\end{align*}
where we recall that $i_{\ast}=\arg\max_{1\leq i\leq N}|1-\eta a_{i}^{2}|$.
It is easy to see that
\begin{equation}
i_{\ast}\in\left\{\arg\min_{1\leq i\leq N}a_{i}^{2},\arg\max_{1\leq i\leq N}a_{i}^{2}\right\},
\end{equation}
and one can further deduce that
\begin{equation}
1-\eta a_{i^{\ast}}^{2}>0
\qquad\text{if and only if}
\qquad
\left|1-\eta\min_{1\leq i\leq N}a_{i}^{2}\right|
>\left|1-\eta\max_{1\leq i\leq N}a_{i}^{2}\right|,
\end{equation}
which is equivalent to
\begin{equation}
\left|1-\eta\min_{1\leq i\leq N}a_{i}^{2}\right|^{2}
>\left|1-\eta\max_{1\leq i\leq N}a_{i}^{2}\right|^{2},
\end{equation}
which holds if and only if 
\begin{equation}
1-2\eta\min_{1\leq i\leq N}a_{i}^{2}
+\left(\eta\min_{1\leq i\leq N}a_{i}^{2}\right)^{2}
>1-2\eta\max_{1\leq i\leq N}a_{i}^{2}
+\left(\eta\max_{1\leq i\leq N}a_{i}^{2}\right)^{2},
\end{equation}
that is equivalent to
\begin{equation}
\eta\min_{1\leq i\leq N}a_{i}^{2}+\eta\max_{1\leq i\leq N}a_{i}^{2}<2.    
\end{equation}
Hence, we conclude that
\begin{equation}
\mathbb{E}\left[ \mbox{sign}(1-\eta a_{i_*}^2)|i_{\ast}=1\right]
=2\mathbb{P}\left(\min_{1\leq i\leq N}a_{i}^{2}+\max_{1\leq i\leq N}a_{i}^{2}<\frac{2}{\eta}\right)-1.
\end{equation}
Note that $a_{i}$ are i.i.d. distributed with a continuous distribution.
Let $f_{a}(x)$ and $F_{a}(x)$ be the probability density
function and the cumulative distribution function of $a_{i}^{2}$. 
By definition, $0\leq\min_{1\leq i\leq N}a_{i}^{2}\leq\max_{1\leq i\leq N}a_{i}^{2}$. 
For any $0<x<y<\infty$, we have
\begin{align*}
\mathbb{P}\left(x<\min_{1\leq i\leq N}a_{i}^{2}, \max_{1\leq i\leq N}a_{i}^{2}<y\right)
&=\mathbb{P}\left(x<a_{i}^{2}<y\,\,\text{for any $i=1,2,\ldots,N$}\right)
\\
&=\left(\mathbb{P}(x<a_{1}^{2}<y)\right)^{N}
=\left(F_{a}(y)-F_{a}(x)\right)^{N}.
\end{align*}
Let $f_{a}(x,y)$ be the joint probability density function of 
$\min_{1\leq i\leq N}a_{i}^{2}$ and $\max_{1\leq i\leq N}a_{i}^{2}$. 
Then, for any $0<x<y<\infty$,
\begin{align*}
f_{a}(x,y)=-\frac{\partial^{2}}{\partial x\partial y}\left(\left(F_{a}(y)-F_{a}(x)\right)^{N}\right)
=N(N-1)\left(F_{a}(y)-F_{a}(x)\right)^{N-2}f_{a}(x)f_{a}(y).
\end{align*}
This implies that
\begin{align*}
&\mathbb{P}\left(\min_{1\leq i\leq N}a_{i}^{2}+\max_{1\leq i\leq N}a_{i}^{2}<\frac{2}{\eta}\right)
\\
&=\iint_{0<x<y<\infty, x+y<\frac{2}{\eta}}f_{a}(x,y)dxdy
\\
&=\int_{0}^{\frac{2}{\eta}}\int_{0}^{\min\left(y,\frac{2}{\eta}-y\right)}N(N-1)\left(F_{a}(y)-F_{a}(x)\right)^{N-2}f_{a}(x)f_{a}(y)dxdy
\\
&=\int_{0}^{\frac{2}{\eta}}N\left(F_{a}(y)-F_{a}(x)\right)^{N-1}\bigg|_{x=\min\left(y,\frac{2}{\eta}-y\right)}^{x=0}f(y)dy 
\\
&=\int_{0}^{\frac{2}{\eta}}N\left(\left(F_{a}(y)\right)^{N-1}-\left(F_{a}(y)-F_{a}\left(\min\left(y,\frac{2}{\eta}-y\right)\right)\right)^{N-1}\right)f_{a}(y)dy
\\
&=\int_{0}^{\frac{1}{\eta}}N\left(F_{a}(y)\right)^{N-1}f_{a}(y)dy
+\int_{\frac{1}{\eta}}^{\frac{2}{\eta}}N\left(\left(F_{a}(y)\right)^{N-1}-\left(F_{a}(y)-F_{a}\left(\frac{2}{\eta}-y\right)\right)^{N-1}\right)f_{a}(y)dy
\\
&=\int_{0}^{\frac{2}{\eta}}N\left(F_{a}(y)\right)^{N-1}f_{a}(y)dy
-\int_{\frac{1}{\eta}}^{\frac{2}{\eta}}N\left(F_{a}(y)-F_{a}\left(\frac{2}{\eta}-y\right)\right)^{N-1}f_{a}(y)dy
\\
&=\left(F_{a}\left(2/\eta\right)\right)^{N}
-\int_{\frac{1}{\eta}}^{\frac{2}{\eta}}N\left(F_{a}(y)-F_{a}\left(\frac{2}{\eta}-y\right)\right)^{N-1}f_{a}(y)dy.
\end{align*}
Hence, we conclude that
\begin{align*}
&\mathbb{E}\left[L_{i_* i_*}\cdot \mbox{sign}(1-\eta a_{i_*}^2)\right]
\\
&=\frac{1}{N}\sum_{i=1}^{N}L_{ii}\cdot\left[2\left(\left(F_{a}\left(2/\eta\right)\right)^{N}
-\int_{\frac{1}{\eta}}^{\frac{2}{\eta}}N\left(F_{a}(y)-F_{a}\left(\frac{2}{\eta}-y\right)\right)^{N-1}f_{a}(y)dy\right)-1\right].
\end{align*}

When $d=1$ and $b=1$, we can also compute that
\begin{align*}
&\mathbb{E}\left\Vert I_{N}-\eta H\right\Vert^{s}
\\
&=\mathbb{E}\left[\left(\max_{1\leq i\leq N}\left|1-\eta a_{i}^{2}\right|\right)^{s}\right]
=\mathbb{E}\left[\max_{1\leq i\leq N}\left|1-\eta a_{i}^{2}\right|^{s}\right] 
\\
&=\int_{0}^{\infty}\mathbb{P}\left(\max_{1\leq i\leq N}\left|1-\eta a_{i}^{2}\right|^{s}>t\right)dt
\\
&=\int_{0}^{\infty}\left(1-\mathbb{P}\left(\max_{1\leq i\leq N}\left|1-\eta a_{i}^{2}\right|^{s}\leq t\right)\right)dt
\\
&=\int_{0}^{\infty}\left(1-\left(\mathbb{P}\left(\left|1-\eta a_{1}^{2}\right|^{s}\leq t\right)\right)^{N}\right)dt
\\
&=\int_{0}^{\infty}\left(1-\left(\mathbb{P}\left(\max\left\{0,(1-t^{\frac{1}{s}})\right\}\leq\eta a_{1}^{2}\leq 1+t^{\frac{1}{s}}\right)\right)^{N}\right)dt
\\
&=\int_{0}^{1}\left(1-\left(F_{a}\left(\frac{1+t^{\frac{1}{s}}}{\eta}\right)-F_{a}\left(\frac{1-t^{\frac{1}{s}}}{\eta}\right)\right)^{N}\right)dt
+\int_{1}^{\infty}\left(1-\left(F_{a}\left(\frac{1+t^{\frac{1}{s}}}{\eta}\right)\right)^{N}\right)dt.
\end{align*}
This completes the proof of Lemma~\ref{lem:first:order}.
\hfill $\Box$ 


Now, we are ready to prove Theorem~\ref{thm:first:order:alpha}.

Under the assumption that
$\hat{\rho}_{dis}=\mathbb{E}\left[\log(\max_{1\leq i\leq N}|1-\eta a_{i}^{2}|)\right]<0$,
we have $\hat{\rho}<0$ as well for sufficiently small $\delta$. 
Thus, there exist unique positive values $\hat{\alpha}$ and $\hat{\alpha}_{dis}$
such that $\hat{h}(\hat{\alpha})=1$ and $\hat{h}_{dis}(\hat{\alpha}_{dis})=1$.

Let us write $\hat{h}(s)=\hat{h}(s,\delta)$ to emphasize
the dependence on $\delta$ so that
\begin{equation}\label{compute:from}
\hat{h}(s,\delta)=\mathbb{E}\left[\left\Vert\mathcal{W}-\eta H\right\Vert^{s}\right].    
\end{equation}
Then we have $\hat{h}_{dis}(s)=\hat{h}(s,0)$.
By differentiating $\hat{h}(\hat{\alpha},\delta)=1$ with respect
to $\delta$, we obtain
\begin{equation}\label{set:delta:0}
\frac{\partial\hat{h}}{\partial s}(\hat{\alpha},\delta)\cdot\frac{\partial\hat{\alpha}}{\partial\delta}+\frac{\partial\hat{h}}{\partial\delta}(\hat{\alpha},\delta)=0.
\end{equation}
We can compute from \eqref{compute:from} that
\begin{equation*}
\frac{\partial\hat{h}}{\partial s}(\hat{\alpha},\delta)\bigg|_{\delta=0}
=\mathbb{E}\left[\log\left(\left\Vert\mathcal{W}-\eta H\right\Vert\right)\left\Vert\mathcal{W}-\eta H\right\Vert^{\hat{\alpha}}\right]\bigg|_{\delta=0}
=\mathbb{E}\left[\log\left(\left\Vert I_{N}-\eta H\right\Vert\right)\left\Vert I_{N}-\eta H\right\Vert^{\hat{\alpha}_{dis}}\right].
\end{equation*}
We can further compute that
\begin{align*}
\mathbb{E}\left[\log\left(\left\Vert I_{N}-\eta H\right\Vert\right)\left\Vert I_{N}-\eta H\right\Vert^{\hat{\alpha}_{dis}}\right]
=\mathbb{E}\left[\log\left(\max_{1\leq i\leq N}\left|1-\eta a_{i}^{2}\right|\right)
\max_{1\leq i\leq N}\left|1-\eta a_{i}^{2}\right|^{\hat{\alpha}_{dis}}\right].
\end{align*}
Notice that for any $x>0$, 
\begin{align*}
\mathbb{P}\left(\max_{1\leq i\leq N}\left|1-\eta a_{i}^{2}\right|\leq x\right)
&=\left(\mathbb{P}\left(\left|1-\eta a_{1}^{2}\right|\leq x\right)\right)^{N}
\\
&=\left(\mathbb{P}\left(\max\left(0,\frac{1-x}{\eta}\right)\leq a_{i}^{2}\leq\frac{1+x}{\eta}\right)\right)^{N}
\\
&=
\begin{cases}
\left(F_{a}\left(\frac{1+x}{\eta}\right)-F_{a}\left(\frac{1-x}{\eta}\right)\right)^{N} &\text{if $0<x<1$},
\\
\left(F_{a}\left(\frac{1+x}{\eta}\right)\right)^{N} &\text{if $x\geq 1$},
\end{cases}
\end{align*}
where we recall that $F_{a}$ is the cumulative distribution function of $a_{i}^{2}$.
By differentiating with respect to $x$, we obtain
\begin{align*}
&\frac{d}{dx}\mathbb{P}\left(\max_{1\leq i\leq N}\left|1-\eta a_{i}^{2}\right|\leq x\right)
\\
&=\begin{cases}
\frac{N}{\eta}\left(F_{a}\left(\frac{1+x}{\eta}\right)-F_{a}\left(\frac{1-x}{\eta}\right)\right)^{N-1}
\left(f_{a}\left(\frac{1+x}{\eta}\right)+f_{a}\left(\frac{1-x}{\eta}\right)\right) &\text{if $0<x<1$},
\\
\frac{N}{\eta}\left(F_{a}\left(\frac{1+x}{\eta}\right)\right)^{N-1}f_{a}\left(\frac{1+x}{\eta}\right) &\text{if $x\geq 1$},
\end{cases}
\end{align*}
where we recall that $f_{a}$ is the probability density function of $a_{i}^{2}$.
Hence, we conclude that
\begin{align*}
&\frac{\partial\hat{h}}{\partial s}(\hat{\alpha},\delta)\bigg|_{\delta=0}
\\
&=\mathbb{E}\left[\log\left(\max_{1\leq i\leq N}\left|1-\eta a_{i}^{2}\right|\right)
\max_{1\leq i\leq N}\left|1-\eta a_{i}^{2}\right|^{\hat{\alpha}_{dis}}\right]
\\
&=\int_{0}^{1}\log(x)x^{\hat{\alpha}_{dis}}\frac{N}{\eta}\left(F_{a}\left(\frac{1+x}{\eta}\right)-F_{a}\left(\frac{1-x}{\eta}\right)\right)^{N-1}
\left(f_{a}\left(\frac{1+x}{\eta}\right)+f_{a}\left(\frac{1-x}{\eta}\right)\right)dx
\\
&\qquad
+\int_{1}^{\infty}\log(x)x^{\hat{\alpha}_{dis}}\frac{N}{\eta}\left(F_{a}\left(\frac{1+x}{\eta}\right)\right)^{N-1}
f_{a}\left(\frac{1+x}{\eta}\right)dx.
\end{align*}

By setting $\delta=0$ in \eqref{set:delta:0}, we get
\begin{align}
\frac{\partial\hat{\alpha}}{\partial\delta}\bigg|_{\delta=0}
=\frac{-\frac{\partial\hat{h}}{\partial\delta}(\hat{\alpha},\delta)|_{\delta=0}}{\frac{\partial\hat{h}}{\partial s}(\hat{\alpha},\delta)|_{\delta=0}}.
\end{align}
Finally, the proof is completed by applying Lemma~\ref{lem:first:order}.
\hfill $\Box$ 

\subsubsection{Proof of Corollary~\ref{cor:first:order:comparison}}

\begin{corollary}[Complete Statement of Corollary~\ref{cor:first:order:comparison}]\label{cor:first:order:comparison:extended}
(i) Under the assumptions in Theorem~\ref{thm:first:order:alpha-extended}, 
further assume that the stepsize $\eta>\frac{2}{F_{a}^{-1}(2^{-\frac{1}{N}})}$.
For sufficiently small $\delta$, $\hat{\alpha}<\hat{\alpha}_{dis}$, i.e.
the tail gets heavier with the presence of network effect.
(ii) Further assume that $a_{i}$ are i.i.d. $\mathcal{N}(0,\sigma^{2})$ distributed.
Then, as $\delta\rightarrow 0$, we have the first-order expansion \eqref{alpha:expansion} with 
$f_{a}(x)=\frac{1}{\sigma\sqrt{2\pi x}}e^{-\frac{x}{2\sigma^{2}}}$
and $F_{a}(x)=\text{erf}\left(\sqrt{\frac{x}{2\sigma^{2}}}\right)$
in \eqref{given:1}-\eqref{given:2}.
Moreover, when the variance $\sigma^{2}>\frac{1}{\eta}(\text{erf}^{-1}(2^{-\frac{1}{N}}))^{-2}$,
for sufficiently small $\delta$, $\hat{\alpha}<\hat{\alpha}_{dis}$, i.e.
the tail gets heavier with the presence of network effect.
\end{corollary}

Now we are ready to prove Corollary~\ref{cor:first:order:comparison}.

(i) First, we notice that
\begin{equation}
\mathbb{E}\left[\log\left(\max_{1\leq i\leq N}\left|1-\eta a_{i}^{2}\right|\right)
\max_{1\leq i\leq N}\left|1-\eta a_{i}^{2}\right|^{\hat{\alpha}_{dis}}\right]
=\frac{\partial\hat{h}}{\partial s}(\hat{\alpha},\delta)\bigg|_{\delta=0}
=\hat{h}'_{c}(\hat{\alpha}_{dis})>0,
\end{equation}
since $\hat{\alpha}_{dis}$ is the unique positive value such that $\hat{h}_{dis}(\hat{\alpha}_{dis})=1$
and $\hat{h}_{dis}(0)=1$ and $\hat{h}_{dis}'(0)<0$ and $\hat{h}_{dis}(s)$ is convex in $s$ (see \citep{ht_sgd_quad})
so that $\hat{h}'_{c}(\hat{\alpha}_{dis})>0$. 

For sufficiently small $\delta$, if follows from Theorem~\ref{thm:first:order:alpha} that
$\hat{\alpha}<\hat{\alpha}_{dis}$ if and only if 
\begin{equation}\label{to:hold}
2\left(\left(F_{a}\left(\frac{2}{\eta}\right)\right)^{N}
-\int_{\frac{1}{\eta}}^{\frac{2}{\eta}}N\left(F_{a}(y)-F_{a}\left(\frac{2}{\eta}-y\right)\right)^{N-1}f_{a}(y)dy\right)-1<0.
\end{equation}
Note that a sufficient condition for \eqref{to:hold} to hold is that
\begin{equation}
2\left(F_{a}\left(\frac{2}{\eta}\right)\right)^{N}<1,
\end{equation}
which is equivalent to
\begin{equation}
\eta>\frac{2}{F_{a}^{-1}(2^{-\frac{1}{N}})}.    
\end{equation}

(ii) Let $f_{a}(x)$ and $F_{a}(x)$ be the probability density
function and the cumulative distribution function of $a_{i}^{2}$. 
In particular, when $a_{i}\sim\mathcal{N}(0,\sigma^{2})$, we have
for any $x>0$, $\mathbb{P}(a_{i}^{2}\leq x)=\mathbb{P}(-\sqrt{x}\leq a_{i}\leq\sqrt{x})
=\frac{1}{\sigma\sqrt{2\pi}}\int_{-\sqrt{x}}^{\sqrt{x}}e^{-\frac{y^{2}}{2\sigma^{2}}}dy$
so that for any $x>0$,
\begin{equation}
f_{a}(x)=\frac{d}{dx}\frac{1}{\sigma\sqrt{2\pi}}\int_{-\sqrt{x}}^{\sqrt{x}}e^{-\frac{y^{2}}{2\sigma^{2}}}dy
=\frac{1}{\sigma\sqrt{2\pi x}}e^{-\frac{x}{2\sigma^{2}}},
\end{equation}
and thus
\begin{equation}
F_{a}(x)=\int_{0}^{x}\frac{1}{\sigma\sqrt{2\pi y}}e^{-\frac{y}{2\sigma^{2}}}dy
=\text{erf}\left(\sqrt{\frac{x}{2\sigma^{2}}}\right),
\end{equation}
where $\text{erf}$ denotes the error function, 
i.e. $\text{erf}(x):=\frac{2}{\sqrt{\pi}}\int_{0}^{x}e^{-t^{2}}dt$.
By applying Theorem~\ref{thm:first:order:alpha}, 
we conclude that 
as $\delta\rightarrow 0$, we have the first-order expansion \eqref{alpha:expansion} with 
$f_{a}(x)=\frac{1}{\sigma\sqrt{2\pi x}}e^{-\frac{x}{2\sigma^{2}}}$
and $F_{a}(x)=\text{erf}\left(\sqrt{\frac{x}{2\sigma^{2}}}\right)$
in \eqref{given:1}-\eqref{given:2}.

Finally, it follows from (i) and its proof that
if 
\begin{equation}\label{suff:cond}
2\left(F_{a}\left(\frac{2}{\eta}\right)\right)^{N}<1,
\end{equation}
then, for sufficiently small $\delta$, $\hat{h}(s)>\hat{h}_{dis}(s)$. 
When $a_{i}$ are i.i.d. $\mathcal{N}(0,\sigma^{2})$ distributed, 
the condition \eqref{suff:cond} is equivalent to
\begin{equation}
F_{a}\left(\frac{2}{\eta}\right)
=\text{erf}\left(\sqrt{\frac{1}{\eta\sigma^{2}}}\right)<2^{-\frac{1}{N}}.    
\end{equation}
This completes the proof.
\hfill $\Box$ 


\subsubsection{Proof of Corollary~\ref{cor:DSGD:CSGD}}
The result directly follows from Proposition~\ref{prop:disconnected:centralized}
and Theorem~\ref{thm:first:order:alpha}.
\hfill $\Box$

\subsection{Proofs of Results in Section~\ref{sec:conv}}

\subsubsection{Proof of Theorem~\ref{thm:moments}}
We adapt the proof of Theorem~6 in \citet{ht_sgd_quad}.
We recall that 
\begin{equation}
x^{(k)}=M^{(k)}x^{(k-1)}+q^{(k)},
\end{equation}
which implies that
\begin{equation}
\left\Vert x^{(k)}\right\Vert
\leq\left\Vert M^{(k)}x^{(k-1)}\right\Vert
+\left\Vert q^{(k)}\right\Vert.
\end{equation}

(i) If the tail-index $\hat{\alpha}\leq 1$,
then for any $0<p<\hat{\alpha}$, 
we have $\hat{h}(p)=\mathbb{E}\Vert M^{(k)}\Vert^{p}<1$
and moreover by applying Lemma~23 in \citet{ht_sgd_quad},
\begin{equation}
\left\Vert x^{(k)}\right\Vert^{p}
\leq\left\Vert M^{(k)}x^{(k-1)}\right\Vert^{p}
+\left\Vert q^{(k)}\right\Vert^{p}.
\end{equation}
so that
\begin{align*}
\mathbb{E}\left\Vert x^{(k)}\right\Vert^{p}
&\leq 
\mathbb{E}\left\Vert M^{(k)}\right\Vert^{p}\left\Vert x^{(k-1)}\right\Vert^{p}
+\mathbb{E}\left\Vert q^{(k)}\right\Vert^{p}
\\
&=
\mathbb{E}\left\Vert M^{(k)}\right\Vert^{p}\mathbb{E}\left\Vert x^{(k-1)}\right\Vert^{p}
+\left\Vert q^{(k)}\right\Vert^{p}
=\hat{h}(p)\mathbb{E}\left\Vert x^{(k-1)}\right\Vert^{p}
+\mathbb{E}\left\Vert q^{(1)}\right\Vert^{p},
\end{align*}
with $\hat{h}(p)\in(0,1)$, where we used the fact that $M^{(k)}$ is independent of $x^{(k-1)}$.
By iterating over $k$, we get
\begin{equation}
\mathbb{E}\left\Vert x^{(k)}\right\Vert^{p}
\leq
\left(\hat{h}(p)\right)^{k}\mathbb{E}\left\Vert x^{(0)}\right\Vert^{p}
+\frac{1-(\hat{h}(p))^{k}}{1-\hat{h}(p)}\mathbb{E}\left\Vert q^{(1)}\right\Vert^{p}.
\end{equation}

(ii) If the tail-index $\hat{\alpha}>1$, 
then for any $1<p<\hat{\alpha}$, 
by applying Lemma~23 in \citet{ht_sgd_quad}, 
for any $\epsilon>0$, we have
\begin{equation}
\left\Vert x^{(k)}\right\Vert^{p}
\leq
(1+\epsilon)\left\Vert M^{(k)}x^{(k-1)}\right\Vert^{p}
+\frac{(1+\epsilon)^{\frac{p}{p-1}}-(1+\epsilon)}{\left((1+\epsilon)^{\frac{1}{p-1}}-1\right)^{p}}\left\Vert q^{(k)}\right\Vert^{p},
\end{equation}
which (similar as in (i)) implies that
\begin{equation}
\mathbb{E}\left\Vert x^{(k)}\right\Vert^{p}
\leq
(1+\epsilon)\hat{h}(p)
\mathbb{E}\left\Vert x^{(k-1)}\right\Vert^{p}
+\frac{(1+\epsilon)^{\frac{p}{p-1}}-(1+\epsilon)}{\left((1+\epsilon)^{\frac{1}{p-1}}-1\right)^{p}}
\mathbb{E}\left\Vert q^{(1)}\right\Vert^{p}.
\end{equation}
We choose $\epsilon>0$
so that $(1+\epsilon)\hat{h}(p)<1$.
By iterating over $k$, we get
\begin{equation}
\mathbb{E}\left\Vert x^{(k)}\right\Vert^{p}
\leq
((1+\epsilon)\hat{h}(p))^{k}\mathbb{E}\left\Vert x^{(0)}\right\Vert^{p}
+\frac{1-((1+\epsilon)\hat{h}(p))^{k}}{1-(1+\epsilon)\hat{h}(p)}
\frac{(1+\epsilon)^{\frac{p}{p-1}}-(1+\epsilon)}{\left((1+\epsilon)^{\frac{1}{p-1}}-1\right)^{p}}
\mathbb{E}\left\Vert q^{(1)}\right\Vert^{p}.
\end{equation}
The proof is complete.
\hfill $\Box$ 


\subsubsection{Proof of Theorem~\ref{thm:conv}}
We adapt the proof of Theorem~8 in \citet{ht_sgd_quad}.
For any $\nu^{(0)},\tilde{\nu}^{(0)}\in\mathcal{P}_{p}(\mathbb{R}^{Nd})$, 
there exists a couple $x^{(0)}\sim\nu^{(0)}$ and $\tilde{x}^{(0)}\sim\tilde{\nu}^{(0)}$ independent
of $\left(M^{(k)},q^{(k)}\right)_{k\in\mathbb{N}}$ and
$\mathcal{W}_{p}^{p}\left(\nu^{(0)},\tilde{\nu}^{(0)}\right)=\mathbb{E}\left\Vert x^{(0)}-\tilde{x}^{(0)}\right\Vert^{p}$.
We define $x^{(k)}$ and $\tilde{x}^{(k)}$ starting from $x^{(0)}$ and $\tilde{x}^{(0)}$ respectively,
via the iterates
\begin{align}
&x^{(k)}=M^{(k)}x^{(k-1)}+q^{(k)},
\\
&\tilde{x}^{(k)}=M^{(k)}\tilde{x}^{(k-1)}+q^{(k)},
\end{align}
and let $\nu^{(k)}$ and $\tilde{\nu}^{(k)}$ denote
the probability laws of $x^{(k)}$ and $\tilde{x}^{(k)}$ respectively. 
For any $p<\hat{\alpha}$, since $\mathbb{E}\Vert M^{(k)}\Vert^{\hat{\alpha}}=1$
and $\mathbb{E}\Vert q^{(k)}\Vert^{\hat{\alpha}}<\infty$, 
we have $\nu^{(k)},\tilde{\nu}^{(k)}\in\mathcal{P}_{p}(\mathbb{R}^{Nd})$
for any $k$. Moreover, we have
\begin{align*}
\mathbb{E}\left\Vert x^{(k)}-\tilde{x}^{(k)}\right\Vert^{p}
\leq
\mathbb{E}\left[\left\Vert M^{(k)}\left(x^{(k-1)}-\tilde{x}^{(k-1)}\right)\right\Vert^{p}\right]
&\leq\mathbb{E}\left[\left\Vert M^{(k)}\right\Vert^{p}\right]
\mathbb{E}\left[\left\Vert x^{(k-1)}-\tilde{x}^{(k-1)}\right\Vert^{p}\right]\\
&=\hat{h}(p)\mathbb{E}\left[\left\Vert x^{(k-1)}-\tilde{x}^{(k-1)}\right\Vert^{p}\right],
\end{align*}
which by iterating implies that
\begin{equation}
\mathcal{W}_{p}^{p}\left(\nu^{(k)},\tilde{\nu}^{(k)}\right)
\leq
\mathbb{E}\left\Vert x^{(k)}-\tilde{x}^{(k)}\right\Vert^{p}
\leq
\left(\hat{h}(p)\right)^{k}
\mathbb{E}\left\Vert x^{(0)}-\tilde{x}^{(0)}\right\Vert^{p}
=\left(\hat{h}(p)\right)^{k}
\mathcal{W}_{p}^{p}\left(\nu^{(0)},\tilde{\nu}^{(0)}\right).
\end{equation}
By letting $\tilde{\nu}^{(0)}=\nu^{(\infty)}$, the probability
law of the stationary distribution $x^{(\infty)}$, 
we conclude that
\begin{equation}
\mathcal{W}_{p}\left(\nu^{(k)},\nu^{(\infty)}\right)
\leq
\left(\left(\hat{h}(p)\right)^{1/p}\right)^{k}
\mathcal{W}_{p}\left(\nu^{(0)},\nu^{(\infty)}\right).
\end{equation}
Finally, notice that $1\leq p<\hat{\alpha}$,
and therefore $\hat{h}(p)<1$.
The proof is complete.
\hfill $\Box$ 

\subsection{Proofs of Results in Section~\ref{sec:general:case}}

\subsubsection{Proof of Lemma~\ref{lem:general:d:b}}
By taking the expectations in \eqref{without:expectation-2}, we get
\begin{equation}
\mathbb{E}\left\| I_{Nd} - \eta H - \delta (L\otimes I_d) \right\|^s 
=\mathbb{E}\left\| I_{Nd} - \eta H \right\|^s
-s\delta\mathbb{E}\left[\mbox{sign} \left(1-\eta\lambda_{j_*(i_{\ast})}(H_{i_*})\right)  L_{i_* i_*}\right]
+o(\delta).
\end{equation}
Note that $H_{i}=\frac{1}{b}\sum_{j=1}^{b}a_{i,j}(a_{i,j})^{T}$ are i.i.d.
distributed, where $a_{i,j}$ are i.i.d. following a continuous distribution. 
Therefore, $i_{\ast}$ is uniformly distributed on $\{1,2,\ldots,N\}$ and we can compute that
\begin{align*}
\mathbb{E}\left[\mbox{sign} \left(1-\eta\lambda_{j_*(i_{\ast})}(H_{i_*})\right)  L_{i_* i_*}\right]
&=\sum_{i=1}^{N}\mathbb{P}(i_{\ast}=i)\mathbb{E}\left[\mbox{sign} \left(1-\eta\lambda_{j_*(i_{\ast})}(H_{i_*})\right)  L_{i_* i_*}|i_{\ast}=i\right]
\\
&=\sum_{i=1}^{N}\mathbb{P}(i_{\ast}=i)L_{ii}\mathbb{E}\left[\mbox{sign} \left(1-\eta\lambda_{j_*(i_{\ast})}(H_{i_*})\right)  |i_{\ast}=i\right]
\\
&=\left(\frac{1}{N}\sum_{i=1}^{N}L_{ii}\right)\mathbb{E}\left[ \mbox{sign} \left(1-\eta\lambda_{j_*(i_{\ast})}(H_{i_*})\right)  |i_{\ast}=1\right].
\end{align*}
Moreover, we can compute that
\begin{align*}
\mathbb{E}\left[ \mbox{sign} \left(1-\eta\lambda_{j_*(i_{\ast})}(H_{i_*})\right)\right]
&=\frac{1}{N}\sum_{i=1}^{N}\mathbb{E}\left[ \mbox{sign} \left(1-\eta\lambda_{j_*(i_{\ast})}(H_{i_*})\right)|i_{\ast}=i\right]
\\
&=\mathbb{E}\left[ \mbox{sign} \left(1-\eta\lambda_{j_*(i_{\ast})}(H_{i_*})\right)|i_{\ast}=1\right],
\end{align*}
and we can further compute that
\begin{align*}
\mathbb{E}\left[ \mbox{sign} \left(1-\eta\lambda_{j_*(i_{\ast})}(H_{i_*})\right)|i_{\ast}=1\right]
&=\mathbb{E}\left[ \mbox{sign} \left(1-\eta\lambda_{j_*(i_{\ast})}(H_{i_*})\right)\right]
\\
&=\mathbb{P}\left(\lambda_{j_*(i_{\ast})}(H_{i_*})<1/\eta\right)-\mathbb{P}\left(\lambda_{j_*(i_{\ast})}(H_{i_*})>1/\eta\right)
\\
&=2\mathbb{P}\left(\lambda_{j_*(i_{\ast})}(H_{i_*})<1/\eta\right)-1,
\end{align*}
and hence as $\delta\rightarrow 0$, we have the first-order expansion:
\begin{equation}
\hat{h}(s)
=\hat{h}_{dis}(s)
-\frac{s\delta\sum_{i=1}^{N}L_{ii}}{N}\left(2\mathbb{P}\left(\lambda_{j_*(i_{\ast})}(H_{i_*})<1/\eta\right)-1\right)
+o(\delta),
\end{equation}
where
\begin{align}
\hat{h}(s)=\mathbb{E}\left\| I_{Nd} - \eta H - \delta (L\otimes I_d) \right\|^s 
\qquad
\hat{h}_{dis}(s)=\mathbb{E}\left\| I_{Nd} - \eta H \right\|^s,
\end{align}
which completes the proof.
\hfill $\Box$ 


\subsubsection{Proof of Theorem~\ref{thm:general:d:b}}
Under the assumption that
$\hat{\rho}_{dis}=\mathbb{E}\left[\log\Vert I_{Nd}-\eta H\Vert\right]<0$,
we have $\hat{\rho}<0$ as well for sufficiently small $\delta$. 
Thus, there exist unique positive values $\hat{\alpha}$ and $\hat{\alpha}_{dis}$
such that $\hat{h}(\hat{\alpha})=1$ and $\hat{h}_{dis}(\hat{\alpha}_{dis})=1$.

Let us write $\hat{h}(s)=\hat{h}(s,\delta)$ to emphasize
the dependence on $\delta$ so that
\begin{equation}\label{compute:from:general}
\hat{h}(s,\delta)=\mathbb{E}\left[\left\Vert\mathcal{W}-\eta H\right\Vert^{s}\right].    
\end{equation}
Then we have $\hat{h}_{dis}(s)=\hat{h}(s,0)$.
By differentiating $\hat{h}(\hat{\alpha},\delta)=1$ with respect
to $\delta$, we obtain
\begin{equation}\label{set:delta:0:general}
\frac{\partial\hat{h}}{\partial s}(\hat{\alpha},\delta)\cdot\frac{\partial\hat{\alpha}}{\partial\delta}+\frac{\partial\hat{h}}{\partial\delta}(\hat{\alpha},\delta)=0.
\end{equation}
We can compute from \eqref{compute:from:general} that
\begin{align*}
\frac{\partial\hat{h}}{\partial s}(\hat{\alpha},\delta)\bigg|_{\delta=0}
&=\mathbb{E}\left[\log\left(\left\Vert\mathcal{W}-\eta H\right\Vert\right)\left\Vert\mathcal{W}-\eta H\right\Vert^{\hat{\alpha}}\right]\bigg|_{\delta=0}
\\
&=\mathbb{E}\left[\log\left(\left\Vert I_{Nd}-\eta H\right\Vert\right)\left\Vert I_{Nd}-\eta H\right\Vert^{\hat{\alpha}_{dis}}\right].
\end{align*}
Hence, by letting $\delta=0$ in \eqref{set:delta:0:general} and applying Lemma~\ref{lem:general:d:b}, we complete the proof.
\hfill $\Box$ 

\subsubsection{Proof of Theorem~\ref{thm:first:order:alpha:general:b}}
When $d=1$, we can compute that
$\hat{\rho}_{dis}=\mathbb{E}\left[\log\left(\max_{1\leq i\leq N}\left|1-\frac{\eta}{b}\sum_{j=1}^{b}a_{i,j}^{2}\right|\right)\right]$.
Under the assumption that
$\hat{\rho}_{dis}<0$,
we have $\hat{\rho}<0$ as well for sufficiently small $\delta$. 
Thus, there exist unique positive values $\hat{\alpha}$ and $\hat{\alpha}_{dis}$
such that $\hat{h}(\hat{\alpha})=1$ and $\hat{h}_{dis}(\hat{\alpha}_{dis})=1$.

When $d=1$, we can compute that
\begin{align*}
\mathbb{P}\left(\lambda_{j_*(i_{\ast})}(H_{i_*})<1/\eta\right)
&=\mathbb{P}\left(\frac{1}{b}\sum_{j=1}^{b}a_{i_{\ast},j}^{2}<\frac{1}{\eta}\right)
\\
&=\mathbb{P}\left(\min_{1\leq i\leq N}\frac{1}{b}\sum_{j=1}^{b}a_{i,j}^{2}
+\max_{1\leq i\leq N}\frac{1}{b}\sum_{j=1}^{b}a_{i,j}^{2}<\frac{2}{\eta}\right)
\end{align*}
Moreover, we can compute that
\begin{align*}
&\mathbb{E}\left[\log\left(\left\Vert I_{N}-\eta H\right\Vert\right)\left\Vert I_{N}-\eta H\right\Vert^{\hat{\alpha}_{dis}}\right]
\\
&=\mathbb{E}\left[\log\left(\max_{1\leq i\leq N}\left|1-\frac{\eta}{b}\sum_{j=1}^{b}a_{i,j}^{2}\right|\right)\max_{1\leq i\leq N}\left|1-\frac{\eta}{b}\sum_{j=1}^{b}a_{i,j}^{2}\right|^{\hat{\alpha}_{dis}}\right]
\end{align*}
The rest of the proof follows similarly as in the proof of Theorem~\ref{thm:first:order:alpha}.
\hfill $\Box$ 

\subsubsection{Proof of Corollary~\ref{cor:first:order:alpha:general:b}}
When $d=1$ and $a_{i,j}$ are i.i.d. $\mathcal{N}(0,\sigma^{2})$, we have
\begin{equation}
\sum_{j=1}^{b}a_{i,j}^{2}=\sigma^{2}\chi^{2}(b),
\end{equation}
in distribution, where $\chi^{2}(b)$ denotes a chi-square random variable
with $b$ degrees of freedom and hence
\begin{equation}
f_{b}(x)=\frac{1}{2^{b/2}\Gamma(b/2)}\frac{x^{\frac{b}{2}-1}}{\sigma^{b}}e^{-\frac{x}{2\sigma^{2}}},
\end{equation}
and
\begin{equation}
F_{b}(x)=\frac{1}{\Gamma(b/2)}\gamma\left(\frac{b}{2},\frac{x}{2\sigma^{2}}\right),
\end{equation}
where $\Gamma(\cdot)$ is the gamma function
and $\gamma(\cdot,\cdot)$ is the lower incomplete gamma function.
The proof is completed by applying Theorem~\ref{thm:first:order:alpha:general:b}.
\hfill $\Box$

\subsubsection{Proof of Lemma~\ref{lem:general:d:b:i}}
By taking the expectations, we get
\begin{equation}
\mathbb{E}\left\| I_{Nd} - \eta H - \delta (L\otimes I_d) \right\|^s 
=\mathbb{E}\left\| I_{Nd} - \eta H \right\|^s
-s\delta\mathbb{E}\left[\mbox{sign} \left(1-\eta\lambda_{j_*(i_{\ast})}(H_{i_*})\right)  L_{i_* i_*}\right]
+o(\delta),
\end{equation}
and we can compute that
\begin{align*}
&\mathbb{E}\left[\mbox{sign} \left(1-\eta\lambda_{j_*(i_{\ast})}(H_{i_*})\right)  L_{i_* i_*}\right]
\\
&=\sum_{i=1}^{N}\mathbb{P}(i_{\ast}=i)\mathbb{E}\left[\mbox{sign} \left(1-\eta\lambda_{j_*(i_{\ast})}(H_{i_*})\right)  L_{i_* i_*}|i_{\ast}=i\right]
\\
&=\sum_{i=1}^{N}\mathbb{P}(i_{\ast}=i)L_{ii}\mathbb{E}\left[\mbox{sign} \left(1-\eta\lambda_{j_*(i_{\ast})}(H_{i_*})\right)|i_{\ast}=i\right]
\\
&=\sum_{i=1}^{N}\mathbb{P}(i_{\ast}=i)L_{ii}\left(2\mathbb{P}\left(\lambda_{j_*(i_{\ast})}(H_{i_*})<1/\eta|i_{\ast}=i\right)-1\right)
\end{align*}

and hence as $\delta\rightarrow 0$, we have the first-order expansion:
\begin{equation}
\hat{h}(s)
=\hat{h}_{dis}(s)
-s\delta\sum_{i=1}^{N}\mathbb{P}(i_{\ast}=i)L_{ii}\left(2\mathbb{P}\left(\lambda_{j_*(i_{\ast})}(H_{i_*})<1/\eta|i_{\ast}=i\right)-1\right)
+o(\delta),
\end{equation}
where
\begin{align}
\hat{h}(s)=\mathbb{E}\left\| I_{Nd} - \eta H - \delta (L\otimes I_d) \right\|^s 
\qquad
\hat{h}_{dis}(s)=\mathbb{E}\left\| I_{Nd} - \eta H \right\|^s,
\end{align}
which completes the proof.
\hfill $\Box$ 


\subsubsection{Proof of Theorem~\ref{thm:general:d:b:i}}
Under the assumption that
$\hat{\rho}_{dis}=\mathbb{E}\left[\log\Vert I-\eta H\Vert\right]<0$,
we have $\hat{\rho}<0$ as well for sufficiently small $\delta$. 
Thus, there exist unique positive values $\hat{\alpha}$ and $\hat{\alpha}_{dis}$
such that $\hat{h}(\hat{\alpha})=1$ and $\hat{h}_{dis}(\hat{\alpha}_{dis})=1$.

Let us write $\hat{h}(s)=\hat{h}(s,\delta)$ to emphasize
the dependence on $\delta$ so that
\begin{equation}\label{compute:from:general:i}
\hat{h}(s,\delta)=\mathbb{E}\left[\left\Vert\mathcal{W}-\eta H\right\Vert^{s}\right].    
\end{equation}
Then we have $\hat{h}_{dis}(s)=\hat{h}(s,0)$.
By differentiating $\hat{h}(\hat{\alpha},\delta)=1$ with respect
to $\delta$, we obtain
\begin{equation}\label{set:delta:0:general:i}
\frac{\partial\hat{h}}{\partial s}(\hat{\alpha},\delta)\cdot\frac{\partial\hat{\alpha}}{\partial\delta}+\frac{\partial\hat{h}}{\partial\delta}(\hat{\alpha},\delta)=0.
\end{equation}
We can compute from \eqref{compute:from:general:i} that
\begin{align*}
\frac{\partial\hat{h}}{\partial s}(\hat{\alpha},\delta)\bigg|_{\delta=0}
&=\mathbb{E}\left[\log\left(\left\Vert\mathcal{W}-\eta H\right\Vert\right)\left\Vert\mathcal{W}-\eta H\right\Vert^{\hat{\alpha}}\right]\bigg|_{\delta=0}
\\
&=\mathbb{E}\left[\log\left(\left\Vert I_{Nd}-\eta H\right\Vert\right)\left\Vert I_{Nd}-\eta H\right\Vert^{\hat{\alpha}_{dis}}\right].
\end{align*}
Hence, by letting $\delta=0$ in \eqref{set:delta:0:general:i} and applying Lemma~\ref{lem:general:d:b:i}, we complete the proof.
\hfill $\Box$

\subsubsection{Proof of Theorem~\ref{thm:first:order:alpha:general:b:i}}
When $d=1$, 
$\hat{\rho}_{dis}=\mathbb{E}\left[\log\left(\max_{1\leq i\leq N}\left|1-\eta\frac{1}{b_{i}}\sum_{j=1}^{N} a_{i,j}^{2}\right|\right)\right]$.
Under the assumption that
$\hat{\rho}_{dis}<0$,
we have $\hat{\rho}<0$ as well for sufficiently small $\delta$. 
Thus, there exist unique positive values $\hat{\alpha}$ and $\hat{\alpha}_{dis}$
such that $\hat{h}(\hat{\alpha})=1$ and $\hat{h}_{dis}(\hat{\alpha}_{dis})=1$.

When $d=1$, 
$i_{\ast}=\arg\max_{1\leq i\leq N}\left|1-\frac{\eta}{b_{i}}\sum_{j=1}^{b_{i}}a_{i,j}^{2}\right|$.
It is easy to see that
\begin{equation}
i_{\ast}\in\left\{\arg\min_{1\leq i\leq N}\frac{1}{b_{i}}\sum_{j=1}^{b_{i}}a_{i,j}^{2},\arg\max_{1\leq i\leq N}\frac{1}{b_{i}}\sum_{j=1}^{b_{i}}a_{i,j}^{2}\right\},
\end{equation}
and one can further deduce that
\begin{equation}
1-\eta\frac{1}{b_{i_{\ast}}}\sum_{j=1}^{b_{i_{\ast}}}a_{i_{\ast},j}^{2}>0
\quad\text{if and only if}\quad
i_{\ast}=\arg\min_{1\leq i\leq N}\frac{1}{b_{i}}\sum_{j=1}^{b_{i}}a_{i,j}^{2},
\end{equation}
which is equivalent to
\begin{equation}
\left|1-\eta\min_{1\leq i\leq N}\frac{1}{b_{i}}\sum_{j=1}^{b_{i}}a_{i,j}^{2}\right|
>\left|1-\eta\max_{1\leq i\leq N}\frac{1}{b_{i}}\sum_{j=1}^{b_{i}}a_{i,j}^{2}\right|,
\end{equation}
which, by the similar argument as in the proof of Lemma~\ref{lem:first:order},
holds if and only if 
\begin{equation}
\eta\min_{1\leq i\leq N}\frac{1}{b_{i}}\sum_{j=1}^{b_{i}}a_{i,j}^{2}
+\eta\max_{1\leq i\leq N}\frac{1}{b_{i}}\sum_{j=1}^{b_{i}}a_{i,j}^{2}<2.    
\end{equation}
Therefore, we can compute that
\begin{align*}
&\mathbb{E}\left[\mbox{sign} \left(1-\eta\frac{1}{b_{i_{\ast}}}\sum_{j=1}^{b_{i_{\ast}}}a_{i_{\ast},j}^{2}\right)  L_{i_* i_*}\right]
\\
&=\sum_{i=1}^{N}L_{ii}\mathbb{P}\left(\frac{1}{b_{i}}\sum_{j=1}^{b_{i}}a_{i,j}^{2}<\min_{k\neq i}\frac{1}{b_{k}}\sum_{j=1}^{b_{k}}a_{k,j}^{2},\frac{1}{b_{i}}\sum_{j=1}^{b_{i}}a_{i,j}^{2}+\max_{k\neq i}\frac{1}{b_{k}}\sum_{j=1}^{b_{k}}a_{k,j}^{2}<\frac{2}{\eta}\right)
\\
&\qquad
-\sum_{i=1}^{N}L_{ii}\mathbb{P}\left(\frac{1}{b_{i}}\sum_{j=1}^{b_{i}}a_{i,j}^{2}>\max_{k\neq i}\frac{1}{b_{k}}\sum_{j=1}^{b_{k}}a_{k,j}^{2},\frac{1}{b_{i}}\sum_{j=1}^{b_{i}}a_{i,j}^{2}+\min_{k\neq i}\frac{1}{b_{k}}\sum_{j=1}^{b_{k}}a_{k,j}^{2}>\frac{2}{\eta}\right).
\end{align*}

Let $f_{i}(x)$ and $F_{i}(x)$ be the probability density
function and the cumulative distribution function of $\frac{1}{b_{i}}\sum_{j=1}^{b_{i}}a_{i,j}^{2}$. 
For any $0<x<y<\infty$, we have
\begin{align*}
\mathbb{P}\left(x<\min_{k\neq i}\frac{1}{b_{k}}\sum_{j=1}^{b_{k}}a_{k,j}^{2}, \max_{k\neq i}\frac{1}{b_{k}}\sum_{j=1}^{b_{k}}a_{k,j}^{2}<y\right)
&=\mathbb{P}\left(x<\frac{1}{b_{k}}\sum_{j=1}^{b_{k}}a_{k,j}^{2}<y\,\,\text{for any $k\neq i$}\right)
\\
&=\prod_{k\neq i}\left(F_{k}(y)-F_{k}(x)\right).
\end{align*}
Let $f_{i}(x,y)$ be the joint probability density function of 
$\min_{k\neq i}\frac{1}{b_{k}}\sum_{j=1}^{b_{k}}a_{k,j}^{2}$ and $\max_{k\neq i}\frac{1}{b_{k}}\sum_{j=1}^{b_{k}}a_{k,j}^{2}$. 
Then, for any $0<x<y<\infty$,
\begin{align*}
f_{i}(x,y)=-\frac{\partial^{2}}{\partial x\partial y}\left(\prod_{k\neq i}\left(F_{k}(y)-F_{k}(x)\right)\right)
=\sum_{k\neq i}\sum_{j\neq k,i}f_{k}(y)f_{j}(x)\prod_{\ell\neq j,k,i}(F_{\ell}(y)-F_{\ell}(x)).
\end{align*}
This implies that
\begin{align*}
&\sum_{i=1}^{N}L_{ii}\mathbb{P}\left(\frac{1}{b_{i}}\sum_{j=1}^{b_{i}}a_{i,j}^{2}<\min_{k\neq i}\frac{1}{b_{k}}\sum_{j=1}^{b_{k}}a_{k,j}^{2},\frac{1}{b_{i}}\sum_{j=1}^{b_{i}}a_{i,j}^{2}+\max_{k\neq i}\frac{1}{b_{k}}\sum_{j=1}^{b_{k}}a_{k,j}^{2}<\frac{2}{\eta}\right)
\\
&=\sum_{i=1}^{N}L_{ii}\int_{0}^{\frac{2}{\eta}}\int_{x}^{\frac{2}{\eta}}\int_{0}^{\min(x,\frac{2}{\eta}-y)}f_{i}(z)f_{i}(x,y)dzdydx
\\
&=\sum_{i=1}^{N}L_{ii}\int_{0}^{\frac{2}{\eta}}\int_{x}^{\frac{2}{\eta}}F_{i}\left(\min\left(x,\frac{2}{\eta}-y\right)\right)f_{i}(x,y)dydx
\\
&=\sum_{i=1}^{N}L_{ii}\int_{0}^{\frac{2}{\eta}}\int_{x}^{\frac{2}{\eta}}F_{i}\left(\min\left(x,\frac{2}{\eta}-y\right)\right)\sum_{k\neq i}\sum_{j\neq k,i}f_{k}(y)f_{j}(x)\prod_{\ell\neq j,k,i}(F_{\ell}(y)-F_{\ell}(x))dydx.
\end{align*}
Similarly, 
\begin{align*}
&\sum_{i=1}^{N}L_{ii}\mathbb{P}\left(\frac{1}{b_{i}}\sum_{j=1}^{b_{i}}a_{i,j}^{2}>\max_{k\neq i}\frac{1}{b_{k}}\sum_{j=1}^{b_{k}}a_{k,j}^{2},\frac{1}{b_{i}}\sum_{j=1}^{b_{i}}a_{i,j}^{2}+\min_{k\neq i}\frac{1}{b_{k}}\sum_{j=1}^{b_{k}}a_{k,j}^{2}>\frac{2}{\eta}\right)
\\
&=\sum_{i=1}^{N}L_{ii}\int_{0}^{\infty}\int_{x}^{\infty}\int_{\max(y,\frac{2}{\eta}-x)}^{\infty}f_{i}(z)f_{i}(x,y)dzdydx
\\
&=\sum_{i=1}^{N}L_{ii}\int_{0}^{\infty}\int_{x}^{\infty}\left(1-F_{i}\left(\max\left(y,\frac{2}{\eta}-x\right)\right)\right)f_{i}(x,y)dydx
\\
&=\sum_{i=1}^{N}L_{ii}\int_{0}^{\infty}\int_{x}^{\infty}\left(1-F_{i}\left(\max\left(y,\frac{2}{\eta}-x\right)\right)\right)\sum_{k\neq i}\sum_{j\neq k,i}f_{k}(y)f_{j}(x)\prod_{\ell\neq j,k,i}(F_{\ell}(y)-F_{\ell}(x))dydx.
\end{align*}

Finally, we can compute that
\begin{align*}
&\mathbb{E}\left[\log\left(\left\Vert I_{N}-\eta H\right\Vert\right)\left\Vert I_{N}-\eta H\right\Vert^{\hat{\alpha}_{dis}}\right]
\\
&=\mathbb{E}\left[\log\left(\max_{1\leq i\leq N}\left|1-\eta\frac{1}{b_{i}}\sum_{j=1}^{b_{i}}a_{i,j}^{2}\right|\right)
\max_{1\leq i\leq N}\left|1-\eta\frac{1}{b_{i}}\sum_{j=1}^{b_{i}}a_{i,j}^{2}\right|^{\hat{\alpha}_{dis}}\right].
\end{align*}
Notice that for any $x>0$, 
\begin{align*}
\mathbb{P}\left(\max_{1\leq i\leq N}\left|1-\eta\frac{1}{b_{i}}\sum_{j=1}^{b_{i}}a_{i,j}^{2}\right|\leq x\right)
&=\prod_{i=1}^{N}\mathbb{P}\left(\left|1-\eta\frac{1}{b_{i}}\sum_{j=1}^{b_{i}}a_{i,j}^{2}\right|\leq x\right)
\\
&=\prod_{i=1}^{N}\mathbb{P}\left(\max\left(0,\frac{1-x}{\eta}\right)\leq\frac{1}{b_{i}}\sum_{j=1}^{b_{i}}a_{i,j}^{2}\leq\frac{1+x}{\eta}\right)
\\
&=
\begin{cases}
\prod_{i=1}^{N}\left(F_{i}\left(\frac{1+x}{\eta}\right)-F_{i}\left(\frac{1-x}{\eta}\right)\right) &\text{if $0<x<1$},
\\
\prod_{i=1}^{N}F_{i}\left(\frac{1+x}{\eta}\right) &\text{if $x\geq 1$},
\end{cases}
\end{align*}
where we recall that $F_{i}$ is the cumulative distribution function of $\frac{1}{b_{i}}\sum_{j=1}^{b_{i}}a_{i,j}^{2}$.
By differentiating with respect to $x$, we obtain
\begin{align*}
&\frac{d}{dx}\mathbb{P}\left(\max_{1\leq i\leq N}\left|1-\eta\frac{1}{b_{i}}\sum_{j=1}^{b_{i}}a_{i,j}^{2}\right|\leq x\right)
\\
&=\begin{cases}
\sum_{i=1}^{N}
\left(f_{i}\left(\frac{1+x}{\eta}\right)+f_{i}\left(\frac{1-x}{\eta}\right)\right)
\prod_{k\neq i}\left(F_{k}\left(\frac{1+x}{\eta}\right)-F_{k}\left(\frac{1-x}{\eta}\right)\right)
&\text{if $0<x<1$},
\\
\sum_{i=1}^{N}f_{i}\left(\frac{1+x}{\eta}\right)\prod_{k\neq i}^{N}F_{k}\left(\frac{1+x}{\eta}\right) &\text{if $x\geq 1$},
\end{cases}
\end{align*}
where we recall that $f_{i}$ is the probability density function of $\frac{1}{b_{i}}\sum_{j=1}^{b_{i}}a_{i,j}^{2}$.
This completes the proof.
\hfill $\Box$ 

\subsubsection{Proof of Corollary~\ref{cor:first:order:alpha:general:b:i}}
When $d=1$, for every $i$, $a_{i,j}$ are i.i.d. $\mathcal{N}(0,\sigma_{i}^{2})$, we have
\begin{equation}
\frac{1}{b_{i}}\sum_{j=1}^{b_{i}}a_{i,j}^{2}=\frac{\sigma_{i}^{2}}{b_{i}}\chi^{2}(b_{i}),
\end{equation}
in distribution, where $\chi^{2}(b_{i})$ denotes a chi-square random variable
with $b_{i}$ degrees of freedom and hence
\begin{equation}
f_{i}(x)=\frac{1}{2^{b_{i}/2}\Gamma(b_{i}/2)}\frac{x^{\frac{b_{i}}{2}-1}}{(\sigma_{i}^{2}/b_{i})^{b_{i}/2}}e^{-\frac{x}{2(\sigma_{i}^{2}/b_{i})}},
\end{equation}
and
\begin{equation}
F_{i}(x)=\frac{1}{\Gamma(b_{i}/2)}\gamma\left(\frac{b_{i}}{2},\frac{b_{i}x}{2\sigma_{i}^{2}}\right),
\end{equation}
where $\Gamma(\cdot)$ is the gamma function
and $\gamma(\cdot,\cdot)$ is the lower incomplete gamma function.
The proof is completed by applying Theorem~\ref{thm:first:order:alpha:general:b:i}.
\hfill $\Box$ 


\subsubsection{Proof of Corollary~\ref{cor:DSGD:CSGD:general}}
The result directly follows from Proposition~\ref{prop:disconnected:centralized}
and Theorem~\ref{thm:general:d:b}.
\hfill $\Box$ 

\section{Further experiments with synthetic data}
\begin{figure}[h!]
    \centering
    \subfigure[Case II: $d$=100, $b$=1, $\sigma$=1, $\sigma_y$=3, $N$=20 over star network.]{
    \includegraphics[width=0.45\columnwidth]{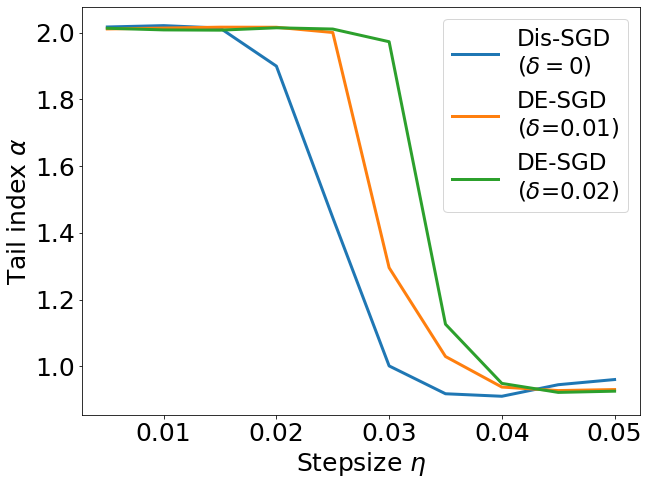}
    \label{fig:eta_tau_eta_d=100}
    }
    \subfigure[Case II: $d$=100, $b$=5, $\sigma$=1, $\sigma_y$=3, $N$=10 over star network.]{
    \includegraphics[width=0.45\columnwidth]{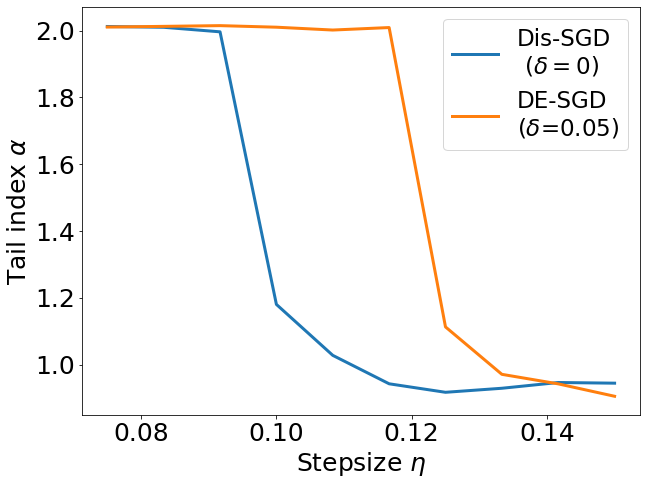}
    \label{fig:eta_tau_eta_fig2}
    }
    \caption{More experiment results illustrating Case II.}
    \label{fig:appendix-case2}
\end{figure}
\paragraph{Further Illustrations of Thm. \ref{thm:first:order:alpha} and Cor. \ref{cor:first:order:comparison}} Due to space limit in the main text, we provide some of our experimental results in Figure~\ref{fig:appendix-case2}. In the main text we have provided Figure~\ref{fig:eta_tau_eta_d=1}, which fit into Case II ($0<\eta_{crit}<\tau<\eta_{max}$) described in the numerical experiments section in the main text. Figure~\ref{fig:eta_tau_eta_d=100} and Figure~\ref{fig:eta_tau_eta_fig2} illustrate Case II as well. In Figure~\ref{fig:eta_tau_eta_d=100}, we set $d=100,b=1,\sigma = 1,\sigma_y = 3$ and generate $\{a_i, y_i\}^n_{i=1}$ by simulating the model~\eqref{eqn:data_generation}. We run the DE-SGD on the star network with $N=20$ and different $\delta$ value $\delta=0.01,0.02$. For Figure~\ref{fig:eta_tau_eta_fig2}, we set $d=100, b=5,\sigma = 1, \sigma_y=3$ and run DE-SGD on the star network with $N=10,\delta=0.05$. We can see that both Dis-SGD and DE-SGD will converge to a heavy-tailed distribution. In the small stepsize regime, Dis-SGD will have heavier tail, and in the big stepsize regime, DE-SGD will have heavier tail. Our observation is the same as our expectation.

\paragraph{Effect of parameters.} In this set of experiments, we investigate the tail-index $\alpha$ of the stationary distribution of DE-SGD over different network topologies with varied stepsize $\eta$ and varied batch size $b$. We set $d=100$, $b=5$, $\sigma = 1$, $\sigma_y = 3$ and generate $\{a_i, y_i\}^n_{i=1}$ by simulating the model~\eqref{eqn:data_generation}. To mimic the C-SGD, we run the single-node SGD in the same experimental setup but with batch-size $bN=40$. We consider 5 different network topologies in the experiments and take the communication matrix $W = I - \delta L$. They are (a) complete network with $\delta = \frac{1}{8}$ (b) cycle network with $\delta = \frac{1}{3}$ (c) hypercube network with $\delta = \frac{1}{3}$ (d) bipartite network with $\delta = \frac{1}{5}$ and (e) star network with $\delta = \frac{1}{8}$. These networks and their basic properties are discussed in Section \ref{sec-network-drawings} and \ref{sec:examples}.


First, we plot how the DE-SGD tail-index $\alpha$ varies with stepsize $\eta$ from $0.16$ to $0.21$ for the five network topologies in Figure~\ref{fig:lin-reg-stepsize}. We observe that $(i)$ The tail-index of DE-SGD over different topologies decreases as stepsizes increase and becomes heavy tailed when stepsize is larger than a threshold (threshold depends on the network structure). $(ii)$ DE-SGD has heavier tails (smaller $\alpha$) compared to C-SGD. This is also consistent with our Corollary~\ref{cor:DSGD:CSGD}. 
In another set of experiments, we investigate the tail-index $\alpha$ of the stationary distribution of DE-SGD over different network topologies with varied batch-size $b$. 
We set stepsize $\eta=0.18$ and vary batch-size $b$ from 1 to 8 for each node in DE-SGD. For C-SGD, we use a batch size of bN as $b$ is varied. 
The results are displayed in Figure~\ref{fig:lin-reg-batch}. 
It is observed that $(i)$ Heavier tails occur for small batch-sizs, the larger batch size gets; we have lighter  behavior ($\alpha=2$ corresponds to Gaussian tails). This is consistent with our result (Theorem~\ref{thm:mono}) and also with the experiments of \citet{panigrahi2019non} in the centralized setting. 
$(ii)$ The tail of the stationary distribution of DE-SGD over all topologies is always heavier than that of the C-SGD. This is in line with our Corollary~\ref{cor:DSGD:CSGD}.
\begin{figure}[t]
\centering
\subfigure[Varied stepsize $\eta$.]{
    \includegraphics[width=0.45\columnwidth]{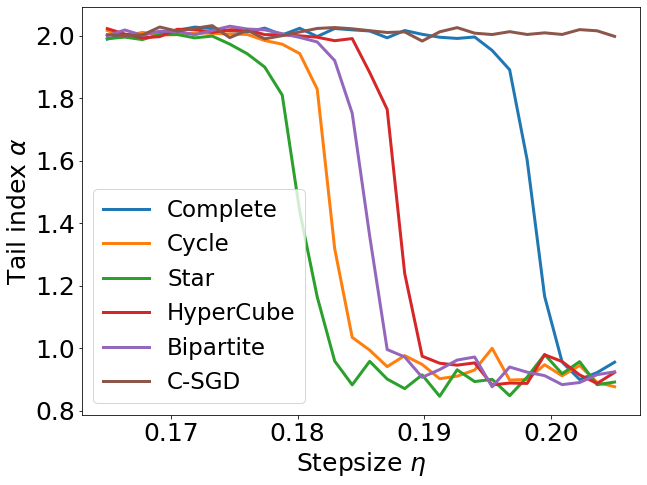}
    \label{fig:lin-reg-stepsize}
    }
    \subfigure[Varied batch size $b$.]{
    \includegraphics[width=0.45\columnwidth]{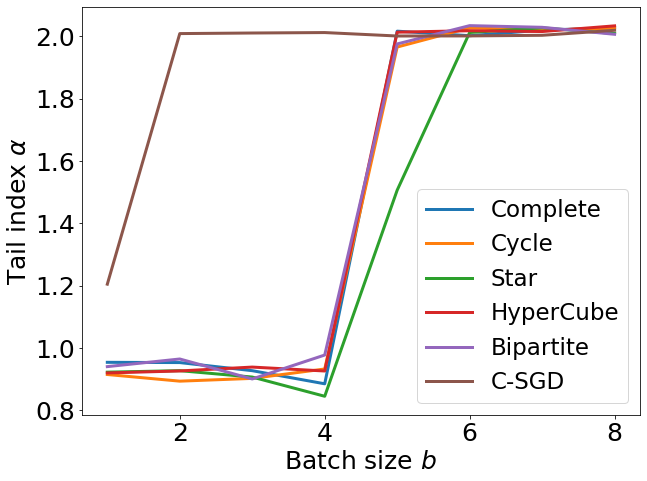}
    \label{fig:lin-reg-batch}
    }
\caption{Tail-index $\alpha$ for different setting of linear regression on synthetic data section.}
\end{figure}
\section{Tail index estimation in the numerical experiments}
\subsection{Synthetic data experiments}\label{subsec-synth-data}
Both DE-SGD and C-SGD iterates converge to a stationary distribution with proper choice of parameters. By the generalized CLT result, the ergodic average of the iterates will converge to an $\alpha$-stable distribution where $\alpha \in (0,2)$ if the stationary distribution has infinite variance, $\alpha=2$ otherwise. To estimate the tail index $\alpha$ of the averaged iterates, we use the same approach proposed in \citet{ht_sgd_quad}. Basically, we run the DE-SGD recursion \eqref{eqn:dsg_update} with a random initial point $x_i^{(0)} \in \mathbb{R}^d$ for each node $i$, and each entry in $x_i^{(0)}$ follows an i.i.d. uniform distribution $\mathcal{U}(-10,10)$. We repeat this procedure 1,600 times for different initial points and obtain 1,600 different random vectors for each node $i$. After that, instead of using the estimator proposed by \citet{mohammadi2015estimating} that evaluates directly on the final iterate $x_i^{(K)}$, building on the generalized CLT result, we compute the average of the `centered' iterates $\frac{1}{K-K_0} \sum_{k=K-K_0+1}^{K} (x_k-\bar{x})$, where $K_0$ is a `burn-in' period aiming to discard the initial phase of the stochastic algorithms, and $\bar{x}$ denotes the mean of the final $K - K_0$ iterates. We apply the estimator of \citet{mohammadi2015estimating} to the average of the `centered iterates'. In our synthetic experiments, we set $K=5000,K_0=500$.

\subsection{Deep learning experiments}
For estimating the tail index in the deep learning experiments, we use a similar approach to the one described in Section \ref{subsec-synth-data}. Basically, after training each neural network model sufficiently, we compute the average of the last 1000 algorithm iterates, whose distribution we expect to be close to an
$\alpha$-stable distribution by the GCLT. We then treat each layer as a collection of i.i.d. $\alpha$-stable random
variables and measure the tail-index of each individual layer separately by using the estimator
from \citet{mohammadi2015estimating}.

\section{Further Details of ResNet-20 Experiments}
Due to space limit in the main text, we are reporting some of the details of our ResNet-20 experiments with $N=24$ nodes (corresponding to Figure \ref{fig:resnet_24n16b} in the main text) in this section. 

For the experiments in Figure \ref{fig:resnet_24n16b}, after the tail-index of each layer in the ResNet-20 model is estimated, we treat the median of them as the tail-index of the whole ResNet-20 model. 
We train the ResNet-20 model for 200 epochs. The learning rate is warmed up in the first 5 epochs and is decayed by a
factor of 10 at the 100th and 150th epoch. We set the network size as $N = 24$ and simulate all nodes with three 2080Ti GPUs (each GPU simulates 8 nodes). Moreover, each node samples $b = 16$ data per iteration. Since this experiment is on a real distributed GPU system, we utilize
BlueFog \citep{ying2021bluefog} for the implementation of decentralized methods including topology organization, weight matrix generation, and efficient decentralized communication. For the implementation of C-SGD, we 
 utilize PyTorch’s native Distributed Data Parallel (DDP) module. All algorithms are tested with stepsizes ranging from 0.025 to 0.2. We find stepsizes larger than 0.2 will cause significant performance dropping in test accuracy.
 
\begin{figure}[h!]
\centering
\subfigure{
    \includegraphics[width=0.45\columnwidth]{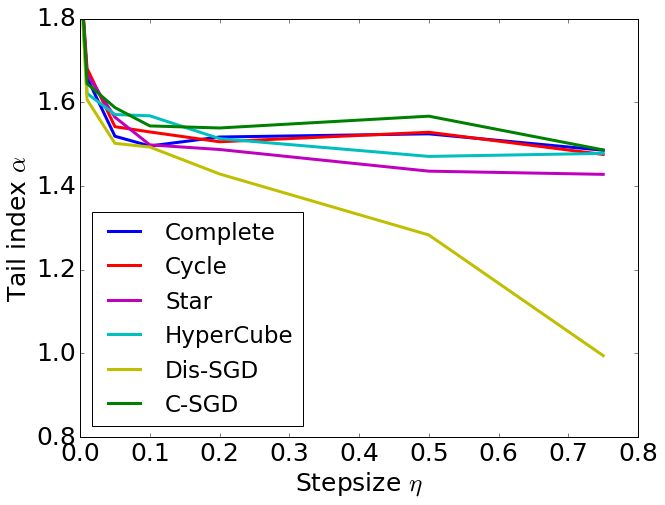}
    }
\caption{Tail-index $\alpha$ for ResNet-20 on CIFAR10 with $N=8$.}
\label{fig:resnet-cifar10-8n}
\end{figure}

Besides the experiment for $N=24$ shown in Figure~\ref{fig:resnet_24n16b}, we also report the results of our experiments for $N=8$ in Figure~\ref{fig:resnet-cifar10-8n}, where all other settings are the same as in Figure~\ref{fig:resnet_24n16b}. We can see that for the $N=8$ case, the tail-index of DE-SGD and C-SGD are closer to each other, whereas for $N=24$ case, the tail-index difference between DE-SGD and C-SGD is clearer. This suggests that  we can observe the phenomenon that DE-SGD has heavier tail than C-SGD, and the phenomenon will be clearer when the size of network $N$ gets larger. In fact, Proposition \ref{prop:disconnected:centralized} shows that the difference between the tail indices of Dis-SGD and C-SGD gets larger as $N$ increases, and our numerical results are compatible with this theoretical result.
\end{document}